\documentclass[11pt,english]{article}

\usepackage{caption, subcaption}
\usepackage{jmlr2e}

\usepackage{amsmath}
\usepackage{babel}
\usepackage{graphicx} 
\usepackage{algorithm}
\usepackage{grffile}
\usepackage{comment}
\usepackage{paralist} 
\usepackage{rotating}
\usepackage{xcolor}
\usepackage{mathtools}
\usepackage{diagbox}
\usepackage{array}
\usepackage{makecell}

\newcommand{\argmin}{\operatorname{argmin}}

\newcommand{\tr}{\operatorname{tr}}

\newcommand{\hbeta}{\hat{\beta}}

\newcommand{\hSigma}{\widehat{\Sigma}}

\newcommand{\tx}{\tilde{x}}

\newcommand{\tX}{\widetilde{X}}

\renewcommand{\le}{\leqslant}
\renewcommand{\ge}{\geqslant}

\newcommand{\e}{\mathbb{E}}

\newcommand{\ep}{\varepsilon}

\newcommand{\N}{\mathcal{N}}

\newcommand{\MSE}{\textnormal{MSE}}

\newcommand{\E}{\mathbb{E}}
\newcommand{\R}{\mathbb{R}}

\newcommand{\V}{\textnormal{Var}}

\newcommand{\bitem}{\begin{itemize}}
\newcommand{\eitem}{\end{itemize}}
\newcommand{\benum}{\begin{enumerate}}
\newcommand{\eenum}{\end{enumerate}}
\newcommand{\beq}{\begin{equation}}
\newcommand{\eeq}{\end{equation}}
\newcommand{\beqs}{\begin{equation*}}
\newcommand{\eeqs}{\end{equation*}}

\newcommand{\Slab}{{\Sigma}_{\mathrm{label}}}
\newcommand{\Ssam}{{\Sigma}_{\mathrm{sample}}}
\newcommand{\Sini}{{\Sigma}_{\mathrm{init}}}
\newcommand{\mT}{\mathcal{T}}
\newcommand{\Vl}{\mathbf{Var}(\lambda)}
\newcommand{\Bl}{\mathbf{Bias}^{2}(\lambda)}

\newcommand{\tM}{\tilde{M}}

\newcommand{\tth}{\tilde{\theta}}
\newcommand{\tlam}{\tilde{\lambda}}
\newcommand{\ga}{\gamma}

\newcommand{\sig}{\sigma}
\newcommand{\lam}{\lambda}
\newcommand{\be}{\bar{e}_d}
\newcommand{\tDel}{\tilde{\Delta}}
\newcommand{\deri}[1]{\mathrm{\frac{d}{d #1}}}
\newcommand{\mse}{\mathbf{MSE}}
\newcommand{\bias}{\mathbf{Bias}}
\newcommand{\var}{\mathbf{Var}}
\newcommand{\cM}{\mathcal{M}}
\newcommand{\tC}{\tilde{C}}

\newcommand{\tw}{\tilde{w}}
\newcommand{\nnum}{\nonumber}
\newcommand{\tu}{\tilde{u}}
\newcommand{\iC}{C^{-1}}
\newcommand{\iCt}{C^{-2}}

\newcommand{\Ep}{\mathcal{E}}

\graphicspath{
{./exp/}{}
}

\jmlrheading{22}{2021}{1-83}{10/20; Revised 2/21}{4/21}{20-1211}{Licong Lin and Edgar Dobriban}

\ShortHeadings{What Causes the Test Error?}{Lin and Dobriban}
\firstpageno{1}

{}
\begin{document}

\title{What Causes the Test Error? \\Going Beyond Bias-Variance via ANOVA}

\author{\name Licong Lin \email llc2000@pku.edu.cn \\
       \addr School of Mathematical Sciences\\
       Peking University\\
       5 Yiheyuan Road, Beijing, China
       \AND
       \name Edgar Dobriban \email dobriban@wharton.upenn.edu \\
       \addr Departments of Statistics \& Computer and Information Science\\
       University of Pennsylvania\\
       Philadelphia, PA, 19104-6340, USA}

\editor{Ambuj Tewari}

\maketitle

\begin{abstract}%
Modern machine learning methods are often overparametrized, allowing adaptation to the data at a fine level. This can seem puzzling; in the worst case, such models do not need to generalize.  This puzzle inspired a great amount of work, arguing when overparametrization reduces test error, in a phenomenon called ``double descent". Recent work aimed to understand in greater depth why overparametrization is helpful for generalization. This lead to discovering the unimodality of variance as a function of the level of parametrization, and to decomposing the variance into that arising from label noise, initialization, and randomness in the training data to understand the sources of the error.

In this work we develop a deeper understanding of this area. Specifically, we propose using \emph{the analysis of variance} (ANOVA) to decompose the variance in the test error in a symmetric way, for studying the generalization performance of certain two-layer linear and non-linear networks. The advantage of the analysis of variance is that it reveals the effects of initialization, label noise, and training data more clearly than prior approaches. Moreover, we also study the monotonicity and unimodality of the variance components. While prior work studied the unimodality of the overall variance, we study the properties of each term in the variance decomposition. 

One of our key insights is that often, the \emph{interaction} between training samples and initialization can dominate the variance; surprisingly being larger than their marginal effect. Also, we characterize ``phase transitions" where the variance changes from unimodal to monotone.
On a technical level, we leverage advanced deterministic equivalent techniques for Haar random matrices, that---to our knowledge---have not yet been used in the area. We verify our results in numerical simulations and on empirical data examples.  
\end{abstract}


\section{Introduction}
Modern machine learning methods are often overparametrized, allowing adaptation to the data at a fine level. For instance, competitive methods for image classification---such as WideResNet \citep{zagoruyko2016wide}---and for text processing---such as GPT-3 \citep{brown2020language}---have from millions to billions of explicit optimizable parameters, comparable to the number of datapoints. From a theoretical point of view, this can seem puzzling and perhaps even paradoxical: in the worst case, models with lots of parameters do not need to generalize (i.e., perform similarly on test data as on training data from the same distribution). 

This puzzle has inspired a great amount of work. Without being exhaustive, some of the main approaches argue the following. (1) Overparametrization beyond the ``interpolation threshold" (number of parameters required to fit the data) can eventually reduce test error (in a phenomenon called ``double descent"). (2) The specific algorithms used in the training process have beneficial ``implicit regularization" effects which effectively reduce model complexity and help with generalization. These two ideas are naturally connected, as the implicit regularization helps achieve decreasing test error with overparametrization. This area has registered a great deal of progress recently, but its roots can be traced back many years ago. We discuss some of these works in the related work section.

\begin{figure}[tb]
\centering
\includegraphics[width=0.49\linewidth]{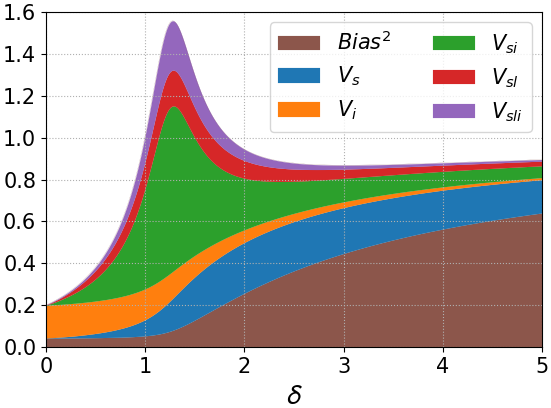}
\includegraphics[width=0.49\linewidth]{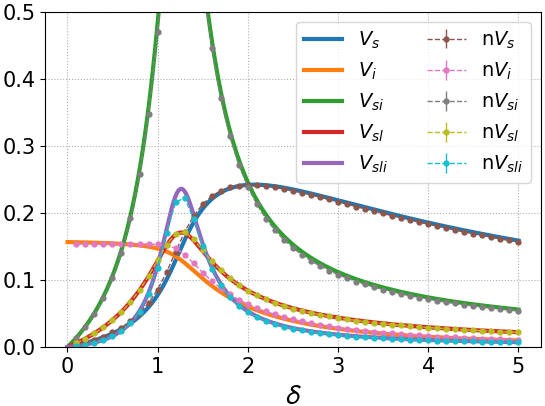}
\caption{ANOVA decomposition of the variance. The plots show the components of the variance (as well as the bias) in certain two-layer linear networks studied in the paper, as a function of the data aspect ratio $\delta = \lim d/n$, where $d$ is the {dimension of features} and $n$ is the number of samples. The variance can be decomposed into its contributions from randomness in label noise ($l$), training data/samples ($s$), and initialization ($i$). Namely, the variance is decomposed into the main effects $V_{a}$ and the interaction effects $V_{ab},V_{abc}$, where $a,b,c \in \{l,s,i\}$. We omit $V_{l},V_{li}$ in the figures since they equal zero. The key observation is that the \emph{interaction effects} (especially $V_{si}$) dominate the variance at the interpolation limit where $\lim p/n=1$ {(this turns out to correspond to $\delta=1.25$ in the Figure)} and $p$ is the number of features in the hidden layer. {\bf Left}: Cumulative figure of the bias and variance components. {\bf Right}: Variance components in numerical simulations. ($\star$: theory, $n\star$: numerical, averaged over $5$ runs, for $\star = V_s$, etc). Parameters: signal strength $\alpha=1$, noise level $\sigma=0.3$, regularization parameter $\lambda=0.01$, parametrization level $\pi=0.8$. See Sections \ref{bvsob}, \ref{deriest} for details.}
\label{filled}
\end{figure}

One particular line of work aims to understand in greater depth why overparametrization is helpful for generalization. In this line of work, \cite{yang2020rethinking} has studied the bias-variance decomposition of the mean squared error (and for other losses), and proposed that a key phenomenon is that the variance is \emph{unimodal} as a function of the level of parametrization. This was verified empirically for a wide range of models including modern neural networks, as well as theoretically for certain two-layer linear networks with only the second layer trained. Moreover, \cite{dascoli2020double} proposed to decompose the variance in a two-layer non-linear network with only second layer trained (i.e., a random features model) into that arising from label noise, initialization, and randomness in the features of the training data (in this specific order), arguing that---in their particular model---the label and initialization noise dominates the variance.


In this work we develop a set of techniques aiming to improve our understanding of this area; and more broadly of generalization in machine learning. Specifically, we propose to use \emph{the analysis of variance} (ANOVA), a classical tool from statistics and uncertainty quantification \citep[e.g.,][]{box2005statistics,mcbook}, to decompose the variance in the generalization mean squared error into its components stemming from the initialization, label noise, and training data (see Figure \ref{filled} for a brief example). The advantage of the analysis of variance is that it reveals the effects of the components in a more clearly interpretable, and perhaps "unequivocal", way than the approach in \cite{dascoli2020double}. The prior decomposition depends on the specific order in which the conditional expectations are evaluated, while ours does not. We carry out this program in detail in certain two-layer linear and non-linear network models (more specifically, random feature models), which have have been the subject of intense recent study, and are effectively at the frontier of our theoretical understanding.

As is well known in the literature on ANOVA, the variance components form a hierarchy whose first level, the \emph{main effects}, can be interpreted as the effects of varying each variable (here: random initialization, features, label noise) separately, while the higher levels can be interpreted as the interaction effects between them. These are symmetric, which is both elegant and interpretable, and thus provide advantages over the prior approaches. See Figure \ref{varsdecomp1} for an example.

Moreover, we study the monotonicity and unimodality of MSE, bias, variance, and the various variance components in a specific variance decomposition. While \cite{yang2020rethinking} studied the unimodality of the overall variance, we study the properties the components individually. On a technical level, our work is quite involved, and leverages some advanced techniques from random matrix theory, that---to our knowledge---have not yet been used in the area. In particular, we discovered that we can leverage the deterministic equivalent results for Haar random matrices from \cite{couillet2012random}. These have  been developed for different purposes, for analyzing random beamforming in wireless communications. 

{After the initial posting of our work, we became aware of the highly related paper \cite{adlam2020understanding}. This was publicly posted on the \url{arxiv.org} preprint server later than our work, but had been submitted for publication earlier. The conclusions in the two works are similar, but the techniques and setting are different. We discuss this at the end of the next section.}

\subsection{Related Works}

There is an extraordinary amount of related work, as this topic is one of the most exciting and popular ones recently in the theory of machine learning. Due to space limitations, we can only review the most closely related work.

The phenomenon of ``double descent", coined in \cite{belkin2018reconciling}, states that the limiting test error first increases, then decreases as a function of the parametrization level, having a ``double descent", or ``w"-shaped behavior. This phenomenon has been studied, in one form or another, in a great number of recent works, see e.g., \cite{advani2020high,bartlett2020benign,belkin2018reconciling,belkin2018overfitting,belkin2020two,derezinski2019exact,geiger2020scaling,ghorbani2021linearized,hastie2019surprises,liang2018just,Li2020ProvableMD,mei2019generalization,muthukumar2020harmless,xie2020weighted}, etc.

Various forms have also appeared in earlier works, see e.g., the discussion on ``A brief prehistory of double descent" \citep{loog2020brief} and the reply in \cite{Belkin10627}. This points to the related works \cite{opper2001learning,kramer2009peaking}. The online machine learning community has engaged in a detailed historical reference search, which unearthed the related early works\footnote{The reader can see the Twitter thread by Dmitry Kobak: \url{https://twitter.com/hippopedoid/status/1243229021921579010}.} \cite{hertz1989phase,opper1990ability,hansen1993stochastic,barber1995finite,duin1995small,opper1995statistical,opper1996statistical,raudys1998expected}. The observations on the ``peaking phenomenon" are consistent with empirical results on training neural networks dating back to the 1990s. There it has been suggested that the difficulties captured by the peak in double descent stem from optimization, such as the ill-conditioning of the Hessian \citep{lecun1991second,le1991eigenvalues}.

Some works that are especially relevant to us are the following. \cite{hastie2019surprises} showed that the limiting MSE of ridgeless interpolation in linear regression as a function of the overparametrization ratio, for fixed SNR, has a double descent behavior. 
\cite{nakkiran2021optimal} rigorously proved that optimally regularized ridge regression can eliminate double descent in finite samples in a linear regression model. \cite{nakkiran2019more} clearly explained that ``more data can hurt", because  algorithms do not always adapt well to the additional data. In comparison, the special case of our results pertaining to linear nets allows for certain non-Gaussian data, while only proved asymptotically.   \cite{Nakkiran2020DeepDD} empirically showed a double descent shape for the test risk for various neural network architectures a function of model complexity, number of samples (``sample-wise" double descent), and training epochs.

 \cite{dascoli2020double} used the  (not fully rigorous) replica method to obtain the bias-variance decomposition for two-layer neural networks in the lazy training regime. They further also decomposed the variance in a specific order into that stemming from label noise, initialization, and training features. Compared to this, our work is fully rigorous, and proposes to use the analysis of variance, from which we show that the sequential decompositions like the ones proposed in \cite{dascoli2020double} can be recovered. Moreover, we are concerned with a slightly different model (with orthogonal initialization), and some of our results are only proved for linear orthogonal networks (e.g., the forms of the variance components). However, going beyond \cite{dascoli2020double}, we also obtain rigorous results for the monotonicity and unimodality of the various elements of the variance decomposition.  

 \cite{jimmy2020generalization}  obtained the generalization error of two-layer neural networks when only training the first or the second layer, and compared the effects of various algorithmic choices involved. Compared with our work, \cite{dascoli2020double,jimmy2020generalization} studied more general settings and provided results that involve more complex expressions; the advantage our our simpler expressions is that we can find the variance components and study properties such as their monotonicity. Our results are simpler mainly because we consider orthogonal initialization; and, in several results, consider a linear network. We believe that our results are complementary.

\cite{wu2020optimal} calculated the prediction risk of optimally regularized ridge regression under a general covariance assumption of the data.
\cite{Jacot2020implicit} argued that random feature models can be close to kernel ridge regression with additional regularization. This is related to the ``calculus of deterministic equivalents" for random matrices \citep{dobriban2018understanding}. \cite{liang2020multiple} argued that in certain kernel regression problems one may obtain generalization curves with multiple descent points. \cite{Chen2020MultipleDD} studied certain models with provable multiple descent curves. 

{When the data has general covariance, \cite{kobak2020the} showed that the optimal ridge parameter could be negative.  Thus any positive ridge penalty would be sub-optimal if the true parameter vector lies on a direction with high predictor variance. Understanding the implications of this work in our context is a subject of interesting future research.} 

More broadly viewed, a great deal of effort has been focused on connecting ``classical" statistical theory (focusing on low-dimensional models) with ``modern" machine learning (focusing on overparametrized models).\footnote{The reader can see e.g., the talks titled ``From classical statistics to modern machine learning" by M. Belkin at the Simons Institute (\url{https://simons.berkeley.edu/talks/tbd-65}) at the Institute of Advanced Studies (\url{https://video.ias.edu/theorydeeplearning/2019/1016-MikhailBelkin}), and and other venues.} From this perspective, there are strong analogies with nonparametric statistics \citep{ibragimov2013statistical}. Non-parametric estimators such as kernel smoothing have, in effect, infinitely many parameters, yet they can perform well in practice and have strong theoretical guarantees.  Nonparameteric statistics already has the same components of the ``overparametrize then regularize" principle as in modern machine learning. {The same principle also arises in high-dimensional statistics, such as with basis pursuit and Lasso \citep{chen1994basis}. Namely, one can get good performance if one considers a large set of potential predictors (overparametrize), and then selects a small, highly-regularized subset.} 

Even more broadly, our work is connected to the emerging theme in modern statistics and machine learning of studying high-dimensional asymptotic limits, where both the sample size and the dimension of data tend to infinity. This is a powerful framework that allows us to develop new methods, and to uncover phenomena not detectable using classical fixed-dimension asymptotics \cite[see e.g.,][]{couillet2011random,paul2014random,yao2015large}. It also dates back to the 1970s, see e.g., the literature review in \cite{dobriban2018high}, which points to works by \cite{raudys1967determining,deev1970representation,serdobolskii1980discriminant}, etc.
Some other recent related  works include \cite{pennington2017nonlinear,louart2018random,liao2018spectrum,liao2019large,benigni2019eigenvalue,goldt2019modelling,fan2020spectra,deng2019model,gerace2020generalisation,liao2020random,adlam2019random,adlam2020neural}. See also \cite{geman1992neural,bos1997dynamics,neal2018modern} for various classical and modern discussions of bias-variance tradeoffs and dynamics of training.

{The most closely related work to ours is \cite{adlam2020understanding}, publicly posted later, but submitted for publication earlier. Both works study the generalization error via ANOVA decomposition, and show that the interaction effect can dominate the variance. The conclusions in the two works are similar. For instance, the $V_{si}$ term (interaction between samples and initialization, defined later) dominates the total variance; $V_{si}$ and $V_{sli}$ (interaction between samples, label noise and initialization) diverge as the ridge regularization parameter $\lambda\to 0$. On the other hand, there are many differences between two works. (1). The mathematical settings are different. Their work studies a two-layer nonlinear network with Gaussian initialization, while we study both linear and nonlinear networks with orthogonal initialization. (2). The mathematical tools employed in the two papers are different. They use Gaussian equivalents and the linear pencil representation, while we exploit  orthogonal deterministic equivalents. (3). The results are different. Beyond the ANOVA decomposition, they also study the effect of ensemble learning. On our end, we study optimally tuned ridge regression and prove properties of the bias, variance and MSE.} 

{Another related paper, also publicly posted after our work is by \cite{rocks2020memorizing}. They study generalization error in linear regression and two-layer networks by deriving the formulas for bias and variance. The main techinque they used is the cavity method originating from statistical physics. Similarly, they also show that the generalization error diverges at the interpolation threhold due to the large variance. We provide a more detailed comparison later, after stating our main results.}

{As already mentioned, our work is related to the one by \cite{yang2020rethinking}.
The model we consider is related to theirs, with several key differences. One is the orthogonal initialization, in contrast to their Gaussian initialization. Also, they assume that the ratio $d/n\to 0$ while we study the proportional regime where $d/n\to \delta>0$ (which can be arbitrarily small, so our setting is in a sense effectively more general). As for the results, they prov the unimodality of variance and monotonicity of the bias under their setting. They also make some conjectures on the variance unimodality that we prove (keeping in mind the different settings), see the results section for more details.} 

\subsection{Our Contributions}
{}

Our contributions can be summarized as follows:

\begin{enumerate}

\item We study a two-layer linear network where the first layer is a fixed partial orthogonal embedding (which determines the latent features) and the second layer is trained with ridge regularization. While the expressive power of this model only captures certain linear functions, training only the second layer already exhibits certain  intriguing statistical and generalization phenomena. We study the prediction error of this learning method in a noisy linear model. We consider three sources of randomness that contribute to the error: the random initialization (a random partial orthogonal embedding), the label noise, and the randomness over the training data. We propose to use \emph{the analysis of variance} (ANOVA), a classical tool from statistics and uncertainty quantification \citep[e.g.,][]{box2005statistics,mcbook} to decompose and understand their contribution. 

We study an asymptotic regime where the data dimension, sample size, and number of latent features tends to infinity together, proportionally to each other. In this model, we calculate the limits of the variance components (Theorem \ref{sobolthm}); in terms of moments of the Marchenko-Pastur distribution \citep{marchenko1967distribution}. We then show how to recover various sequential variance decompositions, such as the one from \cite{dascoli2020double} (albeit only for linear rather than nonlinear networks). We also show that the order in the sequence of decompositions matters.  Our work leverages deterministic equivalent results for Haar random matrices from \cite{couillet2012random} that, to our knowledge, have not yet been used in the area. We also leverage recent technical developments such as the \emph{calculus of deterministic equivalents} for random matrices of the sample covariance type \citep{dobriban2018understanding,dobriban2020wonder}. Proofs are in Appendix \ref{all_proofs}.

\item We then study the bias-variance decomposition in greater detail. As a corollary of the ANOVA results, we study the decomposition of the variance in the order label-sample-initialization, which has some special properties (Theorem \ref{2lthm1}). When using an optimal ridge regularization, we study the monotonicity and unimodality properties of these components (Theorem \ref{2lmon} and Table \ref{allmonotonicity}). With this, we shed further light on phenomena discovered by \cite{yang2020rethinking}, who wrote that ``The main unexplained mystery is the unimodality of the variance". Specifically, we are able to show that the variance is indeed unimodal in a broad range of settings. This analysis goes beyond prior works e.g., \cite{yang2020rethinking} (who studied setting with a number of inner neurons being much larger than the number of datapoints), or ``double descent mitigation" as in \cite{nakkiran2021optimal}, because it studies bias and variance separately.

We uncover several intriguing properties: for instance, for a fixed parametrization level $\pi$, as a function of aspect ratio or ``dimensions-per-sample", the variance is monotonically decreasing when $\pi<0.5$, and unimodal when $\pi\geq0.5$. We discuss and offer possible explanations.

We also discuss the special case of linear models, which has received a great deal of prior attention (Proposition \ref{lin}). We view the results on standard linear models as valuable, as they are both simpler to state and to prove, and moreover they also directly connect to some prior work. 

\item We develop some further special properties of the bias, variance, and MSE. We report a seemingly surprising simple relation between the MSE and bias at the optimum (Section \ref{surp}). We study the properties of the bias and variance for a \emph{fixed} (as opposed to optimally tuned) ridge regularization parameter (Theorem \ref{bias_var_for_fixed_lambda}). In particular, we show that the bias decreases as a function of the parametrization, and increases as a function of the data aspect ratio. 
In contrast to choosing $\lambda$ optimally, we see that double descent is \emph{not} mitigated, and may occur in our setting when we use a small regularization parameter $\lambda$ that is fixed across problem sizes (going beyond the models where this was known from prior work). This corroborates that the lack of proper regularization plays a crucial role for the emergence of double descent.

We also give an added noise interpretation of the initial random initialization step (Section \ref{add_noise}). Further, we provide some detailed analysis and intuition of these phenomena, aided by numerical plots of the variance components (Section \ref{understand_opt}).

\item The above results are about ridge regression as a heuristic for regularized empirical risk minimization. In some settings, ridge regularization is known to have limitations \citep{derezinski2014limits}, thus it is an importat question to understand its fundamental limitations here. In fact, we can show that ridge regression is an asymptotically optimal estimator, in the sense that it converges to the Bayes optimal estimator in our model  (Theorem \ref{ridgeopt}). This provides some justification for studying ridge regression in a two-layer network, which is not covered by standard results.

\item We extend some of our results to two-layer networks with a non-linear activation function with orthogonal initialization. In particular, we provide the limits of the MSE, bias, and variance in the same asymptotic regime (Theorem \ref{nlbiasvardecomp}). Furthermore, we provide the monotonicity and unimodality properties of these quantities as a function of parametrization and aspect ratio (Table \ref{allmonotonicitynonlinear}).

\item We provide numerical simulations to check the validity of our theoretical results (Section \ref{simu}), including the MSE, the bias-variance decomposition, and the variance components. We also show some experiments on empirical data, specifically on the superconductivity data set \citep{hamidieh2018data}, where we test our predictions for two-layer orthogonal nets. Code associated with the paper is available at \url{https://github.com/licong-lin/VarianceDecomposition}.

\eenum

\subsection{Highlights and Implications}

We discuss some of the highlights and implications of our results.

\paragraph{Beyond bias-variance.} Much of the prior work in this area has focused on the fundamental bias-variance decomposition. In this work, we demonstrate that it is possible to go significantly beyond this via the ANOVA decomposition. Specifically, using this methodology, one can understand how the random training data, initialization, and label noise contribute to the test error in more detailed and comprehensive ways than what was previously possible. We carry out this in certain two-layer linear and non-linear networks with only the second layer trained (i.e., random features models), but our approach may be more broadly relevant.

\paragraph{Non-additive test error.} A key finding of our work is that in the specific neural net models considered here, the random training data, initialization, and label noise contribute highly non-additively to the test error. Thus, when discussing ``the effects of initialization", some care ought to be taken; i.e., to clarify which interaction effects (e.g., with label noise or training data) this includes. The interaction term between the initialization and the training data can be large in our setting.

\paragraph{Beyond double descent: Prevalence of unimodality.} While initial work on asymptotic generalization error of one and two-layer neural nets focused on the ``double descent" or peaking shape of the test error, our work gives further evidence that the unimodal shape of the variance is a prevalent phenomenon. Moreover, our work also suggests the the unimodality holds not just for the overall variance, but also for specific and variance components; which was not known in prior work. We show that unimodality with respect to both overparametrization level and data aspect ratio holds in specific parameter settings for the variance and certain other decompositions for the optimal setting of the regularization parameter. In other parameter settings, we obtain monotonicity results for these components. This also underscores that regularization and the associated bias-variance tradeoff plays a key role in determining monotonicity and unimodality.

\section{ANOVA for a Two-layer Linear Network}

\subsection{Setup}

In this section, we study the bias-variance  tradeoff and ANOVA decomposition for a two-layer linear network model.
Suppose that we have  a training data set $\mT$ containing $n$ data points $({x}_i,y_i)\in \R^d\times\R$, with features $x_i$  and outcomes $y_i$. We assume the data is  drawn independently from a distribution such that $x_i=(x_{i1},x_{i2},...,x_{id})$, where $x_{ij}$ are i.i.d. random variables satisfying \begin{align*}&\E x_{ij}=0, \text{\hspace{10em}}\E x_{ij}^2=1, \text{\hspace{10em}}\E x_{ij}^{8+\eta}<\infty,\end{align*} where $\eta>0$ is an arbitrary constant.  
Also, each $(x_i,y_i)$ are drawn from the model
 \begin{align*}
y &=f^*(x)+\ep=x^\top  \theta+\ep, \,\theta\in\R^d,
\end{align*}
where $\ep\sim\N(0,\sigma^2)$ is the label noise independent of $x$ and $\sigma\geq0$ is the noise standard deviation. In matrix form, $Y = X\theta + \Ep$, where $X=({x}_1,{x}_2,...,{x}_n)^\top\in\R^{n\times d}$ has input vectors ${x}_i$, $i=1,2,...,n$ as its rows and $Y=(y_1,y_2,...,y_n)^\top\in\R^{n\times 1}$ with output values $y_i$, $i=1,2,...,n$ as its entries.
 Our task is to learn the true regression function $f^*(x)=x^\top  \theta$ by using a two-layer linear neural network with weights $W\in \R^{p\times d}$, $\beta\in\R^{p\times 1}$, which computes for an input $x\in\R^d$,
 \begin{align}\label{estimate_fun_2lr_form}
  f(x)=(Wx)^\top \beta.
 \end{align}
Later in Section \ref{nonlin} we will also study two-layer nonlinear networks. 
For analytical tractability, we assume that the true parameters $\theta$ are random: $\theta\sim \N(0,\alpha^2I_d/d)$. Here $\alpha^2$ can be viewed as a signal strength parameter. This assumption corresponds to performing an ``average-case" analysis of the difficulty of the problem over random problem instances given by various $\theta$. 

We also consider a random \emph{orthogonal} initialization $W$ independent of $\mT$, so $W$ is a $p\times d$ matrix uniformly distributed over the set of matrices satisfying $WW^\top =I_p$, also known as the Stiefel manifold. This requires that $p \le d$, so the dimension of the inner representation of the neural net is not larger than the number of input features.  To an extent, this can be seen as a random projection model, where a lower-dimensional representation of the high-dimensional input features is obtained by randomly projecting the input features into a subspace. Both training and prediction are based on the lower dimensional representation. In some works studying the orthogonal initialization of neural networks \cite[e.g.,][]{hu2019provable}, the first layer weights $W_{1}$ satisfy $W_1 ^\top W_1 =I_1$, while the last layer weights $W_L$ satisfy $W_LW_L^\top=I_n$, so the dimension of the hidden representation is larger than the dimension of the input features. Similarly, in several recent works on wide neural networks, the number of inner neurons is large.
However, we think that in many applications, the number of ``higher level features" should indeed not be  larger than the number of input features. For instance, the number of features in facial image data such as eyes, hair, is expected to be not more than the number of pixels. 

The model we consider here is related to the one from \cite{yang2020rethinking}, with several key differences. The orthogonal initialization is a key difference, as  \cite{yang2020rethinking} assume that $W$ is a random Gaussian matrix. The expressive power of the two models is the same, but orthogonal initialization has some benefits (see Appendix \ref{orthogonal} for more information).

During training, we fix $W$ and estimate $\beta$ by performing ridge regression:
\begin{align}\label{2lr}
\hbeta_{\lambda,\mT,W}=\arg\min_{\beta\in\R^p}\frac{1}{2n}\|Y-(WX^\top)^\top\beta\|_2^2+\frac{\lambda}{2}\|\beta\|_2^2,\end{align}
where $\lambda>0$ is the regularization parameter. This has a closed-form solution 
\begin{align}\label{closed_form_2lr}
\hbeta_{\lambda,\mT,W}=\left(\frac{WX^\top  XW^\top }{n}+\lambda I_p\right)^{-1}\frac{WX^\top  Y}{n}.
	\end{align}
We will often use the notation $R = (WX^\top  XW^\top /n+\lambda I_p)^{-1}$ for the so-called resolvent matrix of $WX^\top  XW^\top $. 
By plugging it into (\ref{estimate_fun_2lr_form}), we obtain our  estimated prediction function, for a new datapoint $x$, projected first via $W$ and thus accessed via $Wx$:
\begin{align} f(x)&=(Wx)^\top \hbeta_{\lambda,\mT,W}
	=x^\top W^\top R\frac{WX^\top Y}{n}.\label{estimate_fun_l2r}
\end{align}
Ridge regression is equivalent to $\ell_2$ weight decay, a popular heuristic. We will later show that ridge has some asymptotic optimality properties in our model, which thus justify its choice. In contrast, if we follow the approach from \cite{hu2019provable} and take $p\ge d$ with $W^\top W=I_d$, then it is readily verified that  we would obtain
\begin{align*}f(x)&=x^\top \left({X^\top X /n}+\lambda I_d\right)^{-1}{X^\top Y/n}.\end{align*} 
This means that the prediction function reduces to standard ridge regression. 
Thus we assume instead that $WW^\top=I_p$ and this makes our model resemble the ``feature extraction" layers of a neural network.

We will consider the following asymptotic setting. Let $\{p_d,d,n_d\}_{d=1}^{\infty}$ be a sequence such that $p_d\leq d$ and  $p_d,d,n_d\to\infty$ proportionally, i.e.,
\begin{align*}&\lim_{d\to\infty}\frac{p_d}{d}=\pi, &\lim_{d\to\infty}\frac{d}{n_d}=\delta ,\end{align*} where $\pi\in(0,1]$ and $\delta \in(0,\infty)$. 
Here $\pi \in (0,1]$ denotes the parametrization factor, i.e., the number of parameters in $\beta$ relative to the input dimension. Also, $\delta >0$ is the data aspect ratio. We will also use $\gamma = \delta\pi = \lim p/n$, the ratio of learned parameters to number of samples. In \cite{yang2020rethinking}, the assumption $\lim_{d\to\infty}d/n_d=0$ implies  that the number of samples is much larger than the number of parameters, which is limiting in high dimensional problems. Thus, we study a broader setting, in which the sample size is proportional to the model size, with an arbitrary ratio.

\subsection{Bias-Variance Decompositions and ANOVA}
\label{bvsob}      
\subsubsection{{Introduction and the Main Result}}

Now we analyze the mean squared prediction error in our model,
\begin{align*}\mse(\lambda)&=\E_{\theta,x,\ep,X,\Ep,W}(f_{\lambda,\mT,W}(x)-y)^2\\
&=\E_{\theta,x,X,\Ep,W}(f_{\lambda,\mT,W}(x)-x^\top\theta)^2+\sigma^2.\end{align*}
The expectation is over a random test datapoint $(x,y)$ from the same distribution as the training data, i.e., $x\in \R^{d}$ has i.i.d. zero mean, unit variance entries with finite $8+\eta$ moment, $y = \theta^\top x + \ep$, and $\ep\sim \N(0,\sigma^2)$. It is also over the random training input $X$, the random training label noise $\Ep$ 
, the random initialization $W$, and the random true parameter $\theta$. Thus, this MSE corresponds to an average-case error over various training data sets, initializations, and true parameters. In this work, we will always average over the random test data point $x$, as is usual in classical statistical learning. We will also average over the random parameter $\theta$, which corresponds to a Bayesian average-case analysis over various generative models. This is partly for technical reasons; we can also show almost sure convergence over the random $\theta$, but the analysis for general $\theta$ is beyond our scope. Thus, we write the MSE as $\E_{\theta,x}\E_{X,\Ep,W}(f_{\lambda,\mT,W}(x)-x^\top\theta)^2+\sigma^2$, and the outer expectation $\E_{\theta,x}$ is always present in our formulas.

We can study the mean squared error via the standard bias-variance decomposition corresponding to the average prediction function $\E_{X,\Ep,W}f_{\lambda,\mT,W}(x)$ over the random training set $\mT$ ($X,\Ep$) and initialization $W$: $\mse(\lambda)=\Bl+\Vl+\sigma^2$, where
\begin{align*}
\Bl&=\E_{\theta,x}(\E_{X,\Ep,W}f_{\lambda,X,\Ep,W}(x)-x^\top\theta)^2,\\
\Vl= \V[\hat{f}(x)]&=\E_{\theta,x}\E_{X,\Ep,W}(f_{\lambda,\mT,W}(x)-\E_{X,\Ep,W}f_{\lambda,\mT,W}(x))^2.
\end{align*}
A key point is that the variance can be further decomposed into the components due to the randomness in the training data $X$, label noise $\Ep$, and initialization $W$. 

We use $s,l,i$ to represent the samples $X$, label noise $\Ep$ and initialization $W$, respectively. Their impact on the variance can be decomposed in a symmetric way into their main and interaction effects via the the analysis of variance (ANOVA) decomposition \citep[e.g.,][]{box2005statistics,mcbook} as follows:
\begin{align*}
\V[\hat{f}(x)]=V_s+V_l+V_i+V_{sl}+V_{si}+V_{li}+V_{sli},
\end{align*} 
where 
\begin{align*}
V_a&=\E_{\theta,x}\V_a[\E_{-a}(\hat{f}(x)|a)], &a\in\{s,l,i\}\\
V_{ab}&=\E_{\theta,x}\V_{ab}[\E_{-ab}(\hat{f}(x)|a,b)]-V_a-V_b, &a,b\in\{s,l,i\}, a\neq b.\\
V_{abc}&=\E_{\theta,x}\V_{abc}[\E_{-abc}(\hat{f}(x)|a,b,c)]-V_a-V_b-V_c-V_{ab}-V_{ac}-V_{bc}&\\
&=\V[\hat{f}(x)]-V_s-V_l-V_i-V_{sl}-V_{si}-V_{li}, &\{a,b,c\}=\{s,l,i\}.
\end{align*}
Here, $\E_{-a}$ means taking expectation with respect to all components except for $a$. Then $V_a \ge 0$ can be interpreted as the effect of varying $a$ alone, also referred to as the \emph{main effect} of $a$. Also, $V_{ab} = V_{ba} \ge 0$ can be interpreted as the second-order interaction effect between  $a$ and $b$ beyond their main effects,  and $V_{abc} \ge 0$ can be seen as the interaction effect among $a,b,c$, beyond their pairwise interactions. The ANOVA decomposition is symmetric in $a,b,c$. 
We will show how to recover some sequential  variance decompositions from the ANOVA decomposition later. We mention that this decomposition is sometimes referred to as functional ANOVA \citep{mcbook}.

For intuition, consider the noiseless case when the label noise equals zero, so $\Ep=0$ and the index $l$ does not contribute to the variance. Then the main effect of the samples/training data is  $V_s $ $ =$ $\E_{\theta,x}\V_s[\E_{-s}(\hat{f}(x)|s)] $ $ =$  $\E_{\theta,x}\V_X[\E_W(\hat{f}(x)|X)]$. This can be interpreted as the variance of the ensemble estimator $\tilde{f}(x):=\E_{W}\hat{f}(x)$, the average of models parametrized by various $W$-s on the same data set $X$. 
Therefore, $V_s$ can be regarded as the expected variance with respect to training data of the ensemble estimator $\tilde{f}$; where the expectation is over the test datapoint $x$ and the true parameter $\theta$. Furthermore, $V_{si} +V_i$ is the variance that can be eliminated by ensembling. This is because $V_{si}+V_{s}+V_{i}$ is the total variance of the estimator $\hat{f}(x)$; and $V_s$ is the variance of the ensemble, thus $V_{si}+V_{i}$ is the remaining variance that can be eliminated.

Our bias-variance decomposition is slightly different from the standard one in statistical learning theory \citep[e.g.,][p. 24]{friedman2009elements}, where the variability is introduced only by the random training set $\mT$. Here we also consider the variability due to the random initialization $W$; and decompose the variance due to $\mT$ into that due to samples $X$ and label noise $\Ep$. The motivation is because randomization in the algorithms, such as random initialization, as well as randomness in stochastic gradient descent, are very common in modern machine learning. Our decomposition helps understand such scenarios.

Now, we define the following quantities which are frequently used throughout our paper. These are ``resolvent moments" of the well-known Marchenko-Pastur (MP) distribution $F_\gamma$ \citep{marchenko1967distribution}. The MP distribution is the limit of the distribution of eigenvalues of sample covariance matrices $n^{-1} Z^\top Z$ of $n\times p$ data matrices $Z$  with iid zero-mean unit-variance entries when $n,p\to\infty$ with $p/n\to\gamma>0$ \citep{bai2009spectral,couillet2011random,anderson2010introduction,yao2015large}.
\begin{definition}[Resolvent moments]\label{resolvent}In this paper, we use the first and second resolvent moments
\begin{align}
&\theta_1(\gamma,\lam):=\int\frac{1}{x+\lam}dF_\gamma(x),
&\theta_2(\gamma,\lam):=\int\frac{1}{(x+\lam)^2}dF_\gamma(x),\label{deftheta12}
\end{align}
where $F_{\gamma}(x)$ is the Marchenko-Pastur distribution with parameter $\gamma$. Recall that for us, $\gamma = \delta\pi$. Then $\theta_1:=\theta_1(\ga,\lam)$ and $\theta_2:=\theta_2(\ga,\lam)$ have explicit expressions \citep[see e.g.,][for the first one; and our proofs also contain the derivations]{bai2009spectral}:
\begin{align}
\theta_1&=\frac{(-\lambda+\gamma-1)+\sqrt{(-\lambda+\gamma-1)^2+4\lambda \gamma}}{2\lam\ga}\label{exprtheta1},\\
\theta_2&=-\frac{d}{d\lam}\theta_1
=\frac{(\gamma-1)}
{2\gamma\lambda^2}
+
\frac{(\gamma+1) \cdot \lambda+(\gamma-1)^2}
{2\gamma\lambda^2\sqrt{(-\lambda+\gamma-1)^2+4\lambda \gamma}}\label{exprtheta2}.\end{align}
We further define
\begin{equation*}
\tlam:=\lam+\frac{1-\pi}{2\pi}\left[\lam+1-\ga+\sqrt{(\lam+\ga-1)^2+4\lam}\right],
\end{equation*} 
and denote $\tth_1:=\theta_1(\delta,\tlam)$, $\tth_2:=\theta_2(\delta,\tlam)$.
\end{definition}
{\bf Remark. } From \eqref{exprtheta1}, it is readily verified that $\lam\ga\theta_1^2+(\lam-\ga+1)\theta_1-1=0$. Taking derivatives and noting that $\theta_2=-d\theta_1/d\lam$, we get $1+(\ga-1)\theta_1-2\lam^2\ga\theta_1\theta_2-\lam(\lam-\ga+1)\theta_2=0.$ These two equations will be useful when simplifying certain formulas.
Then we have the following fundamental result on the asymptotic behavior of the variance components. This is our first main result.
\begin{theorem}[Variance components]\label{sobolthm}
Under the previous assumptions, consider an $n\times d$ feature matrix $X$ with i.i.d. entries of zero mean, unit variance and finite $8+\eta$-th moment. Take a two-layer linear neural network $f(x) = (Wx)^\top \beta$, with $p \le d$ intermediate activations, and $p\times d$ matrix $W$ of first-layer weights chosen uniformly subject to the orthogonality constraint $WW^\top =I_p$. Then, the variance components have the following  limits as $n,p,d\to\infty$, with $p/d\to \pi\in(0,1]$ (parametrization level), $d/n\to\delta>0$ (data aspect ratio). Here $\alpha^2$ is the signal strength, $\sigma^2$ is the noise level, $\lambda$ is the regularization parameter, $\theta_i$ are the resolvent moments (and $\tlam,\tth_i$ are their adjusted versions).
\begin{align*}
\lim_{d\to\infty} V_s&=\alpha^2[1-2\tlam\tth_1+\tlam^2\tth_2-\pi^2(1-\lam\theta_1)^2]\\
\lim_{d\to\infty} V_l&=0\\
\lim_{d\to\infty} V_i&=\alpha^2\pi(1-\pi)(1-\lam\theta_1)^2\\
\lim_{d\to\infty} V_{sl}&=\sigma^2\delta(\tth_1-\tlam\tth_2)
\end{align*}
\begin{align*}
\phantom{1em}\qquad\lim_{d\to\infty} V_{li}&=0\\
\lim_{d\to\infty}V_{si}&=\alpha^2[\pi(1-2\lam\theta_1+\lam^2\theta_2+(1-\pi)\delta(\theta_1-\lam\theta_2))\nonumber\\&-\pi(1-\pi)(1-\lam\theta_1)^2
-1+2\tlam\tth_1-\tlam^2\tth_2]\\
\lim_{d\to\infty}V_{sli}&=\sigma^2\delta[\pi(\theta_1-\lam\theta_2)-(\tth_1-\tlam\tth_2)].
\end{align*}
\end{theorem}

{\bf Remark.} Theorem \ref{sobolthm} shows that the label noise does not contribute to the variance via a main effect (because $\lim_{d\to\infty} V_l=0$), but instead  through its interaction effects with the sample and initialization {(the terms $V_{sl},V_{sli}$)}. These can be arbitrarily large if we let $\sigma\to\infty$. In our simulations, we assume a reasonable amount of label noise, e.g., $\sigma=0.3\alpha$.

\subsubsection{Ordered Variance Decompositions}
\label{ordvar}

Using Theorem 
\ref{sobolthm}, we can calculate all six variance decompositions corresponding to the ordering of the sources of randomness. Namely, suppose that we decompose the variance in the order $(a,b,c)$, where $\{a,b,c\}=\{s,l,i\}$, i.e., we calculate the following three terms:
\begin{align*}
\Sigma_{abc}^{a}&:=\E_{\theta,x}\E_{a,b,c}[\hat{f}(x)-\E_a \hat{f}(x)]^2\\
\Sigma_{abc}^{b}&:=\E_{\theta,x}\E_{b,c}[\E_{a}\hat{f}(x)-\E_{a,b} \hat{f}(x)]^2\\
\Sigma_{abc}^{c}&:=\E_{\theta,x}\E_{c}[\E_{a,b}\hat{f}(x)-\E_{a,b,c} \hat{f}(x)]^2.
\end{align*} 
We can interpret these as follows: (1) $\Sigma_{abc}^a$ is all the variance related to $a$. (2) $\Sigma_{abc}^b$ is all the variance related to $b$ after subtracting all the variance related to $a$ in the total variance. (3) $\Sigma_{abc}^c$ is the part of the variance that depends only on $c$.  Then, simple calculations show
\begin{align*}
\Sigma_{abc}^{a}&=V_a+V_{ab}+V_{ac}+V_{abc}\\
\Sigma_{abc}^{b}&=V_{bc}+V_b\\
\Sigma_{abc}^{c}&=V_c.
\end{align*}

In previous work,  \cite{dascoli2020double} considered the decomposition in the order label - initialization - sample $(l,i,s)$, which becomes in our case (canceling the terms that vanish) 
\begin{align*}
\Sigma_{lis}^l&=V_{ls}+V_{lsi}+V_l+V_{li}=V_{ls}+V_{lsi}\\
\Sigma_{lis}^i&=V_i+V_{si}\\
\Sigma_{lis}^{s}&=V_s.
\end{align*}

\cite{dascoli2020double} argued that the label and initialization noise dominate the variance. However, different  decomposition orders can lead to qualitatively different results. We take the decomposition order $(l,s,i)$ as an example. Roughly speaking, in this decomposition, $\Sigma_{lsi}^{l}$ and $\Sigma_{lsi}^{s}$ can be interpreted as the variance introduced by the data set given a fixed initialization (and model), and $\Sigma_{lsi}^{i}$ is the variance of the initialization alone.  Figure \ref{varsdecomp1} (left, middle) shows the results
under these two different decomposition orders.

\begin{figure}[htb]
\centering
\begin{subfigure}{0.33\textwidth}
  \centering
  \includegraphics[width=\linewidth]{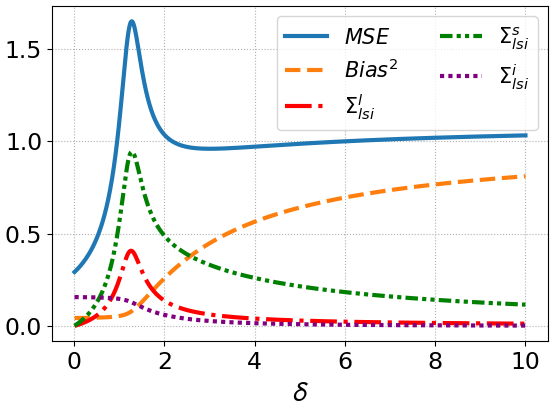}
\end{subfigure}%
\begin{subfigure}{0.33\textwidth}
  \centering
  \includegraphics[width=\linewidth]{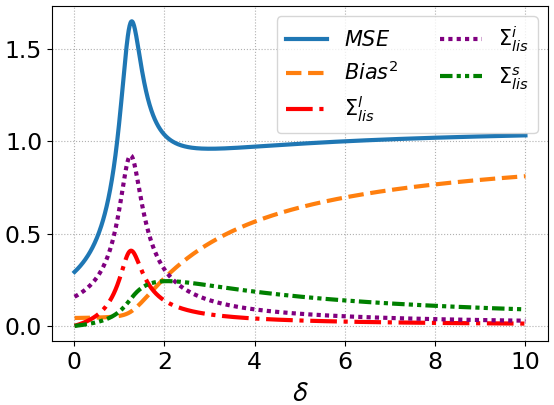}
\end{subfigure}
\begin{subfigure}{0.33\textwidth}
  \centering
  \includegraphics[width=\linewidth]{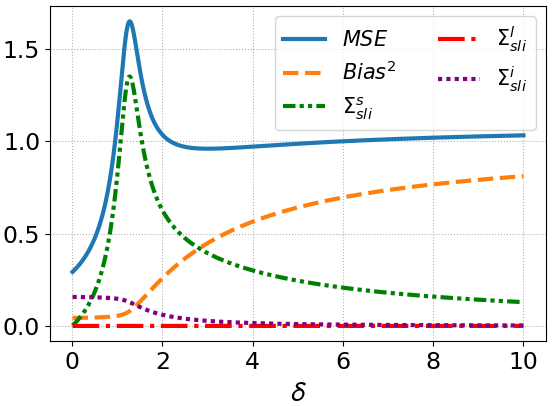}
  \end{subfigure}
\caption{Bias-variance decompositions in three orders. {\bf Left}: Decomposition order: label, sample, initialization $(l, s, i)$; dominated by $\Sigma^s_{lsi}$ (``samples"). {\bf Middle}: Decomposition order: label, initialization, sample $(l, i, s)$; dominated by $\Sigma^i_{lis}$ (``initialization"). {\bf Right}: Decomposition order: sample, label, initialization $(s,l,i)$; dominated by $\Sigma^s_{sli}$ (``samples"). Parameters: signal strength $\alpha=1$, noise level $\sigma=0.3$, regularization parameter $\lambda=0.01$, parametrization level $\pi=0.8$.}
\label{varsdecomp1}
\end{figure}

Comparing the left ($lsi$) and middle ($lis$) panels of Figure \ref{varsdecomp1},  we can see that different decomposition orders indeed lead to qualitatively different results. When $\delta<2$, in $lsi$, the variance with respect to samples dominates the total variance, while in $lis$ the variance with respect to initialization dominates. Therefore, to have a better understanding of the limiting MSE of ridge models, it is preferable to decompose the variance in a symmetric and more systematic way using the variance components. In fact, the discrepancy between these two decompositions  is due to the term $V_{si}$, which dominates the variance (as discussed later) and is contained in both $\Sigma_{lsi}^{s}$ and $\Sigma_{lis}^{i}$. By identifying this key term $V_{si}$, which has not appeared in prior work, we are able to pinpoint the specific reason why the variance is large, namely the interaction between the variation in samples and initialization. Moreover, the variance with respect to samples is even larger in Figure \ref{varsdecomp1} (right), for the $(sli)$ order, since $\Sigma_{sli}^{s}$  also contains the interaction effect between the randomness of samples and labels.

\subsubsection{{A Special Ordered Variance Decomposition}}
Next, we consider a special case of the variance decomposition, in the order of label-sample-initialization. {This is advantageous because it leads to particularly simple formulas, whose monotonicity properties are particularly tractable, as explained below.} The variance decomposes as $\Vl=\Slab+\Ssam+\Sini$, where
\begin{align*}
\Slab&:=\Sigma_{lsi}^l= \E_{\theta,x}\E_{X,W,\Ep}(f_{\lambda,\mT,W}(x)-\E_{\Ep}f_{\lambda,\mT,W}(x))^2\\
\Ssam
&:=\Sigma_{lsi}^s= 
\E_{\theta,x}\E_{X,W}(\E_{\Ep}f_{\lambda,\mT,W}(x)-\E_{X,\Ep}f_{\lambda,\mT,W}(x))^2\\
\Sini&:=\Sigma_{lsi}^i= 
\E_{\theta,x}\E_{W}(\E_{X,\Ep}f_{\lambda,\mT,W}(x)-\E_{W,X,\Ep}f_{\lambda,\mT,W}(x))^2
.
\end{align*}
As above, intuitively $\Slab$ is all variance related to the label noise, $\Ssam$ is the variance related to the samples after subtracting the variance related to the label noise, and $\Sini$ is the variance due only to the initialization. The decomposition "label-init-sample" was studied in \citep{dascoli2020double}. Going beyond what was previously known, we can get explicit expressions not only for the variances (using the ANOVA decomposition), but also for the bias, and moreover prove some monotonicity and unimodality properties of these quantities when the ridge parameter $\lam$ is optimal.

\cite{yang2020rethinking} observed empirically that the variance when fitting certain neural networks can often be unimodal, and proved this for a 2-layer net similar to our setting, with Gaussian initialization $W$ and assuming $n/d\to\infty$. However, they left open the question of understanding this phenomenon more broadly, writing that ``The main unexplained mystery is the unimodality of the variance". Our result sheds further light on this problem. 

\begin{corollary}[Bias \& variance in Two-Layer Orthogonal Net---special ordering]\label{2lthm1}
Under the assumptions from Theorem \ref{sobolthm}, we have the limits
\begin{align}\lim_{d\to\infty}\Bl
=&\alpha^2(1-\pi+\lambda\pi\theta_1)^2
,\label{biasf}\\ 
\lim_{d\to\infty}\Vl
=&
\alpha^2\pi[1-\pi+(\pi-1)(2\lambda-\delta )\theta_1-\pi\lambda^2\theta_1^2+\lambda(\lambda-\delta +\pi\delta )\theta_2]+\nonumber\\&\sigma^2\pi\delta (\theta_1-\lambda\theta_2). \label{varf}
\end{align}
More specifically,
\begin{align}
\lim_{d\to\infty}\Slab(\lambda)&=\sigma^2\pi\delta (\theta_1-\lambda\theta_2)\label{varlabelf},
\\
\lim_{d\to\infty}\Ssam(\lambda)&=\alpha^2\pi\left[-\lambda^2\theta_1^2+\lambda^2\theta_2+(1-\pi)\delta (\theta_1-\lambda\theta_2)\right]\label{varsamplef},
\\
\lim_{d\to\infty}\Sini(\lambda)&=\alpha^2\pi(1-\pi)(1-\lambda\theta_1)^2\label{varinitf},\end{align}
where $\theta_i:=\theta_i(\pi\delta ,\lambda)$, $i=1,2$. 
Therefore \begin{align}\lim_{d\to\infty}\mse(\lambda)
	=&\alpha^2\left\{1-\pi+\pi\delta \left(1-\pi+\sigma^2/\alpha^2\right)\theta_1+\left[\lambda-\delta \left(1-\pi+\sigma^2/\alpha^2\right)\right]\lambda\pi\theta_2\right\}+\sigma^2\label{msef}.\end{align}
For any fixed $\delta, \pi$, the asymptotic MSE has a unique minimum at $\lambda^*:=\delta (1-\pi+\sigma^2/\alpha^2)$. (except when $\pi=1,\sigma=0$).
\end{corollary}
{\bf Remark.} 
Except for the simple formula for the bias, theorem \ref{2lthm1} is direct corollary of theorem \ref{sobolthm}, since  $\Slab,\Ssam,\Sini$ and $\mathrm{Var}$ are all sums of several variance components.

{\bf Almost sure results over random true parameter.} Above, we provide average-case results over the true parameters $\theta\sim\N(0,\alpha^2I_d/d)$. With additional work, we show below a corresponding almost sure result. For the next result, we  assume that $X$ has iid Gaussian entries. 
\begin{theorem}[Almost sure result over true parameter $\theta$]\label{fixthetathm}
For each triple $(p_d,d,n_d)$, suppose that the true parameter $\theta$ is a sample drawn from $\N(0,\alpha^2I_d/d)$. Suppose in addition that each entry of $X$ is iid standard normal. Then as $d\to\infty$, Theorems
\ref{sobolthm} and \ref{2lthm1} still hold almost surely over the selection of $\theta$.
\end{theorem}

\subsubsection{Optimal Regularization Parameter; Monotonicity and Unimodality}

In this section, we present some theoretical results about the risks when using an optimal regularization parameter. Moreover, we study the monotonicity and unimodality of certain variance components in that setting.

We can find explicit formulas for the asymptotic bias and variance at the optimal $\lambda^*$, by plugging in the expressions of $\lambda^*$, $\theta_1$ into equations (\ref{biasf}), (\ref{varf}):
\begin{align}
	\lim_{d\to\infty}\bias^2(\lambda^*) 
 	&=
 \alpha^2(1-\pi+\lambda^*\pi\theta_1)^2\nonumber\\
	&=
\alpha^2\left(\frac{\delta (1-\sigma^2/\alpha^2)-1+\sqrt{(\delta (1+\sigma^2/\alpha^2)+1)^2-4\gamma}}{2\delta }\right)^2.\label{biasexpr}
\end{align}
\begin{align*}
\lim_{d\to\infty}\var(\lambda^*)
&=-\sigma^2\pi+(\alpha^2+\sigma^2\delta )\left(\frac{\delta (2\pi-1-\sigma^2/\alpha^2)-1+\sqrt{(\delta (1+\sigma^2/\alpha^2)+1)^2-4\gamma}}{2\delta ^2}\right).
\end{align*}
\begin{align}
\lim_{d\to\infty}\mse(\lambda^*)
& =\alpha^2\left[1-\pi+ \lambda^*\pi\theta_1\right]+\sigma^2\nonumber\\
&=
\alpha^2\left(\frac{\delta (1-\sigma^2/\alpha^2)-1+\sqrt{(\delta (1+\sigma^2/\alpha^2)+1)^2-4\gamma}}{2\delta }\right)+\sigma^2.\label{mseexpr}
\end{align}


\renewcommand{\arraystretch}{1.2}
\begin{table}[htb]
\centering
	\begin{tabular}{|l|l|l|}
		\hline
		\diagbox{Function}{Variable}    & 
       parametrization $\pi = \lim p/d$ &aspect ratio $\delta= \lim d/n$  \\ \hline
       MSE & $\searrow$ &$\nearrow$\\ \hline
       $\bias^2$ & $\searrow$ & $\nearrow$  \\ \hline
      Var &\makecell[l]{ $\delta <2\alpha^2/(\alpha^2+2\sigma^2)$: $\wedge$, max\\ at $[2+\delta (1+2\sigma^2/\alpha^2)]/4.$\\
       $\delta \geq2\alpha^2/(\alpha^2+2\sigma^2)$: $\nearrow.$ \
      } &\makecell[l]{$\pi\leq0.5:$ $\searrow.$\\$\pi>0.5$: $\wedge$, max at\\ $2(2\pi-1)/[1+2\sigma^2/\alpha^2].$ } \\ \hline
      $\Slab$ &$\nearrow$ & $\wedge$: max at $\alpha^2/(\alpha^2+\sigma^2)$\\ \hline
      $\Sini$ &$\wedge$ &$\searrow$ \\ \hline
      $\Ssam$ &conjecture: $\nearrow$ or $\wedge$ & conjecture: $\wedge$ \\ \hline

	\end{tabular}
	\caption{Monotonicity properties of various components of the risk  at the optimal $\lam^*$, as a function of $\pi$ and $\delta$, while holding all other parameters fixed. $\nearrow$: non-decreasing. $\searrow$: non-increasing. $\wedge$: unimodal. Thus, e.g., the MSE is non-increasing as a function of the parameterization level $\pi$, while holding $\delta$ fixed.}
	\label{allmonotonicity}
\end{table}

\begin{figure}[htb]
	\centering
	\includegraphics[width=.49\textwidth]{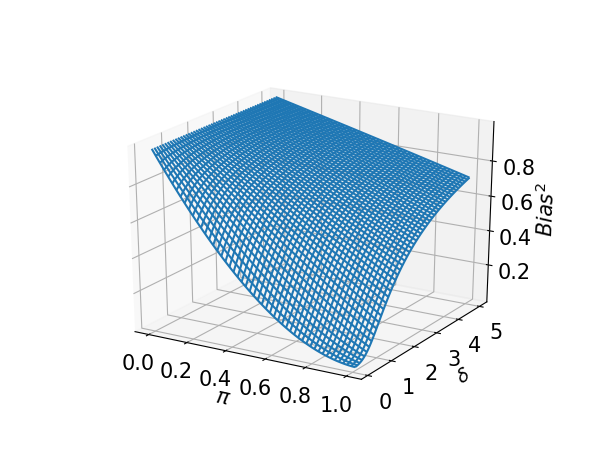}
	\includegraphics[width=.49\textwidth]{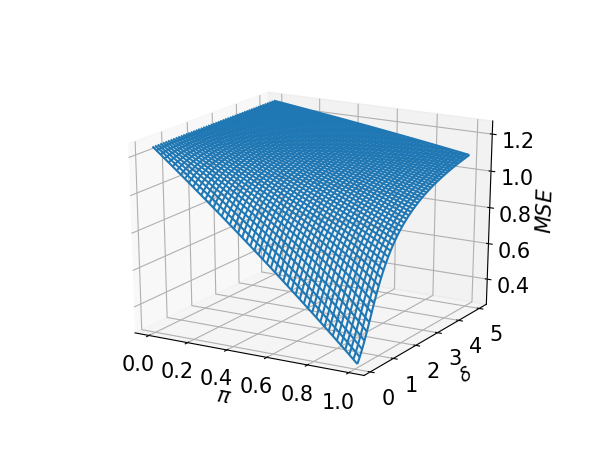}\\
	\includegraphics[width=.49\textwidth]{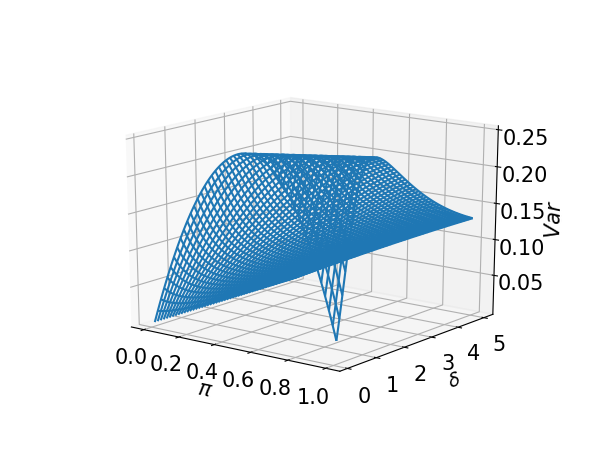}
	\includegraphics[width=.49\textwidth]{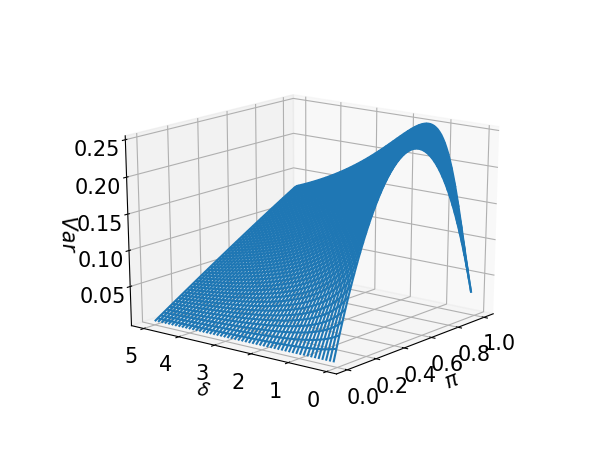}
	\caption{Perspective plots of the performance characteristics. Top row: $\bias^2$ (left), MSE (right). Bottom row: variance, from two perspectives. As functions of $\pi, \delta$, at the optimal $\lambda^*=\delta (1-\pi+\sigma^2/\alpha^2)$, when $\alpha=1$ and $\sigma=0.5$.}
	\label{fig: 2layer bias var mse}
\end{figure}

From Theorem \ref{2lthm1}, we know that the optimal ridge penalty is $\lam^*=\delta(1-\pi+\sig^2/\alpha^2)$. Thus, by plugging the expression of $\lam^*$ into \eqref{biasf}---\eqref{msef}, we are able to study the properties of the MSE, bias and variance components at the optimal $\lam^*$ as functions of $\pi,\delta$. Our results are summarized in Table \ref{allmonotonicity}. See Figure \ref{fig: 2layer bias var mse} for illustration.

As for the monotonicity and unimodality properties, we have the following statement, where the properties are summarized in Table \ref{allmonotonicity} for clarity.

\begin{theorem}[Monotonicity and unimodality]\label{2lmon}
 Under the assumptions above, the MSE, Bias, and components of the sequential variance decomposition in the $l-i-s$ order have the monotonicity and unimodality properties summarized in Table \ref{allmonotonicity}. For instance, the MSE is non-increasing as a function of the parameterization level $\pi$, while holding $\delta$ fixed.
\end{theorem}

We provide some observations below.

{\bf Consistency with prior work.} The MSE result is consistent with optimal regularization mitigating double descent, which was shown in finite samples in a certain two-layer Gaussian  model in \cite{nakkiran2021optimal}. However, our result holds for more general distributions of data matrices with arbitrary iid entries, while only proven asymptotically. 

{\bf Variance as a function of $\delta$.} For fixed parametrization level $\pi = \lim p/d$, as a function of the ``dimensions-per-sample" parameter $\delta  = \lim d/n$, the variance is monotonically decreasing when $\pi<0.5$, and unimodal with a peak at $2(2\pi-1)/[1+2{\sigma^2}/{\alpha^2}]]$ when $\pi\geq0.5$. This prompts the question why $\pi=1/2$ is special? The special role of this value was also noted in \cite{yang2020rethinking}. Recall that $d$ is the original dimension, while $p$ is the number of features in the intermediate layer, and $\pi = \lim p/d$. While the role of $\pi=1/2$  does not seem straightforward to understand, qualitatively for large $\pi$ we keep a lot of features in the inner layer. This is close to a ``well-specified" model. Thus, when we increase the size of the data set  (and thus decrease $\delta = \lim d/n$), it is reasonable that the variance decreases. In contrast, regardless of $\pi$, when we severely decrease the size of the data set   (and thus increase $\delta = \lim d/n$), the optimal ridge estimator will regularize more strongly, and thus it is possible that its variance may decrease (which is what we indeed observe). 


{\bf Variance as a function of $\pi$.} For small $\delta$, the variance is unimodal with respect to $\pi$. A possible heuristic is as follows. Recall that $\pi = \lim p/d$ ($d$ is data dimension, $p$ is number of features in inner layer) denotes the amount of ``parametrization" we allow. When $\pi\approx0$, the number of features in the inner layer is very small, which effectively corresponds to a ``low signal strength" problem (see also our added noise interpretation below). The optimal ridge estimator thus employs strong regularization, and acts like a constant estimator, thus the variance is almost zero. When $\pi=1$, we are using the correct number of features to estimate $\theta$, thus the variance is also small. The variance is zero when $\sigma=0$, and we can plot it when $\sigma>0$ (see Figure \ref{fig: 2layer bias var mse}). The above reasoning also suggests the variance may be larger for intermediate values of $\pi$. Thus, the unimodality of the variance with respect to $\pi$ is perhaps reasonable.

 Beyond our results on $\Slab$ and $\Sini$, we conjecture based on numerical experiments that $\Ssam$ is unimodal as a function of $\delta$ and can be either unimodal or monotone as a function of $\pi$. However, this appears more challenging to establish. 

{{\bf Comparison with \cite{rocks2020memorizing}.} In their paper, they suggest that the training process  $W$ should be separated from the sampling of the training data $X,\varepsilon$ when studying the variance. Thus, 
 they calculate the variance by fixing $\theta,W$, computing conditional variances (due to $X,\ep$), then taking expectation over $\theta,W$. This can, in principle, still be recovered from our general framework, if we look at the variance components conditioned on $\theta,W$. 
 In constrast, we consider the randomness arising from all components together. Our approach allows us to study some problems that do not easily fall within the scope of the conditional approach. For instance, we can provide intuition for why ensembling works; namely that it can reduce the interaction effect $V_{si}$.}

{{\bf Multiple descent.} It has been argued that  other possible shapes of the test error, such as multiple descents, can arise. \cite{liang2020multiple} study kernel regression under the limiting regimes $d\sim n^{c}, c\in (0,1)$. They provide an upper bound for the MSE and show its multiple descent behavior as $c$ increases. \cite{adlam2020neural} study the neural tangent kernel under the limiting regimes $p\sim n^{c}$, $c=1,2$ and observe that the MSE has a triple descent shape as a function of $p$. To conclude, the MSE may exhibit multiple descent when considering different asymptotic regimes. However, since we only consider the proportional limit where $p/d\to\pi,d/n\to\delta$, we have not found evidence of multiple descent in our setting.}

{\bf Remark: Fully linear regression.} The monotonicity of the MSE and bias, and the unimodality of variance at the optimal $\lam^*$ also appear in the simpler (one-layer) linear setting.
Namely, we consider the usual linear model $Y=X\theta+\Ep$, where $X\in\R^{n\times d}$ is the data matrix and $Y\in\R^{n\time 1}$ is the response. We fit a linear regression of $Y$ on $X$, which can be seen as a special case of our two layer setting with $W=I_d$ (and $d=p$). We use the same assumptions (except the assumptions on $W$) and notations as in the two layer setting. In particular, we assume $n,p\to\infty$ with $p/n\to \gamma>0$, where the aspect ratio $\gamma$ is now a measure of the parametrization level. We have the following result.

\begin{proposition}[Properties of the limiting MSE, bias \& variance in linear setting]\label{lin}  Under the same assumptions as in the two-layer setting, the limiting characteristics of optimally tuned ridge regression $(\lam^*=\ga/\alpha^2)$ have the following properties as a function of the degree of parameterization $\gamma$:
\benum
\item
 $\mse(\gamma)$ is increasing  as a function of $\gamma$.
 \item $\bias^2(\gamma)$ is increasing as a function of $\gamma$.
 \item $\var(\gamma)$ is unimodal as a function of $\gamma$, with maximum at $\gamma = \alpha^2/(\alpha^2+1)$. 
\item At the maximum, the bias equals the variance: $\bias^2[\alpha^2/(\alpha^2+1)]=\var[\alpha^2/(\alpha^2+1)]$. 
 \eenum
\end{proposition}

Figure 6 in \cite{liu2019ridge} shows the MSE, variance and bias at optimal $\lam^*$. However, that work only studied it visually, and not theoretically.  The result on the MSE has appeared before as Proposition 6 of \cite{dobriban2020wonder}, in a different context. However, the results on the bias and variance have not been considered in that work.

\subsection{Further Properties of the Bias, Variance and MSE}
In this section, we report some further properties of the bias, variances and MSE, including but not limited to the optimal ridge parameter setting.

\subsubsection{Relation Between MSE and Bias at Optimum}
\label{surp}

We present a somewhat surprising relation between MSE and bias at the optimal $\lambda^*=\lambda^*(\delta,\pi,\alpha^2,\sigma^2)$: Let $\bias^2:=\lim_{d\to\infty}\bias^2(\lambda^*)$, $\bias=|\sqrt{\bias^2}|$,  and denote $\mse$$:=\lim_{d\to\infty}$$\mse(\lambda^*)$. In general for all $\lambda$, we have that {$\mse(\lambda) = \Bl + \var(\lambda)+\sigma^2$}. {Also, in prominent problems such as in non-parametric statistics, optimal rates are achieved by balancing bias and variance \cite[e.g.,][etc]{ibragimov2013statistical}. Thus we are interested to see if the bias and variance are also balanced at the optimal $\lambda$ in our case.} However, based on the explicit expressions above, the bias and variance are balanced via the signal strength $\alpha^2$ via
\begin{align*}
\var &=\bias \cdot (\alpha-\bias).
\end{align*}
This holds for any $\pi,\delta $ and $\alpha$. Thus, the MSE and bias are linked in a nontrivial way at the optimal $\lambda$.  We see that the optimal squared bias and variance are in general not equal at the optimum. Instead, we have the above relation, which also balances the bias with the \emph{signal strength} $\alpha^2$. We think that this explicit relation is remarkable.

\subsubsection{Fixed Regularization Parameter}
\begin{figure}[htb]
\centering
\includegraphics[width=0.49\linewidth]{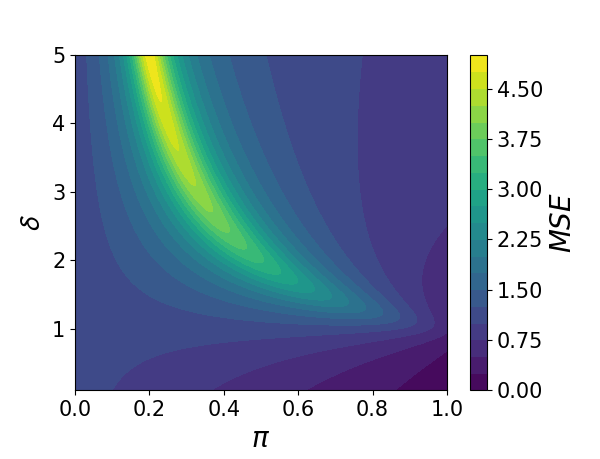}
\includegraphics[width=0.45\linewidth]{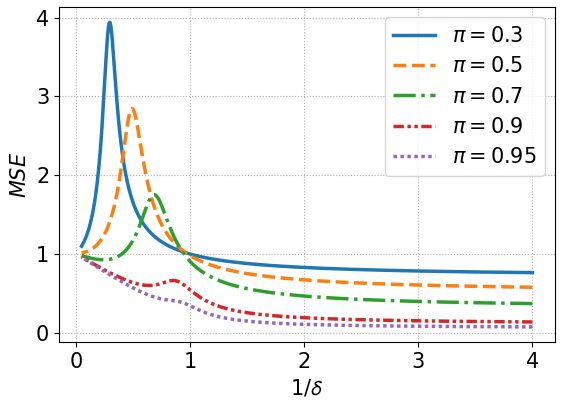}
\caption{{\bf Left}: Asymptotic MSE of ridge models when $\lambda=0.01$, $\sigma=0.3$, $\alpha=1.$ {\bf Right}: Asymptotic MSE as function of $1/\delta $ when $\lambda=0.01$, $\sigma=0.3$, $\alpha=1$. {(Note: this figure is plotted as a function of $1/\delta$, instead of $\delta$ as before. Increasing $1/\delta$ is equivalent to increasing the number of samples $n$.)}}
\label{fig: 2layer_fixed_lambda_mse}
\end{figure}

From Figure \ref{fig: 2layer bias var mse} and Theorem \ref{2lmon} above, we can see that the MSE is monotone decreasing with respect to the parametrization level $\pi=\lim p/d$ if we choose the optimal $\lambda^*$. This is consistent with ``double descent being mitigated", as in the results of \cite{nakkiran2021optimal} for a different problem. 

Here we provide additional results for a \emph{suboptimal} choice of $\lambda$. Specifically, we consider the simplest case when $\lambda$ is fixed across problem sizes. In contrast, we find that double descent is \emph{not} mitigated, and may occur when we use a small regularization parameter (also referred to as the ridgeless limit).
In Figure \ref{fig: 2layer_fixed_lambda_mse} (left), we fix $\lambda=0.01$ and plot a heatmap of the asymptotic MSE as function of the two variables $\pi = \lim p/d$ and $\delta = \lim d/n$. Clearly, the MSE is in general not monotone with respect to $\pi$ or $\delta $. Note the peak in the MSE around the curve $\gamma = \delta\pi=1$, or equivalently $\delta=1/\pi$. This corresponds to the ``interpolation threshold" where $\lim p/n =1$, and the number of learned parameters $p$ is close to the number of samples $n$. Thus, we fit just enough parameters to interpolate the data.

Besides, we see in Figure \ref{fig: 2layer_fixed_lambda_mse} (right) that double descent (which we interpret as a change of monotonicity, or a peak in the risk curve) with respect to $1/\delta =\lim n/d$ occurs when $\pi$ is suitably large, e.g., $\pi=0.9$, while the MSE is unimodal when $\pi$ is small, e.g., $\pi=0.5$. The intuition is, as in many previous works on double descent (e.g., \citealp{lecun1991second,hastie2019surprises}), that the suboptimal regularization can lead to a somewhat ill-conditioned problem, which increases the error (see also section \ref{understand_opt} for more explanation). Thus, we see that here as in prior works, using a suboptimal penalty $\lambda$ may lead to non-monotone MSE.

Moreover, we can obtain some quantitative results about the bias and variance with fixed values of the regularization parameter $\lambda$.
\begin{theorem}[Bias and variance of ridge models given a fixed $\lambda$]
Under the assumptions in our two layer setting, we have 
\begin{enumerate}
\item For any fixed $\lambda>0$, $\lim_{d\to\infty}\Bl$ is monotonically decreasing as a function of $\pi$ and is monotonically increasing as a function of $\delta$. 

\item $\lim_{\lambda\to0}\lim_{d\to\infty}\var(\lambda)=\infty$ on the curve $\delta=1/\pi$ (the interpolation threshold where $\lim p/d =1$). {More specifically, when $\lambda\to0$, $V_{si},V_{sli}$ goes to infinity while other variance components converge to some finite limits on the curve $\delta=1/\pi$.}
\end{enumerate}
\label{bias_var_for_fixed_lambda}
\end{theorem}
The first part implies that more samples or a larger degree of parameterization can always reduce the prediction bias, which is consistent with our intuition that larger models can, in principle, approximate any function better. 

{For the variance components, it is natural to expect that some interaction exists, because even the expressions $W,X$ in the prediction function $f(X) = (WX)^\top \beta$ interact non-additively. But we do not fully understand why the interaction terms $V_{si}, V_{sli}$ are large. This can be viewed as a surprising discovery of our paper.  One somewhat tautological perspective is that the interaction terms are the part of variance that are most affected by ``under-regularization". For instance, the main effect $V_{i}$ comes from the randomness of initializations. Thus, one has to average---or ensemble---over the choices of initialization $W$ to reduce $V_{i}$. However, for $V_{si}$ and $V_{sli}$, both ensembling \emph{and} optimally tuned ridge regularization can reduce their values significantly. Thus, these components seem to be more affected by the ``under-regularization" due to using a sub-optimal ridge parameter. However, this is still a somewhat circular explanation, because the entire reason that they diverge is that they are sensitive to under-regularization.
}

For any fixed $\lambda>0$, we conjecture based on numerical results that $\lim_{d\to\infty}\var(\lambda)$ is unimodal as both a function of $\pi$ and a function of $\delta$. This appears to be more challenging to show. Here the unimodality would be mainly due to being close to the interpolation threshold $p/n\approx 1$ for $\delta \approx 1/\pi$, which leads to ill-conditioned feature matrices and a large risk.


\subsubsection{Added Noise Interpretation}
\label{add_noise}

The random projection step in the initialization can be interpreted as creating additional noise. Thus, we can find a ridge model without the projection step (i.e. without the first layer) with larger training set label noise $\sigma'^2$ and the same test point label noise $\sigma^2$ that has the same asymptotic bias, variance, and MSE as the model with random projection step (i.e. with the first layer). To obtain the  ``effective noise" level $\sigma'^2$, in equation (\ref{biasexpr}), (\ref{mseexpr}), let us equate the formula determining $\lim_{d\to\infty}\mse(\lambda^*) =\alpha^2[1-\pi+ \lambda^*\pi\theta_1]+\sigma^2$ for two sets of parameters $\alpha^2,\sigma^2,\pi,\delta$ and $\alpha^2,\sigma'^2,\pi=1,\delta$. This leads to the equation
\begin{align*}
1-\pi+\lambda^*\pi\theta_1&=\lambda'^*\theta_1'.
\end{align*}
After simple calculation, we obtain \begin{equation*}\sigma'^{2}
=\sigma^2
+
\Delta
:=\sigma^2
+
\alpha^2(1-\pi)\frac{\delta (1+\sigma^2/\alpha^2)+1+\sqrt{(\delta (1+\sigma^2/\alpha^2)+1)^2-4\gamma}}{2\gamma}.\end{equation*}
Note that $\mathrm{Variance=MSE-Bias^2-\sigma^2}$, hence the random projection model has the same asymptotic bias, variance and MSE as the ordinary ridge model with additional training set label noise $\Delta$. However, the variance components are specific to the two-layer case, and do not carry over to the one-layer case.

\subsubsection{Understanding the Effect of the Optimal Ridge Penalty}
\label{understand_opt}

In this section, we provide some intuitions for why unimodality and the double descent shape appears in the MSE of ridge models when using a fixed small penalty $\lam$ (close to the ridgeless limit), and how the optimal $\lam^*$ helps eliminate the non-monotonicity of the MSE. We illustrate this with numerical results.

\begin{figure}[htb]
\centering
\begin{subfigure}{0.33\textwidth}
  \centering
  \includegraphics[width=\linewidth]{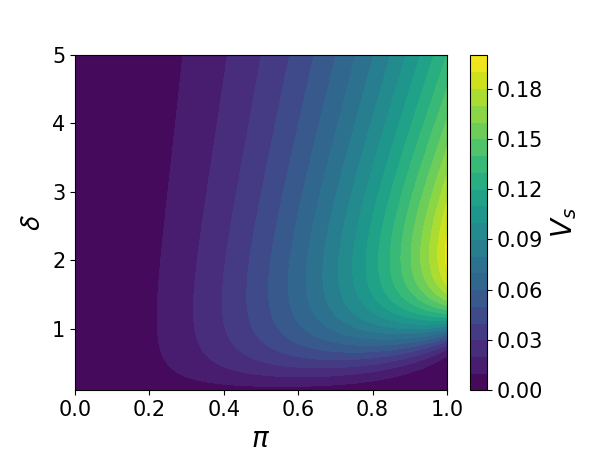}
  \caption{$V_s$}
  \label{sobol_sub_1}
\end{subfigure}%
\begin{subfigure}{0.33\textwidth}
  \centering
  \includegraphics[width=\linewidth]{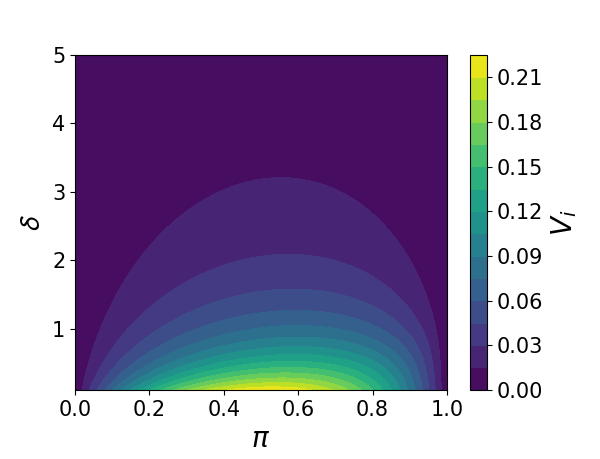}
  \caption{$V_i$}
  \label{sobol_sub_2}
\end{subfigure}
\medskip
\begin{subfigure}{0.33\textwidth}
  \centering
  \includegraphics[width=\linewidth]{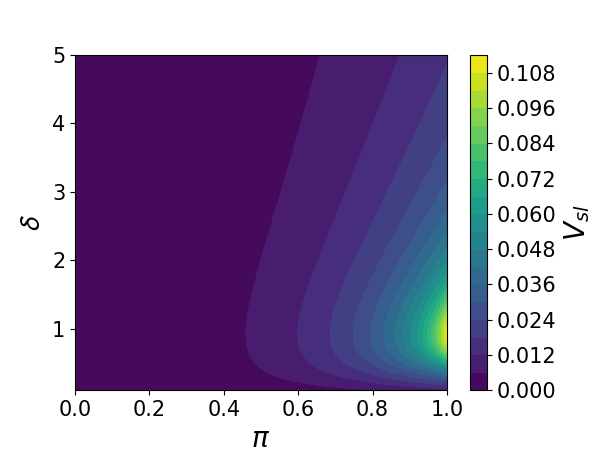}
  \caption{$V_{sl}$}
  \label{sobol_sub_3}
\end{subfigure}
\begin{subfigure}{0.32\textwidth}
  \centering
  \includegraphics[width=\linewidth]{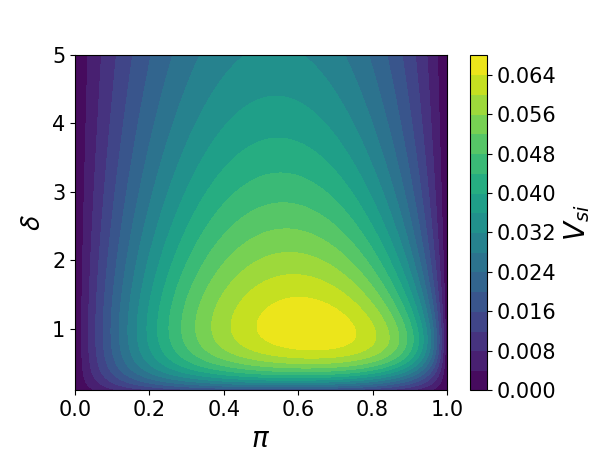}
  \caption{$V_{si}$}
  \label{sobol_sub_4}
\end{subfigure}
\begin{subfigure}{0.32\textwidth}
  \centering
  \includegraphics[width=\linewidth]{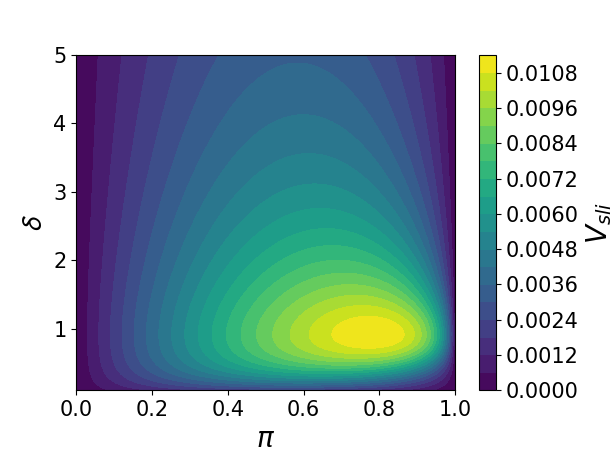}
  \caption{$V_{sli}$}
  \label{sobol_sub_5}
\end{subfigure}
\begin{subfigure}{0.32\textwidth}
  \centering
  \includegraphics[width=\linewidth]{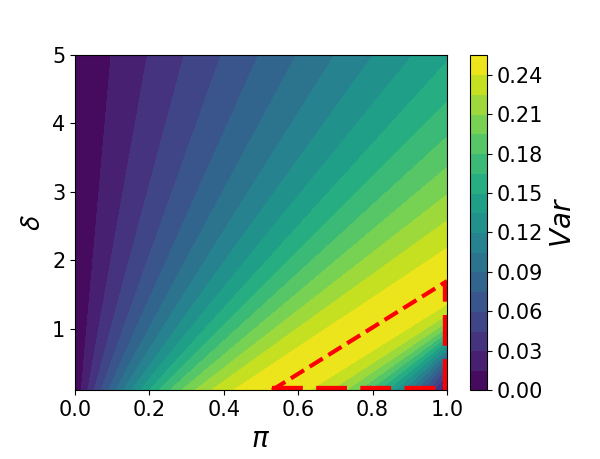}
  \caption{$\var$}
  \label{sobol_sub_6}
\end{subfigure}
\centering
\begin{subfigure}{0.33\textwidth}
  \centering
  \includegraphics[width=\linewidth]{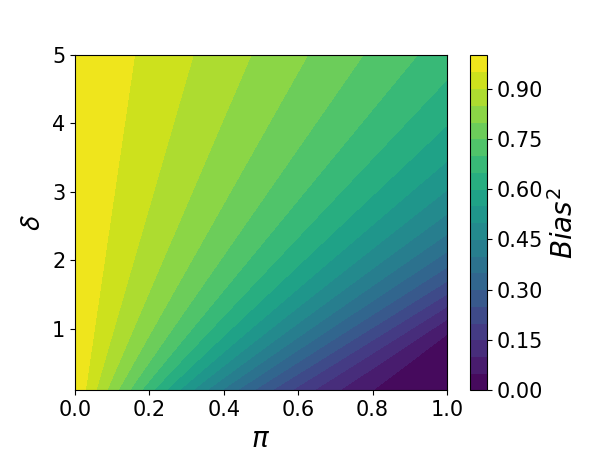}
  \caption{$\bias^2$}
  \label{sobol_sub_7}
\end{subfigure}
\begin{subfigure}{0.33\textwidth}
  \centering
  \includegraphics[width=\linewidth]{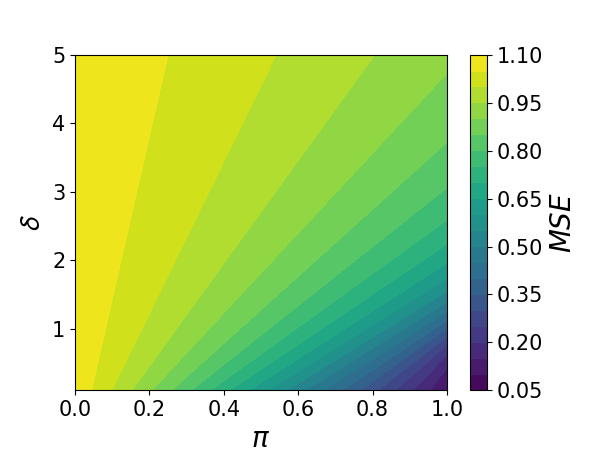}
  \caption{$\mse$}
  \label{sobol_sub_8}
\end{subfigure}

\caption{Heatmaps of the performance characteristics for the optimal regularization parameter $\lambda=\lambda^*$. Variance components, variance, bias and the MSE as functions of $\pi$ and $\delta$ when $\alpha=1,\sigma=0.3$. ($\var=V_s+V_i+V_{sl}+V_{si}+V_{sli}$. $\mse=\bias^2+\var+\sigma^2$.)}
\label{sobol_indices_2d}
\end{figure}
\begin{figure}[htb]
\centering
\begin{subfigure}{0.33\textwidth}
  \centering
  \includegraphics[width=\linewidth]{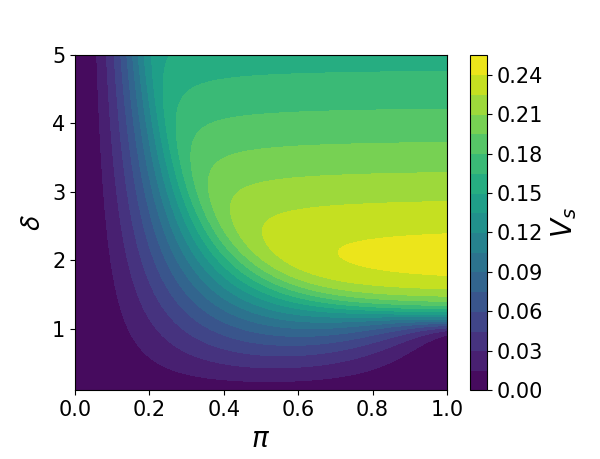}
  \caption{$V_s$}
  \label{sobol_sub_21}
\end{subfigure}%
\begin{subfigure}{0.33\textwidth}
  \centering
  \includegraphics[width=\linewidth]{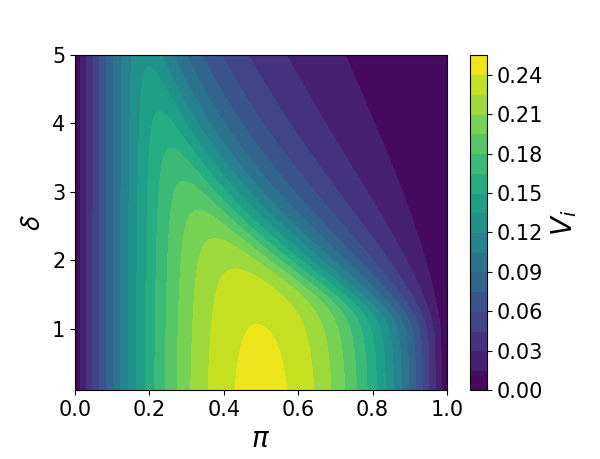}
  \caption{$V_i$}
  \label{sobol_sub_22}
\end{subfigure}
\medskip
\begin{subfigure}{0.33\textwidth}
  \centering
  \includegraphics[width=\linewidth]{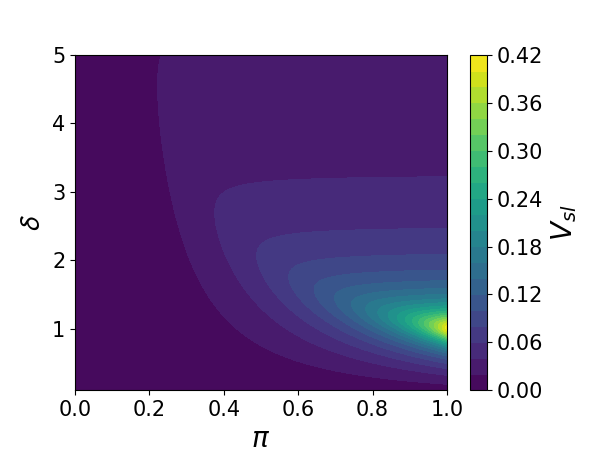}
  \caption{$V_{sl}$}
  \label{sobol_sub_23}
\end{subfigure}
\begin{subfigure}{0.32\textwidth}
  \centering
  \includegraphics[width=\linewidth]{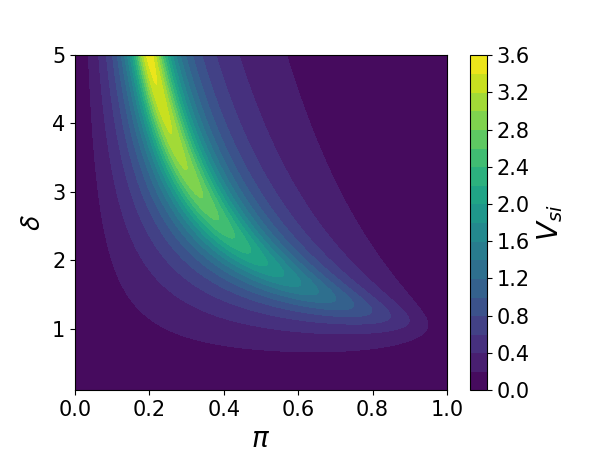}
  \caption{$V_{si}$}
  \label{sobol_sub_24}
\end{subfigure}
\begin{subfigure}{0.32\textwidth}
  \centering
  \includegraphics[width=\linewidth]{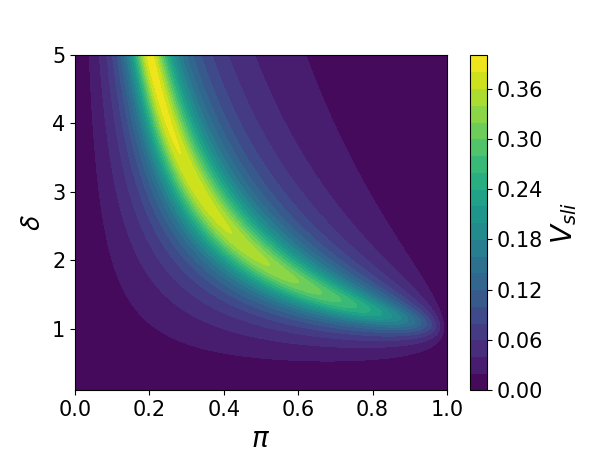}
  \caption{$V_{sli}$}
  \label{sobol_sub_25}
\end{subfigure}
\begin{subfigure}{0.32\textwidth}
  \centering
  \includegraphics[width=\linewidth]{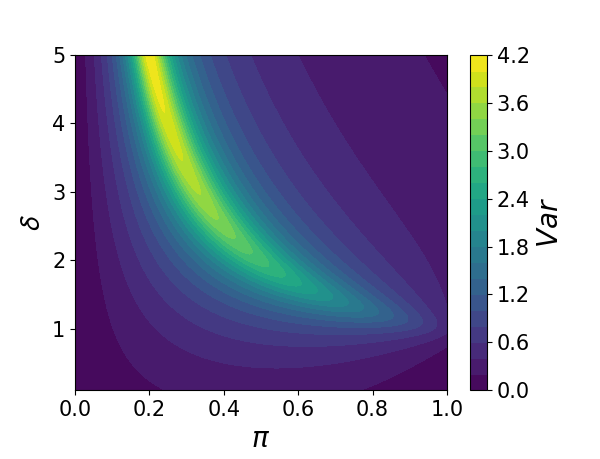}
  \caption{$\var$}
  \label{sobol_sub_26}
\end{subfigure}
\centering
\begin{subfigure}{0.33\textwidth}
  \centering
  \includegraphics[width=\linewidth]{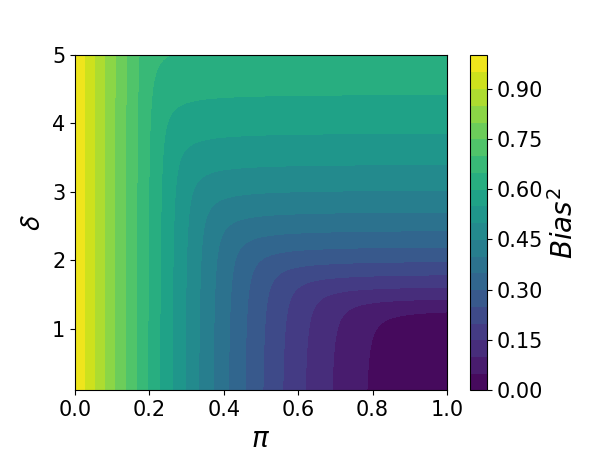}
  \caption{$\bias^2$}
  \label{sobol_sub_27}
\end{subfigure}
\begin{subfigure}{0.33\textwidth}
  \centering
  \includegraphics[width=\linewidth]{mse_fix.png}
  \caption{$\mse$}
  \label{sobol_sub_28}
\end{subfigure}

\caption{Heatmaps of the performance characteristics for a fixed parameter $\lambda=0.01$. Variance components, variance, bias and the MSE as functions of $\pi$ and $\delta$ when $\alpha=1,\sigma=0.3$. ($\var=V_s+V_i+V_{sl}+V_{si}+V_{sli}$. $\mse=\bias^2+\var+\sigma^2$.)}
\label{sobol_indices_2_2d}
\end{figure}
To qualitatively understand the effect of the optimal penalty $\lambda^*$, we plot the variance components, variance, bias and the MSE under two different scenarios. In the first scenario, we use the optimal penaly $\lambda^*$ for all ridge models (see Figure \ref{sobol_indices_2d}). It is readily verified that $V_s$ and $V_i$ contribute to a large portion of the variance, while the contributions of $V_{si}$ and $V_{sli}$ are relatively small. 

In the second scenario, we choose $\lam=0.01$ for all ridge models. From Figure \ref{sobol_indices_2_2d}, we see that, perhaps surprisingly, it is the \emph{interaction term} $V_{si}$ between sample and initialization that dominates the variance. In particular, we think that it is surprising that this interaction term can be larger than the main effects $V_s$ and $V_i$ of sample and initialization. Also, $V_{si}$ and $V_{sli}$ lead to the modes of the variance on the curve $\delta=1/\pi$ (the interpolation threshold).

 Comparing Figure \ref{sobol_indices_2d} and \ref{sobol_indices_2_2d}, we can see that $V_s, V_i$ and $V_{sl}$ are almost on the same scale in the two scenarios. However, $V_{si}$ and $V_{sli}$ are much larger  when $\lam=0.01$ than when $\lam=\lam^*$. These two terms are the main reason why the variance is significantly larger when $\lam=0.01$. Moreover, Figure \ref{sobol_sub_7} and \ref{sobol_sub_27} show that the bias is even relatively smaller when we use $\lam=0.01$ instead of the optimal $\lam^*$. Intuitively, the reason is that the optimal regularization parameter is large, to achieve a better bias-variance tradeoff, and thus makes the bias slightly larger while decreasing the variance a great amount.

Therefore, we may conclude that, under a reasonable assumption on the label noise (e.g., $\sigma=0.3\alpha$ here),
\begin{enumerate}
\item Using a fixed small penalty $\lam$ for all ridge models can lead to unimodality/double descent shape in the MSE. The modes of the MSE as a function of $\delta$ are close to the interpolation limit curve $\delta=1/\pi$.  
\item The unimodality/double descent shape of the MSE given a fixed small penalty $\lam$ is due to the variance. The bias is typically  smaller when using a fixed small penalty $\lam$ instead of the optimal penalty $\lam^*$. As mentioned, this is because the bias and variance are balanced out for the optimal $\lambda^*$, and thus we can increase the bias a bit, while significantly decreasing the variance.
\item Compared with choosing the same small ridge penalty for all models, through using the optimal penalty $\lam^*$, one can reduce the variance significantly, especially along the interpolation threshold curve. The unimodality/double decent shape of the MSE will vanish as a result; but the variance itself may still be unimodal.
\item Using the optimal penalty for all ridge models reduces the variance mainly by reducing the interaction component $V_{si}$. {This component is large in an absolute sense, and thus a reduction has a significant effect. The component $V_{sli}$ is also reduced in a relative sense; however, because it is of a smaller magnitude, this reduction has a more limited effect.}
\item 
There is a special region where the bias and variance (for the optimal $\lambda^*$) change in the same direction, in the sense that increasing the parametrization $\pi = \lim p/d$ or decreasing the aspect ratio $\delta = \lim d/n$ decrease \emph{both} the bias and the variance. See Figures \ref{sobol_sub_6} and \ref{sobol_sub_7}.

This special region is characterized by the ``triangle" $0<\pi\le1$, $\delta>0$, with $\delta  \le 2(2\pi-1)/[1+2\sigma^2/\alpha^2]$. In finite samples, this is approximated by the inequality $d/n \le 2(2p/d-1)/[1+2\sigma^2/\alpha^2]$ between the sample size $n$, data dimension $d$ and the number of parameters $p$. This can be interpreted as saying---for instance---that the parameter dimension $p$ should be large enough. Thus, in that setting, with more parametrization we can get simultaneously better bias and better variance. In a sense, this can indeed be viewed as a blessing of overparametrization.  
\item  There is a ``hotspot"  around $\pi=1/2$, where in $V_i$, and $V_{si}$ both take large values (see Figures \ref{sobol_sub_2} and \ref{sobol_sub_4}). The variance due to initialization is large for intermediate values of the projection dimension $p$. Roughly speaking, one can consider an analogy with Bernoulli random variables, which have large variance for intermediate values of the success probability.

\item For fixed $\lam$, when $\delta<1$ ($d<n$), numerical experiments show that the MSE is decreasing as $\pi$ increases, which means that more parametrization can always give us small MSE when we have enough samples.  See Figure \ref{sobol_sub_28}. It appears that there may be no double descent for fixed $\lambda$ when $\delta$ is sufficiently small; however investigating this is beyond our current scope.
\end{enumerate}


Recall that, in the noiseless case, $V_{s}$ can be interpreted as the variance of an ensemble estimator, and $V_{si}+V_{i}$ is the variance that can be reduced through ensembling. Therefore, the unimodality/double descent shape in the MSE is not intrinsic, and can be removed through regularization techniques such as ensembling, (consistent with \citealt{dascoli2020double}) or optimal ridge penalization.

\subsection{Ridge is Optimal}

We have obtained precise asymptotic results for optimally tuned ridge regression. However, is ridge regression optimal, or are there other methods that outperform it? In fact, we can prove that the ridge estimator is asymptotically optimal.
\begin{theorem}[Ridge is optimal]\label{ridgeopt}
Suppose that the samples are drawn from the standard normal distribution, i.e., $x$ and $X$ both have i.i.d. $\N(0,1)$ entries. Given the projection $W$, projected matrix $XW^\top $ and response $Y$, we define the optimal regression parameter $\beta_{opt}$ as the one minimizing the MSE over the posterior distribution $p(\theta|XW^\top ,W,Y)$ of the parameter $\theta$,
\begin{align}
	\beta_{opt}:&=\argmin_{\beta}\E_{p(\theta|XW^\top ,W,Y)}\E_{x,\ep}[(Wx)^\top \beta-(x^\top \theta+\ep)]^2, \label{optdef}
\end{align} 
where  $x\sim \N(0,I_d)$, $\ep\sim \N(0,\sigma^2)$ and $x$, $\ep$ are independent. We will check that this can be expressed in terms of the posterior of $\theta$ as
\begin{equation}\beta_{opt}=W\cdot\E_{p(\theta|XW^\top ,W,Y)}\theta.\label{optexpr}\end{equation} 
The optimal ridge estimator $\hbeta=(n^{-1}WX^\top  XW^\top+\lambda^* I_p)^{-1}WX^\top  Y/n$ (Theorem \ref{2lthm1}) satisfies the almost sure convergence in the  mean squared error
\begin{align}
	\lim_{d\to\infty}\E_{XW^\top ,W,Y}\|\hbeta-\beta_{opt}\|_2^2=0,  \label{asymopt}
\end{align}
and is thus asymptotically optimal.
Here $d\to\infty$ means $p,d,n\to\infty$ proportionally as in Theorem \ref{2lthm1}.
\end{theorem}
{\bf Remark.}
 In Theorem \ref{ridgeopt}, the optimal parameter $\beta_{opt}$ minimizes the mean squared error over the posterior of $\theta$ given the projection $W$, projected matrix $XW^\top $ and response $Y$.  From the proof of Lemmas \ref{2llm2}, \ref{2llm3}, we know that $\E\|\hbeta\|^2$ converges to some positive constant as $d\to\infty$. Thus, $\hbeta$ has a constant scale as $d\to\infty$, and the result that $\|\hbeta-\beta_{opt}\|^2\to0$ of Theorem \ref{ridgeopt} shows that $\beta_{opt}$ is indeed non-trivially well approximated. This result implies the asymptotic optimality of ridge regression. 

In addition, if we are given the original data matrix $X$ instead of $XW^\top $, then from the optimality of ridge regression in ordinary linear regression with Gaussian prior and noise, we have $\beta_{opt}=W(n^{-1}X^\top  X+d\sigma^2 I_p/[n\alpha^2])^{-1}X^\top  Y/n$ and ridge regression \emph{over projected data} is not asymptotically optimal. However, in our two-layer model, we only exploit the information of $X$ through $XW^\top $, thus it is reasonable to consider the situation above, in which we are only given $XW^\top $.

\section{Nonlinear Activation}
\label{nonlin} 

It is also possible to  consider the bias-variance decomposition for a two-layer neural network with certain scalar  nonlinear activation functions $\sigma(x)$ after the first layer. Namely, suppose that the data are generated through the same process, but instead of using a two-layer linear network, we use 
\begin{align}\label{nonlineardef}
\hat{f}(x)=\sigma(Wx)^\top\beta
\end{align} as the predictor. 
Here $\sigma:\R\to\R$ is an activation function applied to $Wx$ entrywise.
As before, we assume $W\in\R^{p\times d}$ has orthonormal rows, so $p\le d$,  and we only train $\beta$. This can be viewed as a random features model. We apply ridge regression to estimate $\beta$, therefore our prediction function is
\begin{align}\label{nonlinearexpdef}
\hat{f}(x)&:=\sigma(Wx)^\top\hat{\beta}
=\sigma(x^\top W^\top)\left(\frac{\sigma(WX^\top)  \sigma(XW^\top) }{n}+\lam I_p\right)^{-1}\frac{\sigma(WX^\top)Y}{n}.\end{align}
For simplicity, we further assume that $\E\sigma(Z)=0$, where $Z\sim\N(0,1)$ is a standard normal random variable. The results for activation functions with arbitrary mean can be obtained through similar techniques, but are much more cumbersome. This assumption does not capture the ReLU activation function $\sigma_+(x) = \max(x,0)$, but it can handle the function $\sigma_+(x) - \E\sigma_+(Z)$, which only differs from the ReLU by a constant. In particular, the mean of our prediction function $\hat f(x)$ with the current restriction is always zero, i.e., the prediction function does not have an intercept term. In our model, the true regression function $f^*(x)= \theta^\top x$ does not have an intercept term either; thus we think that the zero-mean restriction may not be significant in the current setting. 

Moreover, we suppose that there are constants $c_1,c_2>0$ such that  $\sigma,\sigma'$ grow at most exponentially, i.e.,  $|\sigma(x)|,|\sigma'(x)|\leq c_1e^{c_2|x|}$. Define the moments
\begin{align}
&\mu:=\E Z\sigma(Z), 
&v:=\E \sigma^2(Z),   \label{nlmu1vdef}
\end{align}
where $Z\sim\N(0,1)$. Also, we suppose that the samples are drawn from the standard normal distribution $\N(0,I_d)$, i.e., $X$ and $x$ both have i.i.d. $\N(0,1)$ entries. As before, we can write down the MSE, bias and variance:
\begin{align}
\mse(\lam)&:=\E_{\theta,x,W,X,\Ep}(\hat{f}(x)-x^\top\theta)^2+\sigma^2\nonumber\\
\bias^2(\lam)&:=\E_{\theta,x}(\E_{W,X,\Ep}\hat{f}(x)-x^\top\theta)^2\nonumber\\
\var(\lam)&:=\E_{\theta,x,W,X,\Ep}(\hat{f}(x)-\E_{X,W,\Ep}\hat{f}(x))^2.\nonumber\end{align}
Our main result in this section gives asymptotic formulas for their limits.
\begin{theorem}[Bias-Variance Decomposition for two-layer nonlinear NN]\label{nlbiasvardecomp}
Under \\the previous assumptions (i.e., in the setting of Theorem \ref{sobolthm}), with the further assumption that the samples are drawn from $\N(0,I_d)$, i.e., $x$ and $X$ have i.i.d. $\N(0,1)$ entries,  we have the following 
limits for the bias, variance, and mean squared error. Recall that we have an $n\times d$ feature matrix $X$ and a two-layer nonlinear neural network $f(x) = \sigma(Wx)^\top \beta$, with $p$ intermediate activations, and $p\times d$ orthogonal matrix $W$ of first-layer weights with $WW^\top = I_p$. Here $n,p,d\to\infty$ and $p/d\to \pi\in(0,1]$ (parametrization level), $d/n\to\delta>0$ (data aspect ratio), with $\alpha^2$ the signal strength, $\sigma^2$ the noise level, $\lambda$ the regularization parameter, $\theta_i$ the resolvent moments, and $\mu, v$ the Gaussian moments of the activation function $\sigma$ from  \eqref{nlmu1vdef}. Then
\begin{align}
\lim_{d\to\infty}\mse(\lam)&=
\alpha^2\pi\left[\frac{1}{\pi}-1+\delta(1-\pi)\theta_1+\frac{\lam}{v}\left(\frac{\lam\mu^2}{v^2}-\delta(1-\pi)\right)\theta_2\nonumber\right.\\&\left.+{(v-\mu^2)}\left(\frac{\ga}{v}\theta_1+\frac{1}{v}-\frac{\lam\ga}{v^2}\theta_2\right)\right]+\sig^2\ga\left(\theta_1-\frac{\lam}{v}\theta_2\right)+\sigma^2,\label{nlmsef}\\
\lim_{d\to\infty}\bias^2(\lam)&=\alpha^2\left[\pi\frac{\mu^2}{v}\left(1-\frac{\lam}{v}\theta_1\right)-1\right]^2, \label{nlbiasf}\\
\lim_{d\to\infty}\var(\lam)
&=\alpha^2\pi\left[\frac{2\mu^2}{v}-1+\left(-\frac{2\lam\mu^2}{v^2}+\delta(1-\pi)\right)\theta_1+\frac{\lam}{v}\left(\frac{\lam\mu^2}{v^2}-\delta(1-\pi)\right)\theta_2\nonumber\right.\\&\left.-\frac{\pi\mu^4}{v^2}\left(1-\frac{\lam}{v}\theta_1\right)^2+{(v-\mu^2)}\left(\frac{\ga}{v}\theta_1+\frac{1}{v}-\frac{\lam\ga}{v^2}\theta_2\right)\right]+\sig^2\ga\left(\theta_1-\frac{\lam}{v}\theta_2\right), \label{nlvarf}
\end{align}
where $\theta_1:=\theta_1(\ga,\lam/v)$, $\theta_2:=\theta_2(\ga,\lam/v)$, $\ga=\pi\delta$. Similar to the linear case, the limiting MSE has a unique minimum at $\lam^*:=\frac{v^2}{\mu^2}\left[\delta(1-\pi+\sigma^2/\alpha^2)+\frac{(v-\mu^2)\ga}{v}\right].$
\end{theorem}
{\bf Remarks.} (1). When expanding the function $\sigma(x)$ in the Hermite polynomial basis, $\mu$ is the coefficient of the second basis function $x$, and $\sqrt{v}$ is $\sigma(x)$'s norm in the Hilbert space. Thus $v\geq\mu^2$ and the equality holds iff $\sigma(x)=kx$. (2). When $\sigma(x)=kx$ (i.e. $v=\mu^2$), the results in theorem \ref{nlbiasvardecomp} reduce to those in theorem \ref{2lthm1}.

\renewcommand{\arraystretch}{1.3}
\begin{table}[htb]
\centering
	\begin{tabular}{|l|l|l|}
		\hline
		\diagbox{Function}{Variable}    & 
       parametrization $\pi = \lim p/d$ &aspect ratio $\delta= \lim d/n$  \\ \hline
       $\mse$ & $\searrow$ &$\nearrow$\\ \hline
       $\bias^2$ & $\searrow$ & $\nearrow$  \\ \hline
      $\var$ &\makecell[l]{ $\delta <2\dfrac{\mu^2}{v}\left(2\dfrac{\mu^2}{v}-1\right)/\left(1+2\sigma^2/\alpha^2\right)$: $\wedge$, max\\ at $\dfrac{v}{\mu^2}\left[2+\dfrac{\delta v}{\mu^2} \left(1+\dfrac{2\sigma^2}{\alpha^2}\right)\right]/4.$\\
       $\delta \geq2\dfrac{\mu^2}{v}\left(2\dfrac{\mu^2}{v}-1\right)/\left(1+2\sigma^2/\alpha^2\right)$: $\nearrow.$ 
      } 
      &\makecell[l]{$\pi\leq \dfrac{v}{2\mu^2}:$ $\searrow.$\\$\pi>\dfrac{v}{2\mu^2}$: $\wedge$, max at\\ $\dfrac{2\mu^2(2\pi\mu^2/v-1)}{v(1+2\sigma^2/\alpha^2)}.$ } \\ \hline
	\end{tabular}
	\caption{Monotonicity properties of various components of the risk for a two-layer network with nonlinear activation  at the optimal $\lam^*$, as a function of $\pi$ and $\delta$, while holding all other parameters fixed. $\nearrow$: non-decreasing. $\searrow$: non-increasing. $\wedge$: unimodal. Thus, e.g., the MSE is non-increasing as a function of the parameterization level $\pi$, while holding $\delta$ fixed.}
	\label{allmonotonicitynonlinear}
\end{table}

Also, we have  monotonicity properties similar to in the linear case (Table \ref{allmonotonicitynonlinear}):
\begin{theorem}[Monotonicity and unimodality for non-linear net at optimal $\lam^*$]\label{2lmonnl}
 Under the assumptions from Theorem \ref{nlbiasvardecomp}, for the optimal $\lambda=\lam^*$, the MSE, Bias, and variance have the monotonicity and unimodality properties summarized in Table \ref{allmonotonicitynonlinear}. The MSE and bias are decreasing as a function of the parametrization level $\pi$, and increasing as a function of the data aspect ratio $\delta$. The variance is either monotone or unimodal depending on the setting.
\end{theorem}

Thus, comparing Tables \ref{allmonotonicitynonlinear} and \ref{allmonotonicity}, we see that with stronger Gaussian assumptions on the data distribution, optimal ridge regularization has similar effects in the nonlinear and linear cases. For instance, $\lam^*$ can eliminate the ``peaking" shape of the MSE.

\section{Numerical Simulations}
\label{simu}

In this section, we perform several numerical experiments, to check the correctness of our theoretical results.
\subsection{Verifying the Theoretical Results for the MSE}\label{mseest}
\begin{figure}[htb]
\centering
\includegraphics[width=0.49\linewidth]{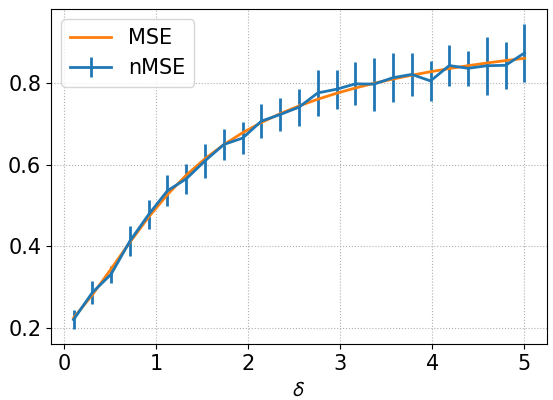}
\includegraphics[width=0.49\linewidth]{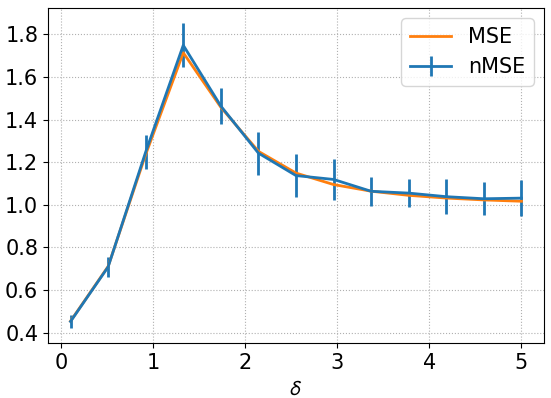}
\caption{Numerical verification of the theoretical results for MSE. We display, as a function of $\delta=\lim d/n$, the theoretical formula and the numerical mean and standard deviation over 20 repetitions. Parameters: $\alpha=1,\sigma=0.3,\pi=0.8,n=150,d=\lfloor n\delta \rfloor ,p=\lfloor d\pi \rfloor.$ {\bf  Left}: linear, $\sigma(x)=x$, $\lam=\lam^*$(optimal). {\bf Right}: nonlinear, $\sigma(x)=\sigma_+(x)-\E_{x\sim\N(0,1)}\sigma_+(x)$, $\sigma_+(x)=\max(x,0)$, $\lam=0.01$.}
\label{diremse}
\end{figure}
To check the correctness of the MSE formula, we estimate the MSE from its definition directly. For simplicity, we subtract the test point label noise $\sigma^2$ from the MSE formula. We randomly generate $k=400$ i.i.d. tuples of random variables $(x_i,\theta_i,\ep_i,X_i,W_i)$, $1\leq i\leq k$, from their assumed distributions (we assume $X$ and $x$ have i.i.d. $\N(0,1)$ entries in numerical simulations), and estimate the MSE by calculating:
\begin{align*}
\widehat{\MSE}=\frac{1}{k}\sum_{i=1}^{k}(\hat{f}_i(x_i)-x_i^\top \theta_i)^2,
\end{align*}
where $k=400$ and $\hat{f}_{i}(x_i)=\sigma(x_i^\top W_i^\top)\left(n^{-1}\sigma(W_iX_i^\top)\sigma(X_iW_i^\top)+\lam I_p\right)^{-1}n^{-1}\sigma( W_iX_i^\top) (X_i\theta_i+\Ep_i)$, for $\sigma(x)$ both linear and nonlinear. We repeat this process $20$ times and plot the mean and standard error in Figure \ref{diremse}. The regularization parameter $\lambda$ is set optimally or fixed. We also plot our theoretical MSE formula from \eqref{msef}. The parameters in the experiment are shown in the captions of the figure. We can see from Figure \ref{diremse} that our theoretical prediction of the MSE is quite accurate.

\subsection{Bias-variance Decomposition and the Variance Components}\label{deriest}
\begin{figure}[htb]
\includegraphics[width=0.48\linewidth]{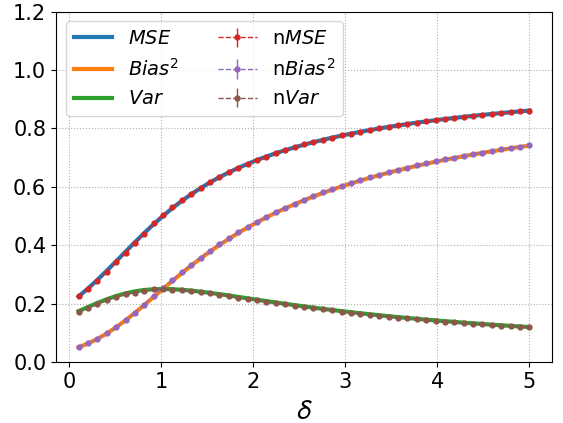}
\includegraphics[width=0.50\linewidth]{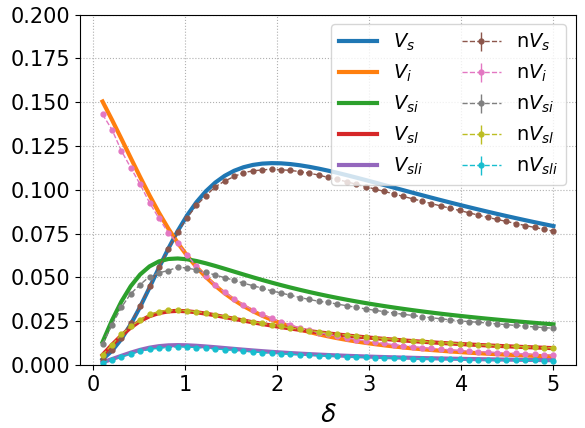}
\caption{{\bf Left}: numerical simulation verifying the accuracy of the bias, variance and MSE formulas. {\bf Right}: simulations with variance components. For each ANOVA component, symbolized by $\star$, we show two curves: $\star$: theory, $n\star$: numerical (averaged over 5 runs). Parameters: $\alpha=1,\sigma=0.3,\pi=0.8,n=150,d=\lfloor n\delta \rfloor ,p=\lfloor d\pi \rfloor.$}
\label{num_bvms}
\end{figure}
We next study the accuracy of the formulas for the bias, variance and the variance components in the linear case. Estimating them directly requires many samples. For example, to estimate the bias based on the defining formula from Section \ref{bvsob}, we may need to generate, say, $100$ pairs of $(x,\theta)$, and for each $(x,\theta)$ generate $500$ triples of i.i.d. $(X,W,\Ep)$. Thus, we may need to simulate $50,000$ samples in total to obtain accurate results. This is beyond our current scope.
\begin{table}[htb]
\centering
  \begin{tabular}{|l|l|}
    \hline
    {Functional} &{Estimator}  \\ \hline
       $\E\tr(\cM\cM^\top)$ & $\dfrac{1}{k}\sum\limits_{i=1}^{k} \tr(\cM_i\cM_i^\top)$\\ \hline
       $\|\E \cM\|_{F}^2$ & $\left\|\dfrac{1}{k}\sum\limits_{i=1}^{k} \cM_i\right\|_F^2$\\ \hline
       $\E_{W}\|\E_X \cM\|_{F}^2$ & $\dfrac{1}{k}\sum_{j=1}^{k}\left\|\dfrac{1}{k}\sum\limits_{i=1}^{k} \cM_{ij}\right\|_F^2$\\ \hline 
       $\E_{X}\|\E_W \cM\|_{F}^2$ & $\dfrac{1}{k}\sum_{i=1}^{k}\left\|\dfrac{1}{k}\sum\limits_{j=1}^{k} \cM_{ij}\right\|_F^2$\\ \hline                
  \end{tabular}
  \caption{Empirical estimators of functionals of interest. Here $\cM$ is a generic matrix that can be  $M$ or $\tM$. For the bias, variance and the MSE, we take $k=100$ and $\cM_i$ denotes the appropriate matrix $\cM$ obtained from the $i$-th pair of $(X,W)$. For variance components, $k=20,50$ and $\cM_{ij}$ denotes the appropriate matrix $\cM$ obtained from the $i$-th $X$ and the $j$-th $W$. Estimators of the quantities in \eqref{2lbias}---\eqref{2lmse}, \eqref{sobolvsf}---\eqref{sobolvslif} are obtained by combining the above.}
  \label{numerical_details}
\end{table}
Therefore, for simplicity, we check instead the formulas that we have derived in Appendix \ref{all_proofs} in equations \eqref{2lbias}---\eqref{2lmse}, \eqref{sobolvsf}---\eqref{sobolvslif}.  We omit the results for $V_l$ and $V_{li}$ since they converge to $0$.
In all experiments, we choose $n=150$. For the bias, variance, and MSE,  we generate $100$ i.i.d. copies of $X,W$ of certain dimensions (we assume $X$ has i.i.d. $\N(0,1)$ entries in numerical simulations)  and estimate the expectations from the proof of Theorem \ref{2lthm1} in Appendix \ref{all_proofs} (\ref{2lbias}---\ref{2lmse}) using the Monte Carlo mean. As for the variance components,  we randomly generate $k$ i.i.d. $X_i$-s and $W_j$-s, and use them to form $k^2$ pairs of $(X_i,W_j)$, where $k=20$ in Figure \ref{num_bvms} (right) and $k=50$ in Figure \ref{filled} (left). Similarly, we also estimate the expectations from the proof of Theorem \ref{sobolthm} in \eqref{sobolvsf}---\eqref{sobolvslif} via the Monte Carlo  mean (see the details in Table \ref{numerical_details}).
{}

Figure \ref{num_bvms} shows the results averaged over five runs.  We can see that the numerical results are quite close to the theoretical predictions. Moreover, the standard deviations over $5$ runs are uniformly less than $0.001$ in all settings we considered, which implies that the variance due to the randomness of $(X,W)$ is negligible. The slight discrepancies between the theoretical predictions and experiments are mostly owing to the bias in our estimators (e.g., the second, third and fourth estimators in Table \ref{numerical_details} are biased, because they are of the form $g(\E M)$, which is estimated by $g(k^{-1}\sum_{i=1}^k M_i)$, and $g$ is nonlinear), and they can be reduced if we have more samples $(X,W)$.  

One may wonder: Why is the standard deviation of different runs of the simulation so small (e.g., less than $0.001$)? The reason is that the Marchenko-Pastur theorem, on which our theoretical predictions depend, has fast convergence rates \cite[e.g.,][]{bai1993convergence,gotze2004rate}.
Since the terms we estimated are expectations of various functionals of the eigenvalue spectrum, we can expect the simulation results to be quite precise.

\subsection{General Covariance}
\begin{figure}[htb]
\includegraphics[width=0.49\linewidth]{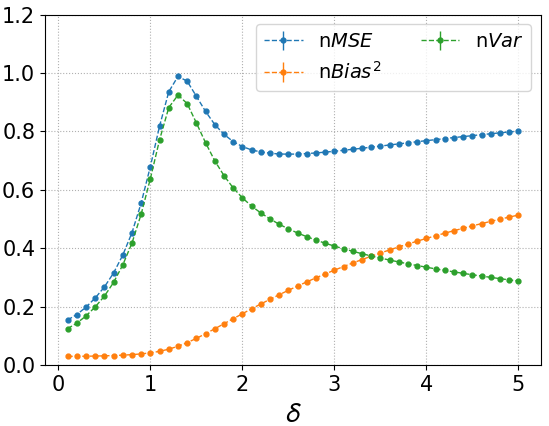}
\includegraphics[width=0.5\linewidth]{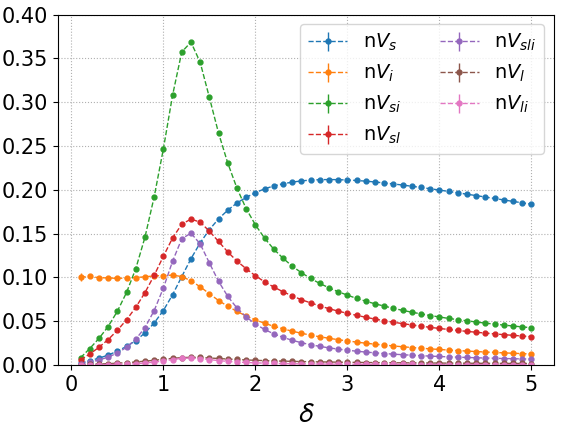}
\includegraphics[width=0.49\linewidth]{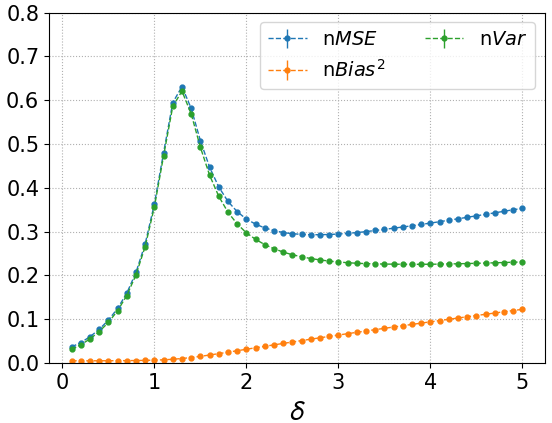}
\includegraphics[width=0.5\linewidth]{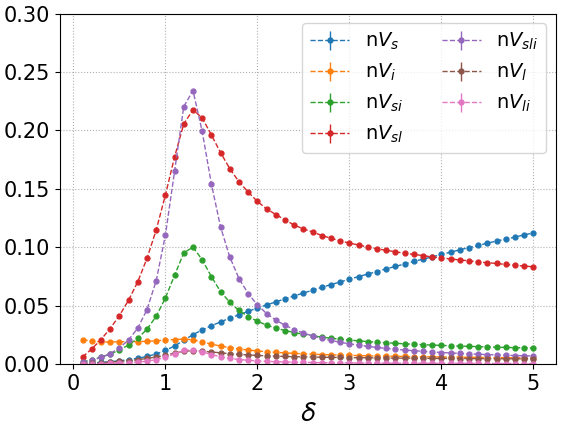}
\caption{Bias-variance decomposition and the variance components under an AR-1 covariance assumption. {\bf  Left}: numerical simulation of the bias, variance and MSE formulas. {\bf Right}: simulations with variance components. Upper: $r=0.5,\lam=0.01$. Down: $r=0.9,\lam=0.001$. For each term, symbolized by $\star$, we show $n\star$: numerical (averaged over 5 runs). Parameters: $\alpha=1,\sigma=0.3,\pi=0.8,n=150,d=\lfloor n\delta \rfloor ,p=\lfloor d\pi \rfloor.$}
\label{general_cov}
\end{figure}
{Although our theoretical results are proved under an assuming the data distribution is isotropic, we also study the model under general covariance assumptions numerically.
Namely, we assume the samples are drawn i.i.d. from $\N(0,\Sigma(r))$, where $\Sigma(r)_{ij}=r^{|i-j|}$ is an AR-1 covariance matrix. We numerically study the bias-variance decomposition and the variance components the same way as in Section \ref{deriest}. The only distinction is that the formulas we estimate are slightly different due to the non-identity covariance. More specifically, one can show that the formulas for the general covariance case are the same as their counterparts in equations \eqref{2lbias}---\eqref{2lmse}, \eqref{sobolvsf}---\eqref{sobolvslif} with $\|\cdot\|_{F}^2$ replaced by $\tr(\cdot\cdot^{\top}\Sigma(r))$.} 

{In the experiment, we choose fixed small penalties $\lam=0.01,0.001$ to mimic the ridgeless limit. Figure \ref{general_cov} shows the numerical results when $r=0.5$, $0.9$. We see that many observations in the isotropic case (e.g. monotonic bias, non-monotonicity of the MSE and variance, the interaction terms dominating the variance at the interpolation threshold) still hold in the general covariance case. However, the terms contributing the most to the total variance are $V_{sl}$ and $V_{sli}$ when $r=0.9$, while they are $V_{si}$ when $r=0.5$ or in the isotropic case. We conjecture that this is because the covariance matrix can implicitly change the ratio between the noise $\sigma$ and signal $\alpha$. However, more work is needed in the future for understanding the generalization error under a general covariance assumption.}

\subsection{Experiments on Empirical Data}

In this section, we present an empirical data example to study several phenomena observed in our theoretical analysis. We use the Superconductivity Data Set \citep{hamidieh2018data} retrieved from the UC Irvine Machine Learning Repository in our  data analysis.

The original data set contains $N=21,263$ superconductors as samples and $d=81$ features for each sample. Our goal is to predict the critical temperatures of the superconductors based on their features. Before doing regression, we  preprocess the data set in the following way. We first randomly shuffle the samples. We then separate the data set into a training set containing the first  $90\%$ of the samples, and a test set containing the rest. Finally, we normalize the features and responses so that they all have zero mean and unit variance. Since $N$ is quite a bit larger than $d$, we can estimate the variances of the features quite well; and thus we can standardize new test datapoints from this distribution using the estimated variances. After these steps, we are ready to start our experiments.

Similar to our theoretical setting,  for each experiment setting $(p,n)$, we randomly select $n$ samples from the training set to form a data matrix $X\in\R^{n\times d}$, map it into a random $p$-dimensional subspace multiplying it by a projection matrix $W$ with orthogonal rows and then perform ridge regression on the $p$-dimensional subspace. For each setting, we generate $50$ i.i.d. sample matrices $X_i$, $1\leq i\leq 50$,  and $50$ i.i.d. random projections $W_j$, $1\leq j \leq 50$, and combine them to form $2500$  sample-initialization pairs $(X_i,W_j)$, $1\leq i,j\leq 50$. Denote by $y_i$ the response vector of $X_i$. Then for each $(X_i,W_j)$, we have the ridge estimator
\begin{align*}
&\hat{f}_{ij}(x)=x^\top\left(\frac{W_jX_i^\top X_iW_j^\top}{n}+\lam I_p\right)^{-1}\frac{W_jX_i^\top y_i}{n},& 1\leq i,j\leq50.
\end{align*}
We use these estimators to make predictions on the test set and estimate the MSE, bias, and some of the variance components as $\widehat{\MSE}:=L^{-1}\sum_{k=1}^{L}\hat{\E}_{}(\hat{f}_{ij}(x_k)-y_k)^2$, and:  
\begin{align*}
&\widehat{\mathrm{Var}}:=\frac{1}{L}\sum_{k=1}^{L}\hat{\E}_{}(\hat{f}_{ij}(x_k)-\hat{\E}\hat{f}_{ij}(x_k))^2,
&\widehat{\mathrm{Bias}}^2:=\frac{1}{L}\sum_{k=1}^{L}(\hat{\E}\hat{f}_{ij}(x_k)-y_k)^2,\\
&\widehat{V_s}:=\frac{1}{Ln_s}\sum_{k=1}^{L}\sum_{i=1}^{n_s}(\hat{\E}_j\hat{f}_{ij}(x_k)-\hat{\E}\hat{f}_{ij}(x_k))^2,
&\widehat{V_i}:=\frac{1}{Ln_i}\sum_{k=1}^{L}\sum_{j=1}^{n_i}(\hat{\E}_i\hat{f}_{ij}(x_k)-\hat{\E}\hat{f}_{ij}(x_k))^2,
\end{align*}
where $\hat{\E}$ denotes the Monte Carlo mean, $n_i=n_s=50$, $L$ is the test set size, and $x_k,y_k$ are test features and responses. 
\begin{figure}[htb]
\centering
\includegraphics[width=.49\textwidth]{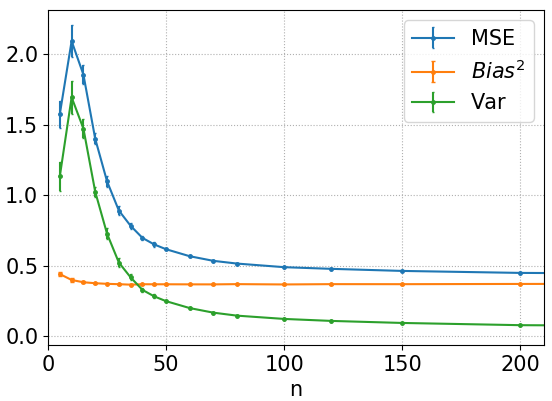}
\includegraphics[width=.49\linewidth]{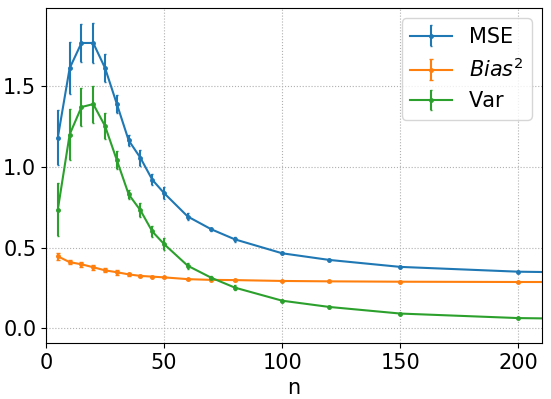}
\caption{Empirically estimated MSE, variance and bias as functions of number of samples $n$. We display the mean and one standard deviation of the numerical results over $10$ repetitions. {\bf  Left}: $\pi=0.2,\lam=0.01$. {\bf Right}: $\pi=0.9,\lam=0.01$. {(Both panels are from the same simulation.)}}
\label{superconfig1}
\end{figure}
\begin{figure}[htb]
\centering
\includegraphics[width=.32\textwidth]{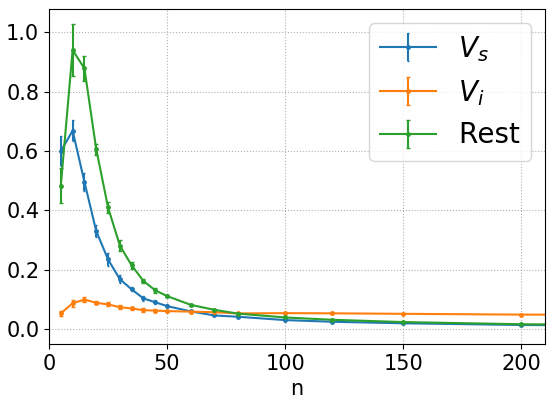}
\includegraphics[width=.32\textwidth]{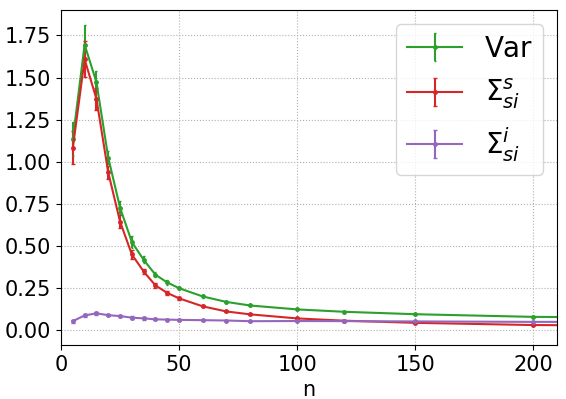}
\includegraphics[width=.32\textwidth]{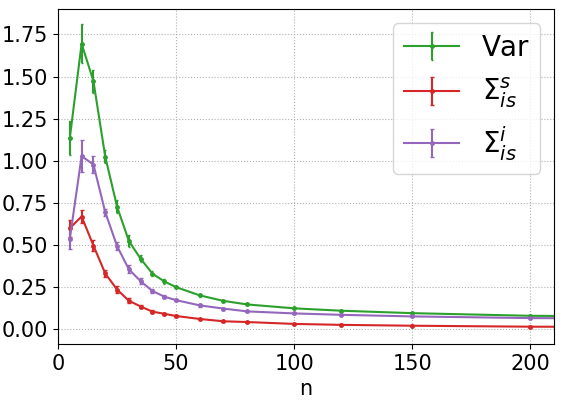}
\caption{Numerically estimated variance components as a function of the sample size $n$. {\bf  Left}: three components of variance ($V_s,V_i$, and ``Rest": the variance due to interaction and response noise, $\mathrm{Rest}:=\mathrm{Var}-V_s-V_i$). Middle and right: two orders of variance decomposition. Middle: sample, initialization. {\bf Right}: initialization, sample. $\Sigma_{ab}^{a}:=\mathrm{Var}-V_{b},\Sigma_{ab}^b:=V_b, \{a,b\}=\{s,i\}$. Parameters: $\pi=0.2, \lam=0.01$.  We display the mean and one standard deviation over $10$ repetitions. {(All three panels are from the same simulation.)}}
\label{superconfig2}
\end{figure}
\begin{figure}[htb]
\centering
\includegraphics[width=.5\textwidth]{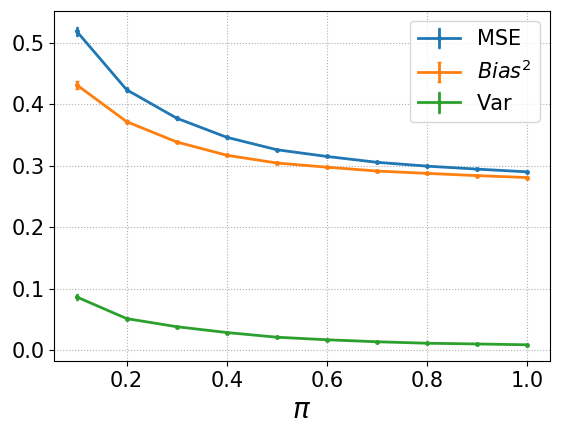}
\caption{Empirically estimated MSE, variance and bias as functions of degree of parameterization $\pi_d = p/d$. We show the the mean and one standard deviation over $10$ repetitions.}
\label{superconfig3}
\end{figure}
\begin{figure}[htb]
\centering
\includegraphics[width=.49\textwidth]{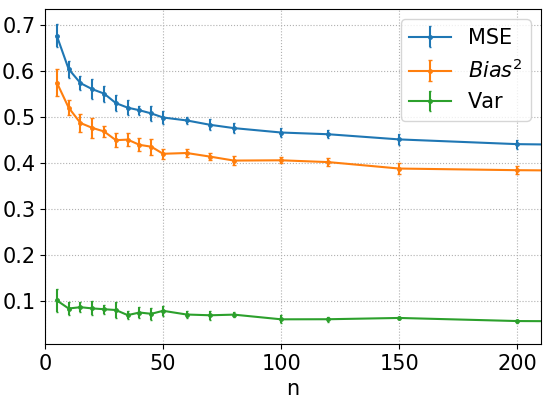}
\includegraphics[width=.49\linewidth]{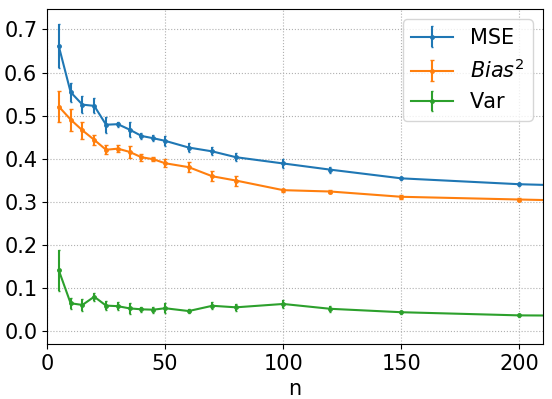}
\caption{Empirically estimated MSE, variance and bias as functions of number of samples $n$ using the optimal $\lam^*$. We display the mean and one standard deviation of the numerical results over $10$ repetitions. {\bf  Left}: $\pi=0.2$. {\bf Right}: $\pi=0.9$. {(Both panels are from the same simulation.)}}
\label{superconfig4}
\end{figure}

Since we have no information about the true noise of the responses, we only study the variance introduced by the choice of the data matrix $X$ and initialization $W$.
Figure \ref{superconfig1} shows the empirically estimated bias, variance and MSE as functions of the number of samples $n$ given a fixed amount of parameterization (fixed $\pi = \lim p/d$). From this figure, we observe the following:
\begin{enumerate}
\item The MSE is unimodal  as a function of number of samples $n$, which corroborates that increasing the number of training samples can sometimes lower the model's performance when we do not have enough  samples and do not regularize well (e.g. use a small $\lam=0.01$). This unimodality is quite similar to the unimodality we  observed in Figure \ref{fig: 2layer_fixed_lambda_mse} in our theoretical setting. It is also consistent with the general phenomenon of sample-wise double descent \citep{nakkiran2019more}.
\item  The bias is decreasing as a function of $n$, when $n$ is small, and stays roughly constant when $n$ is larger. The reason is that the data does not truly come from a linear model, and thus the linear model that we use has a nonzero approximation bias.  
\item The bias is also decreasing  as a function of $1/\delta=n/d$. This suggests that more samples can reduce the bias of the ridge estimator. 
 Furthermore, the variance is the main contributor to the unimodality of the MSE. These two observations are also consistent with our theoretical results from Theorem \ref{bias_var_for_fixed_lambda}, which suggests  that the bias is increasing as a function of $\delta$  and the variance can be very large along the interpolation threshold $\delta\pi=1$ when $\lam$ is small.
\end{enumerate}

 Figure \ref{superconfig2} (left) shows estimates of three components of variance in our data example. In this low parameterization setting ($\pi=0.2$), when $n$ is small ($<100$), the variance $V_s$ due purely to sampling is large, the variance $V_i$ due purely to initialization is small, and the variance $V_{is}$ due to their interaction and also the response noise is large. Thus,  combining these variances together, we can see from Figure \ref{superconfig2} (middle and right) that different orders of decomposition can indeed lead to different interpretations of the variances introduced by sampling and  initialization. 

  Figure \ref{superconfig3} exhibits empirical estimates of the bias, variance, and MSE as functions of the degree of parameterization $\pi = \lim p/d$ when given enough samples, here $n=1000$, so that $\delta_d = d/n$ is small. We see that  all three terms decrease as $\pi$ increases, which means more parameters can improve the estimator's performance when we have enough samples ($\delta$ is small). This is also close to what we have observed in Figure \ref{sobol_sub_28}, i.e., that the MSE is decreasing as $\pi$ increases when $\delta<1$. 

 {We also study the effect of optimal ridge penalty on the empirical data. For simplicity, we select the ridge parameter from the set $\{i\times 10^{-j}|i=1,2,5; j=0,1,2,3\}$ such that it minimizes the empirical MSE. Figure \ref{superconfig4} shows the MSE, variance and bias obtained using the optimal ridge penalty. Compared with Figure \ref{superconfig1}, we see that the optimal ridge penalty can mitigate the non-monotonicity of MSE and keep the bias decreasing as the number of sample $n$ increases. These observations are consistent with what we have shown in our theoretical setting.}

  To conclude, although our theoretical results are based on quite strong assumptions on the data distribution, many conclusions and insights still carry over to certain problems involving  empirical data.

\acks{
We thank the associate editor for handling our paper. We are very grateful for the reviewers for detailed and thorough feedback, which has lead to numerous important improvements.
We thank Yi Ma, Song Mei, Zitong Yang, Chong You, Yaodong Yu for helpful discussions. This work was partially supported by a Peking University Summer Research award, and by the NSF-Simons Collaboration on the Mathematical and Scientific Foundations of Deep Learning THEORINET (NSF 2031985). This work was performed when LL was a student at Peking University.}

\appendix
\section{Comparison of Orthogonal and Gaussian Initialization Models}
\label{orthogonal}
Here we provide a comparison of the orthogonal and Gaussian initialization models for linear networks $f(x)=(Wx)^\top \beta$. 
The \emph{expressive power} of the two models is the same, as with probability one we can write a $p\times d$ matrix $W$ with iid Gaussian entries, where  $p\le d$, via its SVD as $W = UDV$, where $U$ is $p\times p$ orthonormal, $D$ is $p\times p$ diagonal with nonzero entries with probability one, and $V$ is $p\times d$ partial orthonormal with $VV^\top = I_p$. Then $(Wx)^\top \beta = x^\top W^\top \beta = x^\top V^\top D U^\top \beta = x^\top W_o^\top \beta_o$, where $W_o= UV$ is a random partial orthonormal matrix, and $\beta_o= UDU^\top \beta$ is a new regression coefficient. Thus, the two models have the same expressive power. 

The orthogonal model we consider has some advantages over the Gaussian model. Indeed, considering the case when $p=d$, in the orthogonal model, we first rotate $x$ orthogonally, then take a linear combination of the coefficients. In contrast, in the Gaussian model we not only rotate $x$, but also scale it by the singular values of $x$, which due to the Marchenko-Pastur law \citep{marchenko1967distribution} spread out from zero to two. Thus, we induce a significant distortion of the input in the first layer. Then, we can expect that learning may be more challenging due to this additional scaling. Indeed, the regression coefficients corresponding to the directions with near-zero singular values must be scaled up asymptotically by unboundedly large values for accurate prediction. On the other hand, the Gaussian model more closely mimics practical initialization schemes, which can indeed involve iid weights. We also note that recently, some empirical work has argued about the benefits of orthogonal initialization \citep{hu2019provable,qi2020deep}. For instance, \cite{qi2020deep} argues that orthogonality (or isometry) alone enables training practical $>$100 layer CNNs on ImageNet without shortcut and BatchNorm, and therefore provides some justification for orthogonality in our theoretical analysis.

\section{Proofs}\label{all_proofs}
\subsection{Proof of Theorem \ref{2lthm1}}
Different from the order of the theorems, here we first give the proof of theorem \ref{2lthm1} and then the proof of theorem \ref{sobolthm}. This is because the proof of theorem \ref{sobolthm} is more complicated and depends on some lemmas in the proof of theorem \ref{2lthm1}. 

In the proofs, we will often refer to the spectral distribution (or measure) a symmetric matrix $M$, which is simply the discrete distribution placing uniform point masses on each of the (real) eigenvalues of $M$. When the matrix size grows, we will consider settings where the spectral distribution converges in distribution to a fixed probability distribution.

Let us define 
$$\tM_{X,W}(\lambda):=W^\top(n^{-1} WX^\top XW^\top+\lambda I_p)^{-1}WX^\top/n$$ 
(a $d\times n$ matrix), and
$$M_{X,W}(\lambda):=\tM_{X,W}(\lambda)X$$
(a $d\times d$ matrix) and omit their dependence on $\lambda,X,W$ for simplicity. As we will clearly see below, $M$ can be viewed as a ``regularized pseudo-inverse". Also, $\E M-I$ directly controls the bias, and $M-\E M$ controls part of the variance. The calculations of bias and variance in terms of $M$ follow those of \cite{yang2020rethinking}, and we include them here for the reader's convenience. The calculations following them are more novel.


To start, we have $f_{\lambda,\mT,W}(x)=x^\top \tM Y$, and similar to the proof of theorem 1 in \cite{yang2020rethinking}
\begin{align*}
 \bias^{2}(\lambda) &=\E_{\theta,x}\left[\E_{X,W,\Ep}\left(x^\top  M \theta+x^\top  \tM \Ep\right)-x^\top  \theta\right]^2 \\
 	&=\E_{\theta,x}\left[x^\top (\E_{X,W,\Ep} M-I) \theta+\E_{\mT,W} (x^\top \tM\Ep)\right]^{2}. 
\end{align*}
Since $\Ep$ has zero mean and is independent of $x$ and $\tM$, we know $\E_{X,W,\Ep} (x^\top \tM\Ep)=0$. In what follows, sometimes we omit the subscript when we take expectation over all random variables.  Thus, using that for any two vectors and a matrix of conformable sizes, $(a^\top N b)^2 = a^\top N b b^\top N^\top a = \tr N b b^\top N^\top a a^\top $, the above equals
\begin{align}
	&\E_{\theta,x}\left[ x^\top (\E  M-I) \theta \theta^\top (\E  M-I)^\top  x\right] \nonumber\\
	&=\E_{\theta,x} \tr \left[x^\top (\E  M-I) \theta \theta^\top (\E  M-I)^\top  x\right] \nonumber\\
	&=\tr \left[(\E  M-I)\E \left(\theta \theta^\top \right) (\E  M-I)^\top \E \left(x x^\top \right) \right] \nonumber\\
	&=\frac{\alpha^2}{d}\|\E  M-I\|_F^{2}\label{2lbias}.\end{align}
This  shows that the average bias is determined by how well the random matrix $M$ (which depends both on the random data $X$ and the random initialization $W$) approximates the identity matrix.

Similarly, by grouping terms appropriately, and using again that $\E_{X,W,\Ep} (x^\top \tM\Ep)=0$,
\begin{align*}	
 \Vl
&=\E_{\theta,x,X,W,\Ep}\left[x^\top M\theta+x^\top \tM\Ep-\E_{X,W,\Ep}(x^\top M\theta+x^\top \tM\Ep)\right]^2\\
&=\E_{\theta,x,X,W,\Ep}\left[x^\top (M-\E M)\theta+x^\top \tM\Ep\right]^2\\
&=\E_{\theta,x,X,W,\Ep}\left[x^\top (M-\E M)\theta\right]^2+(x^\top \tM\Ep)^2,
\end{align*}
where the interaction term is zero because of the independence between $\Ep$ and other variables. Then, using that $\tr A^\top A = \|A\|_{F}^2$,
\begin{align}                                                                     
\Vl
&=\E_{\theta,x,X,W,\Ep}\left\{\left[x^\top (M-\E M)\theta\theta^\top (M-\E M)^\top x\right]
+x^\top \tM\Ep\Ep^\top \tM^\top x\right\}\nonumber\\
&=\E_{\theta,x,X,W}\left\{\tr \left[(M-\E M)\theta\theta^\top (M-\E M)^\top xx^\top \right]+\sigma^2\tr [x^\top \tM\tM^\top x]\right\}\nonumber\\
&=\E_M\tr \left[(M-\E M)\E (\theta\theta^\top )(M-\E M)^\top \E (xx^\top )\right]+\sigma^2\E\tr \left[\tM\tM^\top \E (xx^\top )\right]\nonumber\\
&=\frac{\alpha^2}{d}\E \|M-\E M\|_{F}^2+\sigma^2\E\|\tM\|_F^2\label{2lvar}.\end{align}
Thus, the variance is determined by how much $M$ varies around its mean, and by how large $\tM$ is.
Combining results for variance and bias, and using that $\E \|M-I\|_{F}^2 = \E \|M-\E M\|_{F}^2 + \E \|\E M-I\|_{F}^2$, we obtain
\begin{align}
\mse(\lam)&=\var(\lam)+\bias^2(\lam)+\sigma^2
=\frac{\alpha^2}{d}\E \|M-I\|_{F}^2+\sigma^2\E\|\tM\|_F^2+\sigma^2.\label{2lmse}
\end{align}
Similarly, for $\Slab,\Ssam,\Sini$, we have
\begin{align*}\Slab&=\E_{\theta,x}\E_{W,X,\Ep}[f_{\lambda,\mT,W}(x)-\E_{\Ep}f_{\lambda,\mT,W}(x)]^2\\
&=\E_{\theta,x,W,X,\Ep}(x^\top \tM\Ep)^2
=\sigma^2\E\|\tM\|_F^2.
\end{align*}
Thus, the variance due to label noise is determined by the magnitude of $\tM$. Also,
\begin{align*}
\Ssam&=\E_{\theta,x}\E_{W,X}[\E_{\Ep}f_{\lambda,\mT,W}(x)-\E_{X,\Ep}f_{\lambda,\mT,W}(x)]^2\\
&=
\E_{\theta,x,W,X}[x^\top {M}\theta-\E_{X}(x^\top {M}\theta)]^2\\
&=\frac{\alpha^2}{d}\E_{W,X}\left\|M-\E_{X}M\right\|_{F}^2.
\end{align*}

Finally,
\begin{align*}
\Sini&=
\E_{\theta,x}\E_{W}(\E_{X,\Ep}f_{\lambda,\mT,W}(x)-\E_{W,X,\Ep}f_{\lambda,\mT,W}(x))^2\\
&=
\E_{\theta,x,W}[\E_{X}(x^\top {M}\theta)-\E_{W,X}(x^\top {M}\theta)]^2\\
&=\frac{\alpha^2}{d}\E_{W}\left\|\E_{X}M-\E_{W,X}M\right\|_{F}^2.
\end{align*}
This shows that in the specific decomposition order: label, samples, initialization, the variance due to the randomness in the sample $X$ is determined by the Frobenius variability of $M$ around its mean with respect to $X$. The respective statement is also true for the variance due to initialization.


Therefore, to prove theorem $\ref{2lthm1}$, it suffices to study the limiting behaviours of $\E M$, $\E_X M$, $\tM$ and $M$. 
Under the assumptions in Theorem \ref{2lthm1}, we have characterize their behavior in the following lemmas.
\begin{lemma}[Behavior of $\E M$]\label{2llm1} 
\begin{align} \label{2llm1eq}
	\lim_{d\to\infty}\frac{1}{d}\E \tr(M)&=\pi(1-\lambda\theta_1), \text{\quad} \forall i\geq1.
	&\lim_{d\to\infty}\frac{1}{d}\|\E M\|_F^2&=\pi^2(1-\lambda\theta_1)^2.\end{align}
\end{lemma}
\begin{lemma}[Behavior of the Frobenius norm of $M$]\label{2llm2}
\begin{align}\label{2llm2eq}
	\lim_{d\to\infty}\frac{1}{d}\E \|M\|_F^2&=\pi\left[1-2\lambda\theta_1+\lambda^2\theta_2+(1-\pi)\delta (\theta_1-\lambda\theta_2)\right].\end{align}
\end{lemma}
\begin{lemma}[Behavior of the Frobenius norm of $\tM$]\label{2llm3}
	\begin{align} \label{2llm3eq}
		\lim_{d\to\infty}\E \|\tM\|_F^2&=\pi\delta (\theta_1-\lambda\theta_2).
	\end{align}
\end{lemma}
\begin{lemma}[Behavior of the Frobenius norm of $\E_X M$]\label{2llm4}
\begin{align} \label{2llm4eq}
\lim_{d\to\infty}\frac{1}{d}\E_{W}\|\E_{X}M\|_{F}^2&=\pi(1-\lambda\theta_1)^2.\end{align}
\end{lemma}
We put the proof of these lemmas after the proof of theorem \ref{2lthm1} and \ref{sobolthm} for clarity (see Appendix \ref{pflems1}). By using Lemmas \ref{2llm1}---\ref{2llm4}, we are able to complete our proof.
\begin{proof}[Proof of Theorem \ref{2lthm1}] 
Plugging equation (\ref{2llm1eq}) into (\ref{2lbias}), we have
\begin{align*}
\Bl&=\frac{\alpha^2}{d}\|\E M-I\|_F^2
\to\alpha^2(\pi(1-\lambda\theta_1)-1)^2
=\alpha^2(1-\pi+\lambda\pi\theta_1)^2.
\end{align*}
Plugging equations (\ref{2llm1eq}), (\ref{2llm2eq}), (\ref{2llm3eq}) into (\ref{2lvar}), we have
\begin{align*}
	\Vl&=\frac{\alpha^2}{d}\E \|M-\E M\|_F^2+\sigma^2\E\|\tM\|_F^2\\
	&=
	\frac{\alpha^2}{d}\left[\E \|M\|_F^2-\|\E M\|_F^2\right]+\sigma^2\E\|\tM\|_F^2\\\	
	&\to
	\alpha^2\pi\biggl[1-2\lambda\theta_1+\lambda^2\theta_2+(1-\pi)\delta (\theta_1-\lambda\theta_2)-\pi(1-\lambda\theta_1)^2\biggr]+\sigma^2\pi\delta (\theta_1-\lambda\theta_2)\\
	&=\alpha^2\pi\biggl[1-\pi+(\pi-1)(2\lambda-\delta )\theta_1-\pi\lambda^2\theta_1^2+\lambda(\lambda-\delta +\pi\delta )\theta_2\biggl]+\nonumber
  \sigma^2\pi\delta (\theta_1-\lambda\theta_2).
\end{align*}
Similarly, by Lemmas \ref{2llm2}, \ref{2llm3}, \ref{2llm4} 
\begin{align*}
\Slab&=\sigma^2\E\|\tM\|_F^2\to\sigma^2\pi\delta (\theta_1-\lambda\theta_2).\\
\Ssam&=\frac{\alpha^2}{d}\E_{W,X}\left\|M-\E_{X}M\right\|_{F}^2=\frac{\alpha^2}{d}[\E\|M\|_F^2-\E_{W}\|\E_{X}M\|_F^2]\\
&\to
\alpha^2\pi\left[-\lambda^2\theta_1^2+\lambda^2\theta_2+(1-\pi)\delta (\theta_1-\lambda\theta_2)\right].
\end{align*}
Finally,
\begin{align*}
\Sini&=\Vl-\Ssam-\Slab
\to\alpha^2\pi(1-\pi)(1-\lambda\theta_1)^2.\\
\mse(\lam)&=\Vl+\Bl+\sigma^2\\
&\to\alpha^2\left\{1-\pi+\pi\delta \left(1-\pi+\sigma^2/\alpha^2\right)\theta_1+\left[\lambda-\delta \left(1-\pi+\sigma^2/\alpha^2\right)\right]\lambda\pi\theta_2\right\}+\sigma^2.
\end{align*}
As for the choice of optimal $\lam^*$, denote $\delta(1-\pi+\sigma^2/\alpha^2)$ by $c$ and calculate
\begin{align*}
\frac{\mathrm{d}}{\mathrm{d}\lambda}\lim_{d\to\infty}\mse(\lambda)
&=
\frac{\mathrm{d}}{\mathrm{d}\lambda}\alpha^2\left[1-\pi+\pi c\theta_1+\left(\lambda-c\right)\lambda\pi\theta_2\right]+\sigma^2\\
&=
\alpha^2\frac{\mathrm{d}}{\mathrm{d}\lambda}\left(\int\frac{\pi c}{x+\lambda}dF_{\gamma}(x)+\int\frac{(\lambda-c)\lambda\pi}{(x+\lambda)^2}dF_{\gamma}(x)\right)\\
&=
\alpha^2\pi\frac{\mathrm{d}}{\mathrm{d}\lambda}\left(\int\frac{\lambda^2+cx}{(x+\lambda)^2}dF_{\gamma}(x)\right)\\
&=
2\alpha^2\pi\int\frac{(\lambda-c)x}{(x+\lambda)^3}dF_{\gamma}(x).
\end{align*}
If $\pi=1$ and $\sigma=0$, then $c=0$ and the asymptotic MSE is monotonically increasing since $F_\gamma(x)$ is supported on $[0,+\infty)$. Therefore the optimal ridge $\lam^*=0$, which is outside of the range $(0,\infty)$ that we considered here. Otherwise
, the derivative is less than zero when $\lambda<c$ and larger than zero when $\lambda>c$. Therefore, the asymptotic MSE as a function of $\lambda$ has a unique minimum at $c=\delta(1-\pi+\sigma^2/\alpha^2)$.

\end{proof}

\subsection{Proof of Theorem \ref{sobolthm}}
In this proof, we will use the same notations as in the proof of theorem \ref{2lthm1}. 
Since the main idea of this proof is quite similar to the proof of theorem \ref{2lthm1}, we will omit some details in the derivation for simplicity.

\begin{proof}[Proof of Theorem \ref{sobolthm}]
 By definition, for $V_s$,
\begin{align}
 V_s&=\E_{\theta,x} \V_X(\E_{\Ep,W}(\hat{f}(x)|X))
	=\E_{\theta,x,X}[x^\top(\E_W M-\E M)\theta]^2\nnum\\&=\frac{\alpha^2}{d}\E_X\|\E_WM-\E M\|_F^2.\label{sobolvsf}
\end{align}
For $V_l$,
\begin{align}
	V_l&=\E_{\theta,x} \V_\Ep(\E_{X,W}(\hat{f}(x)|\Ep))
	=\sigma^2\|\E\tM\|^2_{F}.\label{sobolvlf}
\end{align}
For $V_i$,
\begin{align}
	 V_i&=\E_{\theta,x} \V_W(\E_{\Ep,X}(\hat{f}(x)|W))
	=\E_{\theta,x,W}[x^\top(\E_X M-\E M)\theta]^2\nnum\\&=\frac{\alpha^2}{d}\E_W\|\E_XM-\E M\|_F^2.\label{sobolvif}
\end{align}
Similarly,
\begin{align}
 V_{sl}&=\E_{\theta,x} \V_{\Ep,X}(\E_{W}(\hat{f}(x)|\Ep,X))-V_s-V_l\nnum\\
	&=\E_{\theta,x,X,\Ep}[x^\top(\E_W M-\E M)\theta+x^\top\E_W\tM\Ep]^2-V_s-V_l\nnum\\
	&=\sigma^2\E_X\|\E_W\tM-\E \tM\|_F^2\label{sobolvslf}.\end{align}
\begin{align}
	 V_{li}&=\E_{\theta,x} \V_{\Ep,W}(\E_{X}(\hat{f}(x)|\Ep,W))-V_i-V_l\nnum\\
	&=\E_{\theta,x,\Ep,W}[x^\top(\E_X M-\E M)\theta+x^\top\E_X\tM\Ep]^2-V_i-V_l\nnum\\
	&=\sigma^2\E_W\|\E_X\tM-\E\tM\|_F^2\label{sobolvlif}
\end{align}
\begin{align}
	V_{si}&=\E_{\theta,x} \V_{X,W}(\E_{\tau}(\hat{f}(x)|X,W))-V_s-V_i\nnum\\
	&=\E_{\theta,x,X,W}[x^\top( M-\E M)\theta]^2-V_s-V_i\nnum\\
	&=\frac{\alpha^2}{d}\left(\E\|M\|_F^2-\E_X\|\E_WM\|^2_F-\E_W\|\E_XM\|^2_F+\|\E M\|_F^2\right).\label{sobolvsif}
\end{align}
And
\begin{align}
V_{sli}&=\V(\hat{f}(x))-(V_s+V_l+V_i+V_{sl}+V_{si}+V_{li})\nnum\\
&=\sigma^2(\E\|\tM\|_F^2-\E_W\|\E_X\tM\|_F^2-\E_X\|\E_W\tM\|_F^2+\|\E\tM\|^2).\label{sobolvslif}
\end{align}
After obtaining the expressions of the variance components,  Theorem \ref{sobolthm} follows directly by plugging Lemmas  \ref{2llm1}---\ref{fnormewm}  into equation \eqref{sobolvsf}---\eqref{sobolvslif}. 
\end{proof}

\begin{lemma}[Behaviour of the Frobenius norm of $\tM$]\label{etm0}
\begin{align*}
\lim_{d\to\infty}\|\E\tM\|_F^2=\lim_{d\to\infty}\E_{W}\|\E_X\tM\|_F^2=0.
\end{align*}
\end{lemma}
When $X$ is symmetric, by switching the sign of $X$, clearly $\E_X\tM=0$. This lemma shows that the same result still holds asymptotically  when $X$ is not symmetric, and simply has zero-mean entries with finite sixth moment. 
\begin{lemma}[Behavior of the Frobenius norm of $\E_W\tM$]\label{fnormewtm}
	\begin{align*}\lim_{d\to\infty}\E_X\|\E_{W}\tM\|_{F}^2=\delta (\tth_1-\tlam\tth_2).
	\end{align*}
\end{lemma}
\begin{lemma}[Behavior of the Frobenius norm of $\E_W M$]\label{fnormewm}
\begin{align*}
\lim_{d\to\infty}\frac{1}{d}\E_X\|\E_{W}M\|_{F}^2=(1-2\tlam\tth_1+\tlam^2\tth_2).
\end{align*}
\end{lemma}
See Appendix \ref{pflems2} for the proof of the above lemmas.

\subsection{Proof of Lemmas \ref{2llm1}---\ref{2llm4}}
\label{pflems1}
We first prove these four lemmas under the assumption that $X$ has i.i.d. standard Gaussian entries in Appendix \ref{2llm1s}---\ref{2llm4s} to obtain some heuristics for the formulas. Then in Appendix \ref{2llm5s} we generalize the proof to the non-Gaussian case which only requires $X$ to have i.i.d. zero mean unit variance and finite $8+\eta$ moment entries for any $\eta>0$.  

Next, we denote $R=\left({WX^\top  XW^\top /n}+\lambda I_p\right)^{-1}$ for simplicity. In the Gaussian case, the proof proceeds by moving to the SVD decomposition, and carefully exploiting serveral properties of the normal distribution and the Marchenko-Pastur law. In the general case, we use deterministic equivalents properties for covariance matrices to show that all terms we are concerned with converge to the same limits as in the Gaussian case.
\subsubsection{Proof of Lemma \ref{2llm1} (Under Gaussian Assumption)} \label{2llm1s}
\begin{proof}
By definition,
\begin{align*}
  \E M&=\E M_{X,W}(\lambda)
  =\E W^\top R \frac{WX^\top X}{n}.
\end{align*}
Let $V = [W^\top, W^\top_\perp]$ be an orthonormal matrix containing an arbitrary orthogonal complement of $W^\top$. It will be convenient to write this as $W = DV^\top $, where the $p\times d$ matrix  $D=\left(\mathrm{I}_{p\times p},0_{p\times(d-p)}\right)$ selects the appropriate rows of $V^\top$. Denote $\tX :=XV$, a matrix of the same size $n \times d$ as the original matrix $X$.  Then,

\begin{align*}
  \E M&=\E_{V,X}VD^\top \left(\frac{DV^\top X^\top XVD^\top }{n}
  +\lambda I_p\right)^{-1}
  \frac{DV^\top X^\top XVV^\top }{n}\\
  &=\E_{V,\tX }VD^\top \left(\frac{D\tX ^\top \tX D^\top }{n}
  +\lambda I_p\right)^{-1}
  \frac{D\tX ^\top \tX V^\top }{n}.
\end{align*}
By assumption, $W$ is uniformly sampled from the Stiefel manifold, i.e., the manifold of partial orthogonal $p\times d$ ($p\le d$) matrices with orthonormal rows. Thus, we can assume that $V$ is also uniformly distributed over orthogonal matrices (i.e. the Haar measure). Furthermore, since $V$ is orthogonal, we know that  $\tX $ has independent standard Gaussian entries and is independent of $V$ because $XV$ has the same distribution for any orthogonal matrix $V$. Noting that $\E_{V}VAV^\top =\tr A \cdot I_d/d$ when $V\in\R^{d\times d}$ follows the Haar measure, we get
\begin{align*}
\E M&=\E_{\tX }\E_{V}VD^\top \left(\frac{D\tX ^\top \tX D^\top }{n}
+\lambda I_p\right)^{-1}
\frac{D\tX ^\top \tX V^\top }{n}\\
&=\frac{1}{d}\cdot\E_{\tX }\tr \left(D^\top \left(\frac{D\tX ^\top \tX D^\top }{n}
+\lambda I_p\right)^{-1}
\frac{D\tX ^\top \tX }{n}\right)\cdot I_d\\
&=\frac{1}{d}\cdot\E_{\tX }\tr \left(\left(\frac{D\tX ^\top \tX D^\top }{n}
+\lambda I_p\right)^{-1}
\frac{D\tX ^\top \tX D^\top }{n}\right)\cdot I_d.
\end{align*}

 Further defining $\hat{X}:=\tX D^\top=XW^\top $ which is now of size $n \times p$ (while the original size was $n \times d$), then $\hat{X}$ also has independent standard Gaussian entries. Letting $\hat{X}=\hat{U}\Gamma \hat{V}^\top $ be the SVD decomposition of $\hat X$, we have
\begin{align*}
\E M&=\frac{1}{d}\cdot\E_{\hat{X}}\tr \left(\left(\frac{\hat{X}^\top \hat{X}}{n}
+\lambda I_p\right)^{-1}
\frac{\hat{X}^\top \hat{X}}{n}\right)\cdot I_d
=\frac{p}{d}\left(1-\frac{\lambda}{p}\E_{\Gamma}\tr \left(\frac{\Gamma^\top \Gamma}{n}
+\lambda I_p\right)^{-1}\right)\cdot I_d.
\end{align*}

This is determined by the spectral measure of $n^{-1} \hat{X}^\top \hat{X}$. Since
$\hat{X}$ has independent standard normal entries and $\lim_{d\to\infty}p/n=\pi\delta $, applying the Marchenko-Pastur theorem \citep{marchenko1967distribution,silverstein1995strong,bai2009spectral}, we get
\begin{align*}
\frac{1}{d}\E \tr(M)&
\to\pi\left(1-\lambda\int\frac{1}{(x+\lambda)}dF_{\pi\delta (x)}\right)=\pi(1-\lambda\theta_1(\pi\delta ,\lambda)),\\
\frac{1}{d}\|\E M\|_F^2&
\to\pi^2(1-\lam\theta_1)^2.
\end{align*}
This finishes the proof.
\end{proof}
\subsubsection{Proof of Lemma \ref{2llm2} (Under Gaussian Assumption)}\label{2llm2s}
\begin{proof}
By definition of $M$, 
\begin{align*}
 \E \|M\|_F^2&
 =\E\left\|W^\top R\frac{WX^\top  X}{n}\right\|_F^2\\ 
&=\E \tr \left(W^\top R\frac{WX^\top  X}{n}
\frac{X^\top  XW^\top }{n}R^\top W\right)\\ 
&=\E \tr \left(R
\frac{WX^\top  XX^\top  XW^\top }{n^2}R^\top \right).
\end{align*}
Let $W_\perp =f(W)\in\R^{(d-p)\times d}$ be an orthogonal complement of $W$ and define $X_1:=XW^\top $, $X_2:=XW_\perp ^\top $.
 Since $X$ has Gaussian entries, $X_1$ and $X_2$ both have Gaussian entries. Since
$$\E X_1^\top X_2
=\E WX^\top XW_\perp ^\top
=n\cdot\E WW_\perp ^\top
=0,$$
it follows that $X_1$ and $X_2$ are independent.
Noting that $XX^\top =X_1X_1^\top +X_2X_2^\top $, we have
\begin{align*}
  \E \|M\|_F^2
&=\E \tr \left(R
\frac{WX^\top  X_1X_1^\top  XW^\top }{n^2}R^\top \right)
+
\E \tr \left(R\frac{WX^\top  X_2X_2^\top  XW^\top }{n^2}R^\top \right).\\
&=:\mathrm{C}_1+\mathrm{C}_2.
\end{align*}
For the first term, we have
\begin{align*}
C_1&=\E \tr \left(R
\frac{WX^\top  X_1X_1^\top  XW^\top }{n^2}R^\top \right)
=\E \tr \left(R
\frac{X_1^\top  X_1X_1^\top  X_1}{n^2}R^\top \right).
\end{align*}
Let $X_1=U\Gamma_1 V^\top $ be the singular value decomposition of $X_1$. By plugging in the definition of $R$, we obtain
\begin{align*}
C_1&=\E \tr \left[\left(\frac{\Gamma_1^\top \Gamma_1}{n}+\lambda I_p\right)^{-2}\left(\frac{\Gamma_1^\top \Gamma_1}{n}\right)^2\right].
\end{align*}
Thus, according to  the Marchenko-Pastur theorem,
$$\frac{\alpha^2}{d}C_1\to\alpha^2\pi\int\frac{x^2}{(x+\lambda)^2}dF_{\pi\delta }(x)=\alpha^2\pi(1-2\lambda\theta_1+\lambda^2\theta_2).$$
For the second term, since $X_1$ and $X_2$ are indepedent and noting that $$\E_{X_2} X_2^\top X_2= \E_{X,W}(XW_\perp ^\top W_\perp X^\top )= \E_{W}(\tr (I_d-W^\top W))I_d=(d-p)I_d,$$
we have\begin{align*}
C_2&=\E \tr \left(R
\frac{WX^\top  X_2X_2^\top  XW^\top }{n^2}R^\top \right)
=
\E_{X_1}\E_{X_2}\tr \left(R
\frac{X_1^\top  X_2X_2^\top  X_1 }{n^2}R^\top \right)\\
&=
\E_{X_1}\tr \left(R
\frac{X_1^\top \E_{X_2}(X_2X_2^\top ) X_1 }{n^2}R^\top \right)
=
\frac{d-p}{n}
\E_{X_1}\tr \left(R
\frac{X_1^\top X_1 }{n}R^\top \right).
\end{align*}
Since
\begin{align*}
\E_{X_1}\tr \left(R
\frac{X_1^\top X_1}{n}R^\top \right)
=
\E \tr \left[\left(\frac{\Gamma_1^\top \Gamma_1}{n}+\lambda I_p\right)^{-2}\frac{\Gamma_1^\top \Gamma_1}{n}\right],
\end{align*}
  by the Marchenko-Pastur theorem, $\frac{\alpha^2}{d}C_2\to\alpha^2(1-\pi)\pi\delta \int\frac{x}{(x+\lambda)^2}dF_{\pi\delta }(x)=\alpha^2(1-\pi)\pi\delta (\theta_1-\lambda\theta_2)$.
Finally, combining the results for $C_1$ and $C_2$ gives
\begin{align*}  
\frac{\alpha^2}{d}\E \|M\|_F^2
  &=
  \frac{\alpha^2}{d}(C_1+C_2)
  \to
  \alpha^2\pi\left[1-2\lambda\theta_1+\lambda^2\theta_2+(1-\pi)\delta (\theta_1-\lambda\theta_2)\right],
\end{align*}
and this finishes the proof.
\end{proof}
\subsubsection{Proof of Lemma \ref{2llm3} (Under Gaussian Assumption)}\label{2llm3s}
\begin{proof}
By definition, we have
\begin{align*}
\E\|\tM\|^2_{F}&=\E \tr \left(W^\top R\frac{WX^\top  }{n}\right)
  \left
  (W^\top R\frac{WX^\top  }{n}\right)^\top \\ 
  &=\E \tr \left(W^\top R\frac{WX^\top }{n}
  \frac{XW^\top }{n}R^\top W\right)
  =\E \tr \left(R
  \frac{WX^\top  XW^\top }{n^2}R^\top \right).
\end{align*}
Denote ${XW^\top }$ by $X_1$ and write $X_1=U\Gamma V^\top$ for the SVD of $X_1$. Then,
\begin{align*}
\E\|\tM\|^2_{F}&=\frac{1}{n}\E \tr \left(R
\frac{X_1^\top  X_1}{n}R^\top \right)
=\frac{1}{n}\E \tr \left(R-\lambda
R^2 \right)\\
&=\frac{1}{n}\E \tr \left[(\frac{\Gamma^\top \Gamma}{n}+\lambda)^{-1}-\lambda(\frac{\Gamma^\top \Gamma}{n}+\lambda)^{-2}
\right]
\to\pi\delta (\theta_1-\lambda\theta_2),
\end{align*}
where the last line follows from the Marchenko-Pastur theorem and that $p/n\to\pi\delta $.
\end{proof}

\subsubsection{Proof of Lemma \ref{2llm4} (Under Gaussian Assumption)}\label{2llm4s}
\begin{proof}
Similarly, let $W_\perp=f(W)\in\R^{(d-p)\times d}$ be an orthogonal complement of $W$. Denoting $X_1=XW^\top $, $X_2=XW_\perp ^\top $ and combining the fact that $X_1$ and $X_2$ are independent, with $\E X_2=0$, we have
\begin{align}
  \E_{X}M&=\E_{X} W^\top R\frac{WX^\top X}{n}\nnum\\
&=
W^\top \E_{X}\left(\frac{X_1^\top X_1}{n}+\lambda I_p\right)^{-1}\frac{X_1^\top(X_1W+X_2W_\perp )}{n}\nnum\\
&=
W^\top \E_{X_1}\left[\left(\frac{X_1^\top X_1}{n}+\lambda I_p\right)^{-1}\frac{X_1^\top X_1}{n}\right]W
\label{exmf1}\end{align}
Write the SVD of $X_1$ as $X_1=U\Gamma V^\top $ and note that $V\in\R^{p\times p}$ is uniformly distributed over the set of orthogonal matrices. Then the above equals
\begin{align*}
W^\top \left[I_{p}-\lambda\E_{\Gamma ,V}V\left(\frac{\Gamma^\top \Gamma }{n}+\lambda I_p\right)^{-1}V^\top \right]W
&=
W^\top W\left[1-\frac{\lambda}{p}\E_{\Gamma }\tr\left(\frac{\Gamma^\top \Gamma}{n}+\lambda I_p\right)^{-1} \right].
\end{align*}
Thus, 
\begin{align*}
\lim_{d\to\infty}\frac{\alpha^2}{d}\E_W\|\E_{X}M\|_{F}^2&=
\lim_{d\to\infty}\frac{\alpha^2}{d}\left[1-\frac{\lambda}{p}\E_{\Gamma }\tr\left(\frac{\Gamma^\top \Gamma }{n}+\lambda I_p\right)^{-1} \right]^2\E_W \tr(W^\top W)\\
&=\lim_{d\to\infty}\frac{\alpha^2 p}{d}\left[1-\frac{\lambda}{p}\E_{\Gamma }\tr\left(\frac{\Gamma^\top \Gamma }{n}+\lambda I_p\right)^{-1} \right]^2\\
&\to\alpha^2\pi(1-\lambda\theta_1)^2,
\end{align*}
where the last line follows directly from the Marchenko-Pastur theorem.
\end{proof}

\subsubsection{Proof of Lemma \ref{2llm1}---\ref{2llm4} (General Case)}\label{2llm5s}
In Appendix \ref{2llm1s}---\ref{2llm4s}, we have proved Lemma \ref{2llm1}---\ref{2llm4} under the assumption that the entries of $X$  are i.i.d. standard Gaussian. In this part, we will generalize previous proofs to the non-Gaussian case, i.e., $X$ has i.i.d. zero mean, unit variance entries with finite $8+\eta$ moment, and hence complete the proof of Lemma \ref{2llm1}---\ref{2llm4}.

For simplicity, we only present the proof of Lemmas \ref{2llm2}, \ref{2llm3} in non-Gaussian case.  Lemmas \ref{2llm1}, \ref{2llm4} can be proved using very similar arguments as Lemma \ref{2llm2}.

We recall the \emph{calculus of deterministic equivalents} from random matrix theory, which will be used in our proof \citep{dobriban2018understanding,dobriban2020wonder}. One of the best ways to understand the Marchenko-Pastur law is that \emph{resolvents are asymptotically deterministic}. Let $\hSigma = n^{-1} X^\top X$, where $X = Z\Sigma^{1/2}$ and $Z$ is an $n \times p$ random matrix with iid entries of zero mean and unit variance, and $\Sigma^{1/2}$ is any sequence of $p \times p$ positive semi-definite matrices. We take $n,p,q \to\infty$ proportionally.  
 
We say that the (deterministic or random) not necessarily symmetric matrix sequences $A_n, B_n$ of growing dimensions are \emph{equivalent}, and write 
$$A_n \asymp B_n$$ if 
\begin{align}\lim_{n\to\infty}\left|\tr\left[C_n(A_n-B_n)\right]\right|=0\label{detdefi}\end{align}
almost surely, for any sequence $C_n$ of not necessarily symmetric matrices with bounded trace norm, i.e., such that 
$$\lim\sup\|C_n\|_{tr}<\infty.$$ 

We call such a sequence $C_n$ a \emph{standard sequence}.
Recall here that the trace norm (or nuclear norm) is defined by $\|M\|_{tr}=\tr((M^\top M)^{1/2}) = \sum_i \sigma_i$, where $\sigma_i$ are the singular values of $M$. 

Moreover, if \eqref{detdefi} only holds almost surely for any sequence $C_n\in \R^{d_n\times d_n}$ of positive semidefinite matrices with $O(1/d_n)$ spectral norm, $A_n$ and $B_n$ are said to be weak deterministic equivalents and denoted by $\smash{A_n\overset{w}{\asymp}B_n}$. It is readily verified that deterministic equivalence implies weak deterministic equivalence. 

By the general Marchenko-Pastur (MP) theorem of Rubio and Mestre \citep{rubio2011spectral}, we have that for any $\lam>0$
\begin{align*}
(\hSigma+ \lam I)^{-1} &\asymp (q_p \Sigma + \lam I)^{-1},
\end{align*}

where  $q_p$ is the unique positive solution of the fixed point equation 
\beqs 1-q_p=\frac{q_p}{n}\tr\left[\Sigma(q_p\Sigma+ \lam  I)^{-1}\right]. \eeqs
When $n,p\to\infty$ and the sepctral distribution of $\Sigma$ converges to $H$, $q_p\to q$ and $q$ satisfies the equation
\beqs 1 -q = \gamma \left[1 -  \lam \int_0^\infty \frac{dH(t)}{qt+ \lam }\right].\eeqs 
We now proceed with the proof.

\begin{proof}[Proof of Lemma \ref{2llm2} (general case)]By definition, recalling that $R = (\frac{WX^\top  XW^\top }{n}+\lambda I_p)^{-1}$,
\begin{align}
  \E M&=\E W^\top R \frac{WX^\top X}{n}
  =\E {W^\top WX^\top}\tilde R \frac{X}{n}.\label{Mrepr2}
\end{align}
Therefore, letting $\tilde R = (\frac{XW^\top WX^\top}{n}+\lambda I_n)^{-1}$ be the resolvent obtained in the other order,
\begin{align*}
 \E \tr (MM^\top)
 &=
  \E \tr \left[{W^\top WX^\top}\tilde R \frac{XX^\top}{n^2}\tilde R {XW^\top W}\right].
\end{align*}
Define the regularized resolvent $\tilde R_\tau =\left(\frac{X(W^\top W+\tau)X^\top}{n}+\lambda I_n\right)^{-1}$ and
$$M_\tau :=W^\top W X^\top\tilde R_\tau  \frac{X}{n}.$$
Since we have already proved Lemma \ref{2llm2} under the Gaussian assumption, to generalize the results into the non-Gaussian case, we only need to prove the following two steps:
\begin{align} &(1).  \text{\quad}\lim_{d\to\infty}\E \tr(MM^\top)/d=
\lim_{\tau\to 0}\lim_{d\to\infty}\E\tr M_\tau M^\top_\tau /d \text{\hspace{1em}(assuming the limits exist)}\label{EMMlemmaf0}\\
&(2). \text{\quad}\lim_{\tau\to 0}\lim_{d\to\infty}\E\tr M_\tau M^\top_\tau /d \text{\quad  exists and is a constant independent of the distribution of $X$.}\label{eMMlemma2steps2}  \end{align}
(1).
Note that
\begin{align*}
\Delta&:=
\lim_{\tau\to 0}\lim_{d\to\infty}\frac{1}{d}|\E \tr(MM^\top)-\E\tr M_\tau M_\tau ^\top|\\
&\leq\lim_{\tau\to 0}\lim_{d\to\infty}\frac{1}{d}|\E \tr(M(M^\top-M_\tau ^\top))|+\frac{1}{d}\left|\E\tr [(M-M_\tau )M_\tau ^\top]\right|\\
&\leq\lim_{\tau\to 0}\lim_{d\to\infty}\frac{1}{d}|\E \tr(M(M^\top-M_\tau ^\top))|+\frac{1}{d}\left|\E\tr [M_\tau (M^\top-M_\tau ^\top)]\right|\\
&=:\tilde\Delta_1+\tilde\Delta_2.
\end{align*}
Therefore it suffices to prove $\tilde\Delta_{1,2}\to0$. For $\tilde\Delta_1$, we have the following argument.
 Denote ${X^\top\tilde R X/n}$ by $A_0$, ${X^\top \tilde R_\tau X/n}$ by $A_\tau$ and $W^\top W$ by $P$. Note that $M=PA_0$, and
\begin{align}
\tilde\Delta_1&=\lim_{\tau\to0}\lim_{d\to\infty}\E\frac{\tau}{d}\left|\tr\left(M A_0 A_\tau P\right)\right|\label{eMMdelta1}\\
&\leq\lim_{\tau\to0}\lim_{d\to\infty}\E\frac{\tau}{2d}[\|M\|_{F}^2+\|A_0A_\tau P\|_F^2]\nnum\\\nnum
&\leq
\lim_{\tau\to0}\lim_{d\to\infty}\E\frac{\tau}{2d}[\tr(A_0^2)+\tr(A_0A_\tau^2 A_0)]\\\nnum
&\leq
\lim_{\tau\to0}\lim_{d\to\infty}\tau\E\frac{1}{2d}\left[\frac{1}{\lam^2}\tr\left(\frac{X^\top X}{n}\right)^2+\frac{1}{\lam^4}\tr\left(\frac{X^\top X}{n}\right)^4\right]\\\nnum
&\leq
\lim_{\tau\to0}O(\tau)=0,\nnum
\end{align}
where the second line follows from the properties of the Frobenius norm. In the third and fourth line we use $A_0,A_\tau\preceq X^\top X/n\lam$ and the fact that $$\tr(M_1M_2M_1)\preceq \tr(M_1 M_3 M_1) \qquad\forall M_1,M_2,M_3 \text{\quad positive semi-definite and\quad} M_2\preceq M_3.$$ Finally, the last line is due to the finite $8+\eta$ moment assumption and some direct caclulations. Using the same techinique, it is not hard to show that $\tilde\Delta_2$ also converges to zero, and hence we conclude the proof of (1).

\quad\\
\noindent(2).
We first give an alternative expression for $M_{\tau}$. Denote $(W^\top W+\tau)$ by $S$ and let $V=U^\top=X/\sqrt{n}$, $C=I_n/\lam$ and $A=(W^\top W+\tau)^{-1}$ in the Woodbury identity:
\begin{align}
(A+UCV)^{-1}&=A^{-1}-A^{-1}U(C^{-1}+VA^{-1}U)^{-1}VA^{-1}\nnum\\ \iff\,\, 
U(C^{-1}+VA^{-1}U)^{-1}V&=A-A^{1/2}(I+A^{-1/2}UCVA^{-1/2})^{-1}A^{1/2}. \label{woodburytranformed}
\end{align}
We have by left multiplying $W^\top W$ in \eqref{woodburytranformed} that, with $R_S:=  (\lam I_d+\frac{S^{1/2}X^\top XS^{1/2}}{n})^{-1}$ 
\begin{align}
M_\tau 
&=W^\top W S^{-1/2}\left[I_d-\lam R_S \right]S^{-1/2}.\label{eMMlemmaalt}
\end{align}
Fix $\tau$ and suppose that the spectral distribution of $R$ converges in distribution to $H_\tau $ as $d\to\infty$. Since $W^\top W$ has $p$ eigenvalues $1$ and $d-p$ eigenvalues $0$ and $S=W^\top W+\tau$, it is clear that $H_\tau=\pi\delta_{1+\tau}+(1-\pi)\delta_{\tau}$.  Also, we have by theorem 1 in  \cite{rubio2011spectral} and some simple calculations, that
\begin{align}
 R_S \asymp (x_d S+\lam I_d)^{-1},\label{eMMlemmadet}
\end{align}
where $x_d$ is the unique positive solution of the fixed point equation (where we omit $x_d$'s dependence on $\tau$ for notational simplicity)
\begin{align*}
1-x_d=\frac{x_d}{n}\tr[S(x_dS+\lam I_d)^{-1}].
\end{align*}
When $n,p,d\to\infty$ proportionally, $x_d\to x$ and $x$
satisfies the equation (again omitting $x$'s dependence on $\tau$ for notational simplicity)
\begin{align}
1-x=\delta\left(1-\lam\int_{0}^{\infty}\frac{dH_\tau (t)}{xt+\lam}\right).\nnum 
\end{align} 
Now, we start to calculate $\E \tr(M_\tau  M_\tau ^\top)$. By definition,
\begin{align}
&\E \tr(M_\tau  M_\tau ^\top)
=\E \tr\left[W^\top W S^{-1}\left(I_d-\lam R_S \right)S^{-1}\left(I_d-\lam R_S \right)\right]\nnum
\\&=:\Delta_1+\Delta_2+\Delta_3, \nnum
\end{align}

where
\begin{align*}
\Delta_1&:=\E \tr\left[W^\top WS^{-2}\right]\\
\Delta_2&:=-2\lam\E \tr\left[W^\top W S^{-2} R_S \right]\\
\Delta_3&:=\lam^2\E \tr\left[W^\top W S^{-1} R_S S^{-1} R_S \right].
\end{align*}
Since $W^\top W$ has $p$ eigenvalues equal to $1$ and $d-p$ eigenvalues equal to $0$, 
\begin{equation}\lim_{\tau\to0}\lim_{d\to\infty}\Delta_1/d=\lim_{\tau\to0}\lim_{d\to\infty}\frac{p}{d(1+\tau)^2}=\pi.\label{emmlemmad1}\end{equation} 
Since $\|W^\top WS^{-2}\|_2\leq1$,  using the deterministic equivalent property \eqref{eMMlemmadet} and the bounded convergence theorem,  we get
\begin{align}\lim_{\tau\to0}\lim_{d\to\infty}\Delta_2/d&=-2\lam\lim_{\tau\to0}\lim_{d\to\infty}\E \tr\left[W^\top W S^{-2} R_S \right]\nnum\\&=-2\lam\lim_{\tau\to0}\lim_{d\to\infty}\E \tr\left[W^\top W S^{-2}\left(x_dS+\lam I_d\right)^{-1}\right]\nnum\\
&=-2\lam\lim_{\tau\to0}\lim_{d\to\infty}\frac{p}{d}\frac{1}{[(1+\tau)^2[x_d(1+\tau)+\lam]}\nnum\\
&=-2\lam\lim_{\tau\to0}\frac{\pi}{[(1+\tau)^2[x_{\tau}(1+\tau)+\lam]}=-2\lam\lim_{\tau\to0}\frac{\pi}{x_{\tau}+\lam},\label{emmlemmad2}\end{align}
where the third line follows from the fact that $W$ and $S=(W^\top W+\tau)$ are simultaneously diagonizable and $W^\top W$ has $p$ eigenvalues $1$ and $d-p$ eigenvalues $0$. It can be verified that the limit in \eqref{emmlemmad2} exists and is a constant independent of the distribution of $X$.
Now,  it remains to prove that $\lim_{\tau\to0}\lim_{d\to\infty}\Delta_3/d$ converges to a constant limit. Note that
\begin{align}
&\text{\quad}\lim_{\tau\to0}\lim_{d\to\infty}\Delta_3/d
=\lim_{\tau\to0}\lim_{d\to\infty}\frac{\lam^2}{d}\E \tr\left[W^\top W S^{-1} R_S S^{-1} R_S \right]\nnum\\
&=\lim_{\tau\to0}\lim_{d\to\infty}\frac{\lam^2}{d}\E \tr\left[W^\top W S^{-1/2} R_S S^{-1} R_S S^{-1/2}\right]\nnum\\
&=\lim_{\tau\to0}\lim_{d\to\infty}\frac{\lam^2}{d}\E \tr\left[W^\top W\left(\lam\tau+\lam W^\top W+\frac{SX^\top XS}{n}\right)^{-1}\left(\lam\tau+\lam W^\top W+\frac{SX^\top XS}{n}\right)^{-1}\right]\label{emmlemmadelt3expr}.\end{align}
For any $z\in E:=\mathbb{C}\backslash\R^{+}$, we have by Lemma \ref{emmsquaredet} that    
\begin{align}
\left(\lam W^\top W+\frac{SX^\top XS}{n}-zI_d\right)^{-2}\asymp-(\lam W^\top W+x_dS^2-zI_d)^{-2}(x_d'(z)S^2-I_d),\label{emmsquaredetf0}
\end{align}
where for any $z\in E$, $x_d(z)$ is the unique solution of certain fixed point equation independent of $X$, and $x_d'(z):=dx_d(z)/dz$. Furthermore, there exists $x(z)$ such that $x_d(z)\to x(z)$, $x_d'(z)\to x'(z)$. 
Now, letting $z=-\lam\tau$ and replacing $(\lam W^\top W+SX^\top XS/n-zI_d)^{-2}$ by its deterministic equivalent \eqref{emmsquaredetf0}, we have from \eqref{emmlemmadelt3expr} and the bounded convergence theorem that
\begin{align}
\lim_{\tau\to0}\lim_{d\to\infty}\Delta_3/d&=\lim_{\tau\to0}\lim_{d\to\infty}\frac{\lam^2}{d}\E \tr\left[W^\top W\left(\lam\tau I_d+\lam W^\top W+\frac{SX^\top XS}{n}\right)^{-2}\right]\nnum\\
&=\lim_{\tau\to0}\lim_{d\to\infty}\frac{-\lam^2}{d}\E \tr\left[W^\top W(\lam W^\top W+x_dS^2+\lam\tau I_d)^{-2}(x_d'S^2-I_d)\right]\nnum\\
&=\lim_{\tau\to0}\lim_{d\to\infty}{-\pi\lam^2}\frac{x_d'(1+\tau)^2-1}{(\lam+x_d(1+\tau)^2+\lam\tau)^2}\nnum\\
&=\lim_{\tau\to0}\pi\lam^2\frac{1-x'(1+\tau)^2}{(\lam+x(1+\tau)^2+\lam\tau)^2}
=\lim_{\tau\to0}\pi\lam^2\frac{1-x'}{(\lam+x)^2},\label{emmlemmafinalexpr}
\end{align}
where $x:=x(-\lam\tau), x':=x'(-\lam\tau)$. Since in the Gaussian case the limit of the L.H.S. of \eqref{EMMlemmaf0} exists, by the proof of \eqref{EMMlemmaf0}, we know that the limit in \eqref{emmlemmafinalexpr} also exists, and does not depend of $X$. This concludes the proof of (2).
\end{proof}

\begin{lemma}[Second order deterministic equivalent]\label{emmsquaredet} Suppose that $X\in\R^{n\times d}$ has i.i.d. zero mean, unit variance entries with finite $8+\eta$ moment. Then for any $z\in C\backslash\R^{+}$,
\begin{align}
\left(\lam W^\top W+\frac{SX^\top XS}{n}-zI_d\right)^{-2}\asymp-(\lam W^\top W+x_dS^2-zI_d)^{-2}(x_d'(z)S^2-I_d),
\end{align}
where  $x_d(z)$ is the unique solutions of a certain fixed point equation independent of $X$, and $x_d'(z):=dx_d(z)/dz$. Furthermore, there exists $x(z)$ such that $x_d(z)\to x(z)$, $x_d'(z)\to x'(z)$. 
\end{lemma}
\begin{proof}[Sketch of the proof] Since this lemma can be proved following the same steps as the proof of theorem 3.1 (b) in \cite{dobriban2020wonder}, here we only provide a sketch of the proof.\\
\quad
Step 1. (First order deterministic equivalent) Denote
\begin{align*}
f(z,W)&:=(\lam W^\top W+x_dS^2-zI_d)^{-1}\\
g(z,W,X)&:=\left(\lam W^\top W+\frac{SX^\top XS}{n}-zI_d\right)^{-1},
\end{align*} where $x_d(z)$ is the unique solution of the fixed point equation 
\begin{align*}
1-x_d=\frac{x_d}{n}\tr[S^2(\lam W^\top W+x_dS^2-z I_d)^{-1}].
\end{align*}
Then we have from theorem 1 in \cite{rubio2011spectral} that (here we need the finite $8+\eta$ moment assumption)
\begin{align*}
f(z,W)\asymp g(z,W,X).
\end{align*}
Furthermore, for a sequence of $W$ and fixed $\tau,z$, when $n,p,d\to\infty$ proportionally, it can be verified that $x_d(z)\to x(z)$ and $x(z)$
satisfies the fixed point equation 
\begin{align}
1-x=\delta\left(1-\int_{0}^{\infty}\frac{[\lam(t-\tau)-z] dH_\tau (t)}{xt^2+\lam(t-\tau)-z}\right).\nnum
\end{align} 

Step 2. (Second order deterministic equivalent) Using the same technique as in  the proof of theorem 3.1 (b) in \cite{dobriban2020wonder}, it can be proved that $$f'(z,W)\asymp g'(z,W,X).$$ for all $z\in E$.

Step 3. (Explicit expressions for the derivatives) For invertible $A(z)$, we have $$\frac{dA}{dz}=-A^{-1}\frac{dA}{dz}A^{-1}.$$
Therefore
\begin{align*}f'(z,W)&=-(\lam W^\top W+x_dS^2-zI_d)^{-1}(x_d'(z)S^2-I_d)(\lam W^\top W+x_dS^2-zI_d)^{-1}\\
&=-(\lam W^\top W+x_dS^2-zI_d)^{-2}(x_d'(z)S^2-I_d)
\end{align*}
and
\begin{align*}
g'(z,W,X)&=\left(\lam W^\top W+\frac{SX^\top XS}{n}-zI_d\right)^{-2}.
\end{align*}
Also, it can be shown that $x_d'(z)\to x'(z)$ as in the proof of theorem 3.1 (b) in \cite{dobriban2020wonder}. 
\end{proof}
Similarly, we can prove Lemma \ref{2llm1} and \ref{2llm4} in the general case via the two steps in \eqref{EMMlemmaf0}, \eqref{eMMlemma2steps2}. Since the proofs are almost the same as Lemma \ref{2llm2} (and in fact even simpler), we omit them  for simplicity. Also, here we only mention one difference. When bounding $\Delta_1$ from \eqref{eMMdelta1}, we need to first replace $M$ by $\E M$ (or $\E_X M)$ and move the expectation outside the trace operator, e.g. in the proof of \ref{2llm1},
\begin{align*}
\Delta_1&:=\lim_{\tau\to0}\lim_{d\to\infty}\frac{1}{d}\left|\tr[\E M(\E M^\top-\E M^\top_{\tau})]\right|\\
&\leq\lim_{\tau\to0}\lim_{d\to\infty}\frac{1}{d}\E\left|\tr[\E M(M^\top-M^\top_{\tau})]\right|.
\end{align*}
Then all results follow the same argument as in (1) of the proof of Lemma \ref{2llm2}.

\begin{proof}[Proof of Lemma \ref{2llm3} (general case)]
By definition, we have
\begin{align}
\E\|\tM\|^2_{F}
&=\E \tr \left\|W^\top R\frac{WX^\top }{n}\right\|_{F}^2\nnum\\
&= \E \tr\left[R\frac{WX^\top  XW^\top }{n}R\right]
=\E \frac{1}{n}\tr \left[ R-\lam R^2\right].\label{etmtmdetonlyf}
\end{align}
Denote ${WX^\top  XW^\top }/{n}$ by $Q_1$ and $(W^\top W)^{1/2}X^\top X (W^\top W)^{1/2}/n$ by $Q_2$. Since $Q_1$ and $Q_2$ 
have the same non-zero eigenvalues, their Limiting Spectral Distributions (LSD) (if one of them exists) only differ from a constant scale and a mass at $0$, i.e.,
$$LSD(Q_2)=\pi LSD(Q_1)+(1-\pi)\delta_0.$$
Since the LSD of $W^\top W$ converges to $\pi\delta_1+(1-\pi)\delta_0$ almost surely, from \cite{silverstein1995strong} theorem 1.1, we know that the LSD of $Q_2$ almost surely weakly converges to a nonrandom distribution.
Therefore,  the LSD of $Q_1$ also almost surely weakly converges to a nonrandom distribution and $\frac{1}{n}\tr[{WX^\top  XW^\top }/{n}+\lam I_p]^{-i}$ for $i=1,2$ almost surely converge to some nonrandom limits. Thus, by the bounded convergence theorem, we know \eqref{etmtmdetonlyf} converges to a nonrandom limit independent of the exact distribution of $X$ (which only requires $X$ to have i.i.d. zero mean, unit variance entries). Since we have proved Lemma \ref{2llm3} in the Gaussian case, the general case follows directly.
\end{proof}

\subsection{Proof of Lemmas \ref{etm0}---\ref{fnormewm}}
\label{pflems2}
\subsubsection{Proof of Lemma \ref{etm0}}
\begin{proof}
It is clear that $\|\E\tM\|_F^2\le\E_{W}\|\E_X\tM\|_F^2$. 
Thus we only need to prove the second convergence.
By definition,
\begin{align*}
\E_X \tM
&=
\E_X W^\top R \frac{WX^\top}{n}.
\end{align*}
Since $X$ has i.i.d. rows, by switching the $i$-th and $j$-th row of $X$, it is readily verified that $\E_X\tM = \E_X W^\top(n^{-1} WX^\top XW^\top+\lambda I_p)^{-1}WX^\top/n$ has identically distributed columns. Note that $\E_X \tM$ is a $d\times n$ matrix, it is thus enough to prove:
\begin{align}
n\|\E_X \tM_{\cdot1}\|_2^2=n\left\|\E_X W^\top R \frac{W(X_{1\cdot})^\top}{n}\right\|^2_2\overset{u}{\longrightarrow} 0,
\label{etmlemmaf1}
\end{align}
where $\overset{u}{\longrightarrow}$ denotes convergence uniformly in $W$.
For notational simplicity, we denote the column vector $W(X_{1\cdot})^\top$ by $\tilde{x}$ (formed by taking the first row $X_{1\cdot}$ of $X$), $X_{-1\cdot}W^\top$ by $\tilde{X}$ (formed by taking the complement of the first row $X_{1\cdot}$ of $X$) and $(\tX^\top \tX/n+\lam I_p)$ by $C$. Let also $F =  \iC\tx\tx^\top\iC/n $. Then
\begin{align}
\eqref{etmlemmaf1}&=\frac{1}{n}\left\|\E_X W^\top\left(\lam I_p+\frac{\tX^\top \tX}{n}+\frac{\tx\tx^\top}{n}\right)^{-1}\tx\right\|_2^2\nnum\\
&=\frac{1}{n}\left\|\E_XW^\top\left(\iC-\frac{ F }{1+\tx^\top\iC \tx/n}\right)\tx \right\|_2^2\nnum\\
&=\frac{1}{n}\left\|\E_X W^\top\left(\frac{ F }{1+ f }\right)\tx \right\|_2^2,\label{etmlemmaf2}
\end{align}
where in the second line we used the Sherman-Morrison formula,
 and the last line follows from the fact that $\iC$ and $\tx$ are independent and $\E\tx=0$. 
 Let us denote $f =  \tx^\top\iC\tx/n $.
To prove that \eqref{etmlemmaf2} $\overset{u}{\longrightarrow}0$, we only need to prove the following:
\begin{align} &(1). \text{\qquad} \frac{1}{n}\left\|\E_XW^\top\left(\frac{ F }{\E_{\tx} (1+ f )}\tx\right)\right\|_2^2\overset{u}{\longrightarrow} 0  \label{etmlemma2steps2}  \\&(2). \text{\qquad} \frac{1}{n}\left\|\E_X W^\top\left(\frac{ F }{1+ f }\tx\right)-\E_XW^\top\left(\frac{ F }{\E_{\tx} (1+ f )}\tx\right)\right\|_2^2\overset{u}{\longrightarrow} 0. \label{etmlemma2steps1}  \end{align}
(1). For any fixed $W$, since $\tx$ and $\iC$ are independent, \begin{align}\E_{\tx}\frac{\tx^\top\iC\tx}{n}=\E_{\tx}\frac{x^\top (W^\top\iC W)x}{n}=\frac{\tr(W^\top\iC W)}{n}=\frac{\tr(\iC)}{n}.\label{etmlemmaedeno}\end{align}
Denote $1+\tr(\iC)/n$ by $c_0$, and $W^\top\iC W$ by $\tC$. Then
\begin{align}
&\frac{1}{n}\left|\E_X\left(W^\top\frac{ F }{\E_{\tx} (1+ f )}\tx\right)_i\right|^2=\frac{1}{n^3}\left(\E_X \frac{1}{c_0}e_i \tC xx^\top \tC x\right)^2\nnum\\
&=\frac{1}{n^3}\left(\E_{X} \sum_{j,k,l=1}^{d}\frac{1}{c_0}e_i \tC_{ij} x_jx_k \tC_{kl} x_l\right)^2
=\frac{1}{n^3}\left(\E_{\tX} \sum_{j=1}^{d}\frac{1}{c_0}\tC_{ij}\tC_{jj}\E x_j^3\right)^2\nnum\\
&\leq\frac{(\E X_{11}^3)^2}{n^3}\left(\E_{\tX} \sum_{j=1}^{d}\tC_{ij}^2\cdot \sum_{j=1}^{d}\tC_{jj}^2\right),\nnum
\end{align}
where the last line follows from the Jensen inequality, Cauchy-Schwartz inequality and the fact that $c_0\geq 1$.
Summing up all coordinates and noting that $0\preceq\tC\preceq I_d/\lam$, we get
\begin{align*}
\text{L.H.S. of \eqref{etmlemma2steps2}\phantom{.}}&\leq\frac{(\E X_{11}^3)^2}{n^3}\E_{\tX}\left(\|\tC\|_F^2\cdot \sum_{j=1}^{d}\tC_{j,j}^2\right)\leq\frac{(\E X_{11}^3)^2}{ n^3}\cdot\frac{d}{\lam^2}\cdot\frac{d}{\lam^2}\leq\frac{(\E X_{11}^3)^2d^2}{\lam^4 n^3}\overset{u}{\longrightarrow} 0.
\end{align*}
(2).  By definition, theL.H.S. of \eqref{etmlemma2steps1} equals
\begin{align*}
&\frac{1}{n}\left\|\E_X W^\top\left[\left(\frac{ F }{1+ f }\right)-\left(\frac{ F }{\E_{\tx} (1+ f )}\right)\right]\tx\right\|_2^2
=\frac{1}{n^3}\left\|\E_X W^\top\left[\frac{\iC\tx\tx^\top\iC( f -\E_{\tx}  f )}{(1+ f )(\E_{\tx} 1+ f )}\right]\tx\right\|_2^2\\
&\leq\frac{1}{n^3}\E_X\left\| W^\top{\iC\tx\tx^\top\iC}\tx\right\|_2^2\cdot\E_X \left[\frac{( f -\E_{\tx}  f )}{(1+ f )(\E_{\tx} 1+ f )}\right]^2\\
&\leq\frac{1}{n^3}\E_X\left\|{\tC xx^\top\tC}x\right\|_2^2\cdot \E_{\tX} \mathrm{Var}_x(x^\top\tC x/n)
\end{align*}
For the first term in the last line, note that $0\preceq\tC\preceq I_d/\lam$,
$$\E_X\|{\tC xx^\top\tC}x\|_2^2=\E_X(x^\top\tC xx^\top{\tC^2 xx^\top\tC}x)\leq\E_x(x^\top x)^3/\lam^4=O(n^3),$$
where the last equality is due to the fact that there are $O(n^3)$ terms of the form $x_i^2x_j^2x_k^2$, with  $1\leq i,j,k\leq d$ in $(x^\top x)^3$.
Now, it is enough to prove $\E_{\tX}\mathrm{Var}_{x}(x^\top\tC x/n)\overset{u}{\longrightarrow} 0$.
\begin{lemma}\label{varlemmagen}
Suppose that $x=(x_1,...,x_d)$ has i.i.d. entries satisfying $\E x_i=0$, $\E x_i^2=1$. Let $A\in\R^{d\times d}$. Then we have (see e.g. \citealp{bai2009spectral,couillet2011random} and \citealp{mei2019generalization} Lemma B.6.)
$$\mathrm{Var}(x^\top A x)=\sum_{i=1}^{d}A_{ii}^2(\E x_1^4-3)+\|A\|_F^2+\tr(A^2).$$ 
\end{lemma}
Using  Lemma \ref{varlemmagen} above and recalling that $0\preceq\tC\preceq I_d/\lam$, we have
\begin{align*}
\E_{\tX}\mathrm{Var}_{\tx}(x^\top\tC x)&=\E_{\tX}\sum_{i=1}^{d}\tC_{ii}^2(\E x_1^4-3)+\|\tC\|_F^2+\tr(\tC^2)\\
&=\E_{\tX}\sum_{i=1}^{d}\tC_{ii}^2(\E x_1^4-3)+2\|\tC\|_F^2
\leq \frac{d|\E x_1^4-3|}{\lam^2}+\frac{2d}{\lam^2}=O(n).
\end{align*}
Therefore $\E_{\tX}\mathrm{Var}_{x}(x^\top\tC x/n)=O(1/n)\overset{u}{\longrightarrow} 0$ and we finished the proof.
\end{proof}

\subsubsection{Proof of Lemmas \ref{fnormewtm}, \ref{fnormewm}}
The results in these two lemmas follow from using matrix identities (e.g., the Woodbury identity) for the target matrices and deterministic equivalent results for orthogonal projection (Haar) matrices.
\begin{proof}
By definition,
\begin{align*}
\E_W{\tM}=\E_{W}W^{\top}\left(\frac{WX^{\top}XW^{\top}}{n}+\lam\right)^{-1}\frac{WX^{\top}}{n}.
\end{align*}
Let $A=X^{\top}X/n+\lam_1$, where $0<\lam_1<\lam$ is an arbitrary value. The Woodbury matrix identity states that
$$\left(A^{-1} + UCV \right)^{-1} = A - AU \left(C^{-1} + VAU \right)^{-1} VA.$$
Define $\lam_2:=\lam-\lam_1$ and take $V=U^\top=W$, $C=I/\lam_2$, to get
$$\left(A^{-1} +W^\top W/\lam_2 \right)^{-1} = A - AW^\top \left(\lambda_2 + WAW^\top \right)^{-1} W A.$$
Therefore
 \begin{align}
 C:=&\E_W W^{\top}\left(\frac{WX^{\top}XW^{\top}}{n}+\lam\right)^{-1}W\nonumber
 	=\E_{W}W^\top\left(\lam_2+WAW^\top\right)^{-1}W\\&=A^{-1}-A^{-1}\E_{W}\left(A^{-1} + W^\top W/\lam_2 \right)^{-1}A^{-1}\nonumber\\
&= A^{-1/2}\left[I_d-\E_{W}\left(I_{d} + \frac{A^{1/2}W^\top WA^{1/2}}{\lam_2} \right)^{-1}\right] A^{-1/2}.\label{woodrepr}
\end{align}
Also, 
\begin{align}
\E_{X}\|\E_W\tM\|_{F}^2&=\E_{X}\tr\left(C\frac{X^\top X}{n^2}C^\top\right)
=\E_{X}\frac{1}{n}\tr C(A-\lam_1)C^\top.\label{cal_ewtm}\\
\frac{1}{d}\E_{X}\|\E_W M\|_{F}^2&=\frac{1}{d}\E_{X}\tr\left(C\frac{X^\top XX^\top X}{n^2}C^\top\right)
=\frac{1}{d}\E_{X}\tr C(A-\lam_1)^2 C^\top.\label{cal_ewm}
\end{align}
Now, we define \begin{align*}C_1:&= A^{-1}-A^{-1/2}\left(I_{p} + \frac{\be A}{\lam_2} \right)^{-1}A^{-1/2}\\
&=A^{-1}\left[I_d-\left(I_{p} + \frac{\be A}{\lam_2} \right)^{-1}\right]
=\left(A+\frac{\lam_2}{\be}I_{d}  \right)^{-1},
\end{align*}
where $\be$ is defined in Lemma \ref{detlm}.
Then from Lemmas \ref{detlm}, \ref{replace2}, we know that $C$ can be replaced by $C_1$ when calculating the limits of \eqref{cal_ewtm} and  \eqref{cal_ewm}. Hence
\begin{align*}
	\lim_{d\to\infty}\E_{X}\|\E_W\tM\|_{F}^2&=\lim_{d\to\infty}\frac{1}{n}\E_{X}\tr\left(C_1(A-\lam_1)C_1^\top\right)\\
&=\lim_{d\to\infty}\frac{1}{n}\E_{X}\tr\left[\left(A-\lam_1\right)\left(A+\frac{\lam_2}{\be}I_{d}  \right)^{-2}\right]\\
&\to\delta[\theta_1(\delta,\tlam)-\tlam\theta_2(\delta,\tlam)]=\delta(\tth_1-\tlam\tth_2),
\end{align*}
where $\tlam:=\lam_1+\lam_2/\bar{e}_0$ is a constant independent of the choice of $\lam_1,\lam_2$. The last line follows from  Lemma \ref{tedlimit}, the definition that $A=X^\top X/n+\lam_1$ and the Marchenko-Pastur theorem.
Similarly,
\begin{align*}
	\lim_{d\to\infty}\frac{1}{d}\E_{X}\|\E_WM\|_{F}^2&=\lim_{d\to\infty}\E_{X}\frac{1}{d}\tr\left(C_1(A-\lam_1)^2 C_1^\top\right)\\
	&=\lim_{d\to\infty}\E_{X}\frac{1}{d}\tr\left[\left(A-\lam_1\right)^2\left(A+\frac{\lam_2}{\be}I_{d}  \right)^{-2}\right]\\
	&\to[1-2\tlam\theta_1(\delta,\tlam)+\tlam^2\theta_2(\delta,\tlam)]=1-2\tlam\tth_1+\tlam\tth_2.
\end{align*}
This finishes the proof.
\end{proof}
\begin{lemma}[Weak deterministic equivalent for Haar matrices]\label{detlm}
	Under the above assumptions, we have 
	\begin{align}\left(I_{d} + \frac{A^{1/2}W^\top WA^{1/2}}{\lam_2} \right)^{-1}\overset{w}{\asymp} \left(I_d+\frac{\be  A}{\lam_2}\right)^{-1},\label{detformula}
\end{align}
	where $(\bar e_d,e_d)$ is the unique solution of the system of equations
	\begin{align*}
		\be&=\frac{p}{d}\left(e_d+1-e_d\be\right)^{-1}\\
		e_d&=\frac{1}{d}\tr{A\left(\be A+\lam_2 I_d\right)^{-1}}.
	\end{align*}
\end{lemma}
\begin{proof}[Proof of Lemma \ref{detlm}]
	From properties of sample covariance matrices, we know that the largest eigenvalue of $A_d:=X^\top X/n+\lam_1\in\R^{d\times d}$ converges to $\lam_1+(1+\sqrt{\delta})^2$ almost surely as $d\to\infty$. Therefore,  the sequence of values $\|A_d\|_2$ is bounded almost surely. For a fixed sequence of non-negative symmetric $A_d$ with bounded 2-norm, \eqref{detformula} follows from the proof of theorem 7 (see the sketch of the proof) in \cite{couillet2012random}.  Since $A$ is independent of $W$ and the sequence $\|A_d\|_2$ is bounded almost surely, \eqref{detformula} holds generally.
\end{proof}

\begin{lemma}(Replacing $C$ by $C_1$)\label{replace2}
Under the previous assumptions,
we have
\begin{align}
\lim_{d\to\infty}\E_{X}\frac{1}{n}\tr(C(A-\lam_1)C^\top)&=\lim_{d\to\infty}\E_{X}\frac{1}{n}\tr(C_1(A-\lam_1)C_1^\top)\label{replace2f1}\\
\lim_{d\to\infty}\E_{X}\frac{1}{d}\tr(C(A-\lam_1)^2C^\top)&=\lim_{d\to\infty}\E_{X}\frac{1}{d}\tr(C_1(A-\lam_1)^2C_1^\top).\label{replace2f2}
\end{align}
\end{lemma}
\begin{proof}[Proof of Lemma \ref{replace2}]
Since the proof of two claims is almost the same, for simplicity,
we only present the proof of  \eqref{replace2f2}.

For \eqref{replace2f2}, we define (here $\tDel,\tDel_1,\tDel_2$ are different from those in Appendix \ref{2llm5s})
\begin{align*}
	\tDel&:=\lim_{d\to\infty}\frac{1}{d}\E_X\tr C(A-\lam_1)^2C^\top-\frac{1}{d}\E_{X}\tr C_1(A-\lam_1)^2C_1^\top\\
	&=\lim_{d\to\infty}\frac{1}{d}\E_X\left\{\tr\left[(C-C_1)(A-\lam_1)^2C^\top\right]+\tr\left[ C_1(A-\lam_1)^2(C-C_1)^\top\right]\right\}	\\
	&=:\tDel_1+\tDel_2.
\end{align*}
Also, denote \begin{align*}G_0&:=(I_{d}+A^{1/2}W^\top WA^{1/2}/\lam_2)^{-1}\\G_1&:= (I_{d}+\be A/\lam_2)^{-1}\\
G&:=\E_{W}G_0.\end{align*}
Therefore, from Lemma \ref{detlm}, we know that $G_0\overset{w}{\asymp} G_1$.
Substituting the definition of $C$, $C_1$ into $\tDel_1$, we have
	\begin{align*}
		\tDel_1&=\lim_{d\to\infty}\frac{1}{d}\E_{X}\tr\left[A^{-1/2}(G-G_1)A^{-1/2}(A-\lam_1)^2A^{-1/2}(I_d-G)^\top A^{-1/2}\right]\\
		&=\lim_{d\to\infty}\frac{1}{d}\E_{X,W}\tr\left[A^{-1/2}(A-\lam_1)^2A^{-1/2}(I_d-G)^\top A^{-1}(G_0-G_1)\right].
	\end{align*}
By Lemma \ref{commutative},  we know $A$ and $G$ are simultaneously diagonalizable. Moreover, it is readily verified that $A^{-1/2}(A-\lam_1)^2A^{-1/2}(I_d-G)^\top A^{-1}$ is symmetric and non-negative. Since
\begin{align*}
\|A^{-1/2}(A-\lam_1)^2A^{-1/2}(I_d-G)^\top A^{-1}\|_{2}&=\|(I_d-G)^\top A^{-3/2}(A-\lam_1)^2A^{-1/2}\|_{2}\\&\leq
\|A^{-2}(A-\lam)^2\|_2\\
&\leq\lam^2\|A^{-1}\|_2^2+2\lam\|A^{-1}\|_2+1\leq4,
\end{align*}
we have by Lemma \ref{detlm} and the bounded convergence theorem that $\tDel_1\to0$.
Similarly, we can prove that $\tDel_2\to0$ and these conclude the proof of \eqref{replace2f2}.

\end{proof}

\begin{lemma}[Commutativity of $A$ and $G$] \label{commutative}
Under the previous definitions, $A$ and $G$ are simultaneously diagonalizable, and therefore there are commutative.
\end{lemma}
\begin{proof}[Proof of Lemma \ref{commutative}]
Let $A=U\Gamma U^\top$ be the spectral decomposition. By definition and equation \eqref{woodrepr}, 
\begin{align*}
G&=\E_{W}(I_{d}+A^{1/2}W^\top WA^{1/2}/\lam_2)^{-1}\\
&=I_{d}-A^{1/2}\E_W W^\top\left(\lam_2+WAW^\top\right)^{-1}WA^{1/2}\\
&=I_{d}-U\Gamma^{1/2}\E_W (WU)^\top\left[\lam_2+WU\Gamma(WU)^\top\right]^{-1}WUD^{1/2}U^\top\\
&=I_{d}-U\Gamma^{1/2}\left[\E_W  W^\top\left(\lam_2+W\Gamma^\top W^\top\right)^{-1}W\right]\Gamma^{1/2}U^\top,
\end{align*}
where the last line is due to $W\overset{d}{=}WU$. Now it suffices to show that $\E_W  W^\top\left(\lam_2+W\Gamma^\top W^\top\right)^{-1}W$ is a diagonal matrix. 

Write $W=(w_1,w_2,..,w_d)$, where $w_i$, $1\leq i\leq d$ are the columns of $W$. Denote the $i$-th $(1\leq i\leq d)$ diagonal entry of $\Gamma$ by $\ga_i$ and define $W_{-i}:=(w_1,...,w_{i-1},w_{i+1}...,w_d)$ to be the matrix obtained by removing the $i$-th column from $W$. Then, for any $1\leq i\neq j\leq d$
\begin{align}
&\E_{W}(W^\top(\lam_2+W\Gamma W^\top)^{-1}W)_{ij}
=\E_{W} w_i^\top\left(\lam_2+\sum\limits_{k=1}^{d} \ga_k w_k w_k^\top\right)^{-1}w_j\nnum\\
&=\E_{W_{-j}}\E_{w_j|W_{-j}} w_i^\top\left(\lam_2+\sum\limits_{k=1}^{d} \ga_k w_k w_k^\top\right)^{-1}w_j
=0\label{ewmf1},\end{align}
where the last line follows from the symmetry of $w_j$. Therefore,  $\E_{W}(W^\top(\lam_2+W\Gamma W^\top)^{-1}W)$ is a diagonal matrix and hence $A$ and $G$ are simultaneously diagonalizable.
\end{proof}
\begin{lemma}[Convergence of $\be$]\label{tedlimit}
Under the previous assumptions and notations,
suppose that $(\bar e_d(A),e_d(A))$ is the unique solution of
\begin{align}
	\be&=\frac{p}{d}\left(e_d+1-e_d\be\right)^{-1}\label{conofted1}\\
	e_d&=\frac{1}{d}\tr{A\left(\be A+\lam_2 I_d\right)^{-1}}\label{conofted2}.\end{align}
Then $(\bar e_d(A),e_d(A))\to (\bar e_0, e_0)$ almost surely, where $(\bar e_0,e_0)$ is the unique solution of 
\begin{align}
	\bar{e}_0&=\pi\left(e_0+1-e_0\bar{e}_0\right)^{-1}\label{conofted3}\\
	e_0&=\frac{1}{\bar{e}_0}\left[1-\frac{\lam_2}{\bar{e}_0}\theta_1\left(\delta,\lam_1+\frac{\lam_2}{\bar{e}_0}\right)\right]\label{conofted4}.\end{align}
Furthermore, for any decomposition $\lam=\lam_1+\lam_2$, we have 
\begin{align}
\lam_1+\lam_2/\bar{e}_0=
\lam+\frac{1-\pi}{2\pi}\left[\lam+1-\ga+\sqrt{(\lam+\ga-1)^2+4\lam}\right].\label{allequal}
\end{align}
\end{lemma}
We introduce $\lam_1,\lam_2$ only to ensure the invertibility of $A$ and the uniform boundedness of $\|A^{-1}\|_2$. Equation \eqref{allequal} shows that different decompositions of $\lam$ do not affect the results.
\begin{proof}
Plugging \eqref{conofted2} into \eqref{conofted1}, we obtain
\begin{align}	\frac{1}{d}\tr\left(\frac{A}{A+\lam_2/\be}\right)=\frac{p/d-\be}{1-\be}.\label{conofted5}
\end{align}
The uniqueness of the solution is guaranteed by theorem 7 in \cite{couillet2012random}. 
Now, define for $x\in\R$,
$$g_d(x):=\frac{1}{d}\tr\left(\frac{A}{A+\lam_2/x}\right)-\frac{p/d-x}{1-x}.$$ 
$$g(x):=\left[1-\frac{\lam_2}{x}\theta_1\left(\delta,\lam_1+\frac{\lam_2}{x}\right)\right]-\frac{\pi-x}{1-x}.$$ 
First, we consider the case when $0<\pi<1$. Noting that  $g_d(0+)=-p/d$, $g_d(1-)>0$ and $g_d(x)$ is increasing on $(0,1)$, it follows that $g_d(x)$ has a unique zero $\bar e_d$ on $(0,1)$.  Similarly, we can conclude that $g(x)$ has a unique zero $\bar e_0$ on $(0,1)$.

Since $A=X^\top X/n+\lam_1$ and $p/d\to\pi$, by applying the Marchenko-Pastur theorem, we know that $g_d(x)\to g(x)$ for any $x\in(0,1)$, almost surely. Since $g_d(x)$ are increasing functions on $(0,1)$, we further have, on any closed interval in $(0,1)$, $g_d(x)$ converges uniformly to $g(x)$, almost surely.

 Due to the uniformly convergence of $g_d(x)$ on any closed interval and the fact that the zeros of $g_d(x), g(x)$ are in $(0,1)$, we conclude that the zeros of $g_d(x)$ converge to the zero of $g(x)$ almost surely, i.e., $\be\to\bar{e}_0$ almost surely, and hence $(\be,e_d)\to(\bar e_0,e_0)$ almost surely. 

 Plugging  \eqref{conofted3} into \eqref{conofted4}, we obtain
$$\left[1-\frac{\lam_2}{\bar{e}_0}\theta_1\left(\delta,\lam_1+\frac{\lam_2}{\bar{e}_0}\right)\right]=\frac{\pi-\bar{e}_0}{1-\bar{e}_0}, $$
Thus, 
\begin{equation}\theta_1\left(\delta,\lam_1+\frac{\lam_2}{\bar{e}_0}\right)=\frac{1-\pi}{\lam_2(1/\bar{e}_0-1)}\label{theta1equation}.
\end{equation}
Plugging the explicit expression of $\theta_1(\delta,\lam_1+\lam_2/\bar{e}_0)$ into \eqref{theta1equation}  and reorganizing the result, we obtain
\begin{align*}
\pi\left(\lam_1+\frac{\lam_2}{\bar{e}_0}\right)^2-[(1+\pi)\lam+(1-\pi)(1-\ga)]\left(\lam_1+\frac{\lam_2}{\bar{e}_0}\right)+\lam[\lam+(1-\delta)(1-\pi)]=0.
\end{align*}
Thus, defining $$h(x):=\pi x^2-[(1+\pi)\lam+(1-\pi)(1-\ga)] x+\lam[\lam+(1-\delta)(1-\pi)],$$
we get that $\lam_1+\lam_2/\bar{e}_0$ is a solution of $h(x)=0$. Since $\bar{e}_0\in(0,1)$, $\lam_1+\lam_2/\bar{e}_0>\lam$. Note that $h(\lam)=-\delta(1-\pi)^2\lam\leq0$. From the properties of quadratic functions, we conclude that $\lam_1+\lam_2/\bar{e}_0$ is the larger solution of $h(x)=0$. Therefore,
\begin{align*}
	\lam_1+\lam_2/\bar{e}_0=
	\lam+\frac{1-\pi}{2\pi}\left[\lam+1-\ga+\sqrt{(\lam+\ga-1)^2+4\lam}\right]=:\tlam.
\end{align*}
If $\pi=1$, it can be readily verified that the unique solution of \eqref{conofted3}, \eqref{conofted4} is $$(\bar{e}_0,{e}_0)=(1,1-\lam_2\theta_1(\delta,\lam_1+\lam_2)).$$
Thus, equation \eqref{allequal} follows directly.
As for the convergence of $(\bar e_d,\e_d)$, since $g_d(x)<0$ for $x<p/d$, the zero of $g_d(x)$ (i.e. $\bar e_d$) is larger than $p/d$ and hence converges to $1$ as $p/d\to\pi=1$. Therefore, we have shown that $\be\to\bar{e}_0$, and hence $e_d\to e_0$ follows  from equation \eqref{conofted2} and the Marchenko-Pastur theorem. This finishes the proof.
\end{proof}

\subsection{Proof of Theorem \ref{2lmon}}

\begin{proof}
In the following proof, we will calculate the derivatives of variances and bias separately and obtain the results in Table \ref{allmonotonicity} based on them.
Throughout the proof, we denote $c:=\delta (1+\sigma^2/\alpha^2)+1$ for simplicity.

(1) MSE.  Plugging the explicit expressions of $\theta_1,\theta_2,\lambda^*$ into equation \eqref{biasf} and taking derivatives with respect to $\pi$ and $\delta $ separately 
\begin{align}
&\mathrm{\frac{d}{d\pi}}\lim_{d\to\infty}\mse(\lambda^*)
=\mathrm{\frac{d}{d\pi}}[\alpha^2(1-\pi+\lambda^*\pi\theta_1)+\sigma^2]\nnum\\
&=\alpha^2\mathrm{\frac{d}{d\pi}}\left(\frac{2\delta -c+\sqrt{c^2-4\gamma}}{2\delta }\right)
=-\frac{\alpha^2}{\sqrt{c^2-4\gamma}}\label{dmsedpimon}\leq0,
\end{align}
Thus, $\lim_{d\to\infty}\mse(\lambda^*)$ is monotonically decreasing as a function of $\pi$. Similarly,
\begin{align}
\mathrm{\frac{d}{d\delta }}\lim_{d\to\infty}\mse(\lambda^*)&
=\mathrm{\frac{d}{d\delta }}[\alpha^2(1-\pi+\lambda^*\pi\theta_1)+\sigma^2]
=\frac{\alpha^2}{2\delta ^2}\left(1+\frac{2\gamma-c}{\sqrt{c^2-4\gamma}}\right)\label{dmseddelmon}.
\end{align}
Since  $c^2-4\gamma\geq c^2-4\gamma(c-\gamma)=(c-2\gamma)^2$, the derivative is larger than 0. Therefore, $\lim_{d\to\infty}\mse(\lambda^*)$ is increasing as a function of $\delta $.

(2) Bias$^2$. 
Since $(\lim_{d\to\infty}\mse(\lambda^*)-\sigma^2)^2=\alpha^2\lim_{d\to\infty}\bias^2(\lambda^*)$, the monotonicity of the MSE implies the monotonicity of Bias$^2$.

(3) Var. Note that $\var(\lambda^*)=\mse(\lambda^*)-\bias^2(\lambda^*)-\sigma^2$
\begin{align*}
	\mathrm{\frac{d}{d\pi }}\lim_{d\to\infty}\var(\lambda^*)&
=\mathrm{\frac{d}{d\pi }}\left(\lim_{d\to\infty}\mse(\lambda^*)-\lim_{d\to\infty}\bias^2(\lambda^*)-\sigma^2\right)\\
&=\mathrm{\frac{d}{d\pi }}\left[\lim_{d\to\infty}(\mse(\lambda^*)-\sigma^2)\left(1-\frac{1}{\alpha^2}\lim_{d\to\infty}(\mse(\lambda^*)-\sigma^2)\right)\right]\\
&=\left(\mathrm{\frac{d}{d\pi }}\lim_{d\to\infty}\mse(\lambda^*)\right)\left(1-\frac{2}{\alpha^2}\lim_{d\to\infty}(\mse(\lambda^*)-\sigma^2)\right)\\
&=-\frac{\alpha^2}{\sqrt{c^2-4\gamma}}
\cdot\left(\frac{\delta \sigma^2/\alpha^2+1-\sqrt{c^2-4\gamma}}{\delta }\right).
\end{align*}
Since $\sqrt{c^2-4\gamma}$ is decreasing as $\pi$ increases, simple calculations reveal that, as a function of $\pi$, $\lim_{d\to\infty}\var(\lambda^*)$ is monotonically increasing when $\delta \geq2\alpha^2/(\alpha^2+2\sigma^2)$, while it is increasing on $(0,[2+\delta (1+2\sigma^2/\alpha^2)]/4]$ and decreasing on $([2+\delta (1+2\sigma^2/\alpha^2)]/4,1]$  when $\delta <2\alpha^2/(\alpha^2+2\sigma^2)$. 
Similarly,
\begin{align*}
	\mathrm{\frac{d}{d\delta }}\lim_{d\to\infty}\var(\lambda^*)&
=\left(\mathrm{\frac{d}{d\delta }}\lim_{d\to\infty}\mse(\lambda^*)\right)\left(1-\frac{2}{\alpha^2}\lim_{d\to\infty}(\mse(\lambda^*)-\sigma^2)\right)\\
&=\frac{\alpha^2}{2\delta ^2}\left(1+\frac{2\gamma-c}{\sqrt{c^2-4\gamma}}\right)\cdot\left(\frac{\delta \sigma^2/\alpha^2+1-\sqrt{c^2-4\gamma}}{\delta }\right).
\end{align*}
From \eqref{dmseddelmon}, we know the first term in the last line is non-negative. 
 Plugging in the expression of $c$, it follows from some simple calculations that, as a function of $\delta$, $\lim_{d\to\infty}\var(\lambda^*)$ is monotonically decreasing  when $\pi<0.5$, while it is increasing on $(0,2(2\pi-1)/[1+2\sigma^2/\alpha^2]]$ and decreasing on $(2(2\pi-1)/[1+2\sigma^2/\alpha^2],+\infty)$  when $\pi>0.5$.

(4) $\Slab$. Plugging the definition of $\theta_1$, $\theta_2$ and $\lambda^*$ into $\Slab(\lambda^*)$, we obtain
\begin{align*}
	\lim_{d\to\infty}\Slab(\lambda^*)&
	=\sigma^2\pi\delta (\theta_1-\lambda^*\theta_2)\\
	&=\frac{\sigma^2}{2\lambda^*}\left(\sqrt{c^2-4\gamma}-\frac{(\gamma+1)\lambda^*+(\gamma-1)^2}{\sqrt{c^2-4\gamma}}\right)
	=\frac{\sigma^2c}{\sqrt{c^2-4\gamma}}.
\end{align*}
Therefore, as a function of $\pi$, 	$\lim_{d\to\infty}\Slab(\lambda^*)$ is monotonically increasing on $(0,1]$. Since 
\begin{align*}
\mathrm{\frac{d}{d\delta }}	\lim_{d\to\infty}\Slab(\lambda^*)=\frac{2\pi\sigma^2(1-\delta (1+\sigma^2/\alpha^2))}{\sqrt{c^2-4\gamma}},
\end{align*}
it follows that $\lim_{d\to\infty}\Slab(\lambda^{*})$ is increasing on $(0,\alpha^2/(\sigma^2+\alpha^2)] $ and decreasing on\\ $(\alpha^2/[\sigma^2+\alpha^2],+\infty)$ as a function of $\delta $.

(5) $\Sini$. We have
\begin{align*}
\mathrm{\frac{d}{d\pi}}\lim_{d\to\infty}\Sini(\lambda^*)&=\mathrm{\frac{d}{d\pi}}\alpha^2\pi(1-\pi)(1-\lambda^*\theta_1)^2\\
&=\alpha^2(1-\lambda^*\theta_1)[(1-2\pi)(1-\lambda^*\theta_1)-2\pi(1-\pi)(\lambda^*\theta_1)']\\
&=2\alpha^2(1-\lambda^*\theta_1)\left(\frac{4\delta \pi^3-2c^2\pi+c^2}{(c^2-4\gamma\pi+c\sqrt{c^2-4\gamma})\sqrt{c^2-4\gamma}}\right).
\end{align*}
Since $f(t):=4\delta t^3-2c^2t+c^2$ satisfies the following properties (a). $f(0)=c^2>0$, (b). $f(1)=4\delta -c^2\leq-(\delta (1+\sigma^2/\alpha^2)-1)^2\leq0$, (c). $f(t)$ has a unique extremum on $(0,+\infty)$; we know that $f$ is decreasing and has a unique zero on $(0,1]$. Therefore, $\lim_{d\to\infty}\Sini(\lambda^*)$ is unimodal as a function of $\lambda_1$.
Now, taking derivatives with respect to $\delta $,
\begin{align*}
	\mathrm{\frac{d}{d\delta }}\lim_{d\to\infty}\Sini(\lambda^*)&=\mathrm{\frac{d}{d\pi}}\alpha^2\pi(1-\pi)(1-\lambda^*\theta_1)^2\\
	&=-2\alpha^2\pi(1-\pi)(1-\lambda^*\theta_1)(\lambda^*\theta_1)'\\
	&=-\frac{\alpha^2}{\delta ^2}(1-\pi)(1-\lambda^*\theta_1)\left(\frac{2\gamma-c}{\sqrt{c^2-4\gamma}}+1\right)\leq0,
\end{align*}
where the last inequality follows from the fact that $c^2-4\gamma\geq(2\gamma-c)^2$. Thus, $\lim_{d\to\infty}\Sini(\lambda^*)$ is monotonically decreasing as a function of $\delta $.

\end{proof}

\subsection{Proof of Proposition \ref{lin}}
\begin{proof}[Proof of Proposition \ref{lin}]
 It is known that if the design matrix $X$ has independent entries of zero mean and unit variance, then as $n,p\to\infty$ proportionally, i.e., $p/n\to \gamma>0$, the MSE converges almost surely to the expression \citep{tulino2004random,couillet2011random,dobriban2018high}

\begin{equation}\mse(\gamma)=\gamma m_\gamma(-\lam^*)=\gamma \theta_1(\gamma,\lam^*)+\sigma^2,\label{linearexprmse}\end{equation}
where $\lambda^* = \gamma/\alpha^2$ is the limit of the optimal regularization parameters, and $m_\gamma$ is the Stieltjes transform  of the limiting eigenvalue distribution $F_\gamma$ of $\hSigma=n^{-1} X^\top X$, i.e., the Stieltjes transform of the standard Marchenko-Pastur distribution with aspect ratio $\gamma$.

Furthermore, as shown in the proof of theorem 2.2 in \cite{liu2019ridge}, the specific forms of the bias and variance are: with $\theta_i:=\theta_i(\ga,\lam)$, $i=1,2$
\begin{align}
&(a).\, \bias^2=\alpha^2\int\frac{\lambda^2}{(x+\lambda)^2}dF_\gamma(x)=\alpha^2\lam^2\theta_2,
\,\, &(b).\,
\var=\gamma
\int\frac{x}{(x+\lambda)^2}dF_\gamma(x)=\ga(\theta_1-\lam\theta_2).\label{linearexprbiasvar} \end{align}
Therefore, we can obtain the explicit formulas of the bias, variance and MSE by plugging $\lam^*=\lam/\alpha^2$, equation \eqref{exprtheta1}, \eqref{exprtheta2} into equation \eqref{linearexprmse}, \eqref{linearexprbiasvar}. All results in Proposition \ref{lin} can be derived by calculating the derivatives as we have done in the proof of theorem \ref{2lmon}. However, since the proofs are simpler in this special case, we present them here for the convenience of readers. Throughout this proof, we denote $1/\alpha^2$ by $c$.

{\bf MSE:}
Let $\tau:=(1/\alpha^2-1)\ga-1$, substitute \eqref{exprtheta1} into \eqref{linearexprmse} and take derivatives:
\begin{align*}
\frac{\mathrm{d}}{\mathrm{d}\ga}\mse(\gamma)=\frac{\mathrm{d}}{
\mathrm{d}\ga}[\gamma \theta_1(\ga,\lam^*)+\sigma^2]
&=\frac{\mathrm{d}}{\mathrm{d}\ga}\frac{\tau+\sqrt{\tau^2+4c \gamma^2}}
{2c\gamma}\\
&=\frac{\tau+1+[\tau(\tau+1)+4c\ga^2]/\sqrt{\tau^2+4c\ga^2}-(\tau+\sqrt{\tau^2+4c\ga^2})}{2c\ga^2}\\
&=\frac{\tau+\sqrt{\tau^2+4c\ga^2}}{2c\ga^2\sqrt{\tau^2+4c\ga^2}}\geq0.
\end{align*}
Thus, the MSE is strictly increasing as $\ga$ increases.

{\bf Bias:}
Plugging equation \eqref{exprtheta2} into \eqref{linearexprbiasvar}(a) and denoting $1/\ga$  by $x$ , we have
\begin{align}
\bias^2(\gamma)=\alpha^2\lam^{*2}\theta_2(\gamma,\lam^*)
&=\frac{\alpha^2}{2}\left(1-\frac{1}
{\gamma }
+
\frac{c\gamma(\gamma+1)+(\gamma-1)^2}
{\gamma\sqrt{(c+1)^2\gamma^2+2(c-1)\gamma+1}}\right)\nnum\\
&=\frac{\alpha^2}{2}\left(1-x
+
\frac{c(x+1)+(x-1)^2}
{\sqrt{(c+1)^2+2(c-1)x+x^2}}\right)=:\frac{\alpha^2}{2}(1+f(x,c)).\nnum
\end{align}
Thus it is enough to show that $f(x,c)$ is decreasing for $x> 0$.
Taking derivatives with respect to $x$, we have after some calculations that
\begin{align*}
\frac{d^2}{dx^2}f(x,c)=\frac{6(c+1)^3+12c^2 x}{({(c+1)^2+2(c-1)x+x^2})^{\frac{5}{2}}}>0,     \forall x>0.
\end{align*}
Thus, for any fixed $c>0$, $f(x,c)$ is a strictly convex function with respect to $x$ on $(0,+\infty)$. \\
Furthermore, fixing $c>0$ and letting $x\to\infty$, we get
\begin{align*}
\lim_{x\to\infty}f(x,c)&=\lim_{x\to\infty}-x
+
\frac{c(x+1)+(x-1)^2}
{\sqrt{(c+1)^2+2(c-1)x+x^2}}\\
&=\lim_{x\to\infty}
\frac{c(x+1)+(x-1)^2-{x\sqrt{(c+1)^2+2(c-1)x+x^2}}}
{\sqrt{(c+1)^2+2(c-1)x+x^2}}\\
&=\lim_{x\to\infty}
\frac{(x^2+(c-2)x+c+1)-{x(x+c-1)\sqrt{1+\frac{4c}{(x+c-1)^2}}}}
{x}\\
&=\lim_{x\to\infty}
\frac{(x^2+(c-2)x+c+1)-{x(x+c-1){(1+O(\frac{1}{x^2}))}}}
{x}\\
&
=\lim_{x\to\infty}
\frac{-x+O(1)}{x}
=-1.
\end{align*}
Combining the results above and using the fact that a strictly convex function with a finite limit is strictly decreasing, it follows that $f(x,c)$ is both strictly decreasing and convex. Therefore, we have proved that $\bias^2(\gamma)$ is strictly increasing on $(0,+\infty)$.\\
{\bf Variance:}
Similarly, plugging \eqref{exprtheta1}, \eqref{exprtheta2} into \eqref{linearexprbiasvar} (b), we get
\begin{align*}
\var(\gamma)&=-\frac{1}{2}+\frac{(c+1)\gamma+1}{2\sqrt{(\gamma(1-c)-1)^2+4\gamma^2 c}}\\
&=-\frac{1}{2}+\frac{(c+1)\gamma+1}{2\sqrt{(c+1)^2\gamma^2+2(c-1)\gamma+1}}.
\end{align*}
Let $x=(c+1)\gamma$ and $t=\frac{c-1}{c+1}$. Then it suffices to prove the unimodality of
\begin{align*}g_t(x):&=\frac{x+1}{\sqrt{x^2+2tx+1}},   &x\in(0,+\infty),t\in(-1,1).\end{align*}
Differentiating $g_t(x)$ with respect to $x$ gives
\begin{align*}
\frac{d}{dx}g_t(x)&=\frac{(1-t)(1-x)}{({x^2+2tx+1})^\frac{3}{2}}.
\end{align*}
Since $g_t'(x)>0$ when $x<1$, and $g_t'(x)<0$ when $x>1$, we see that $g_t(x)$ is strictly increasing on $(0,1)$ and strictly decreasing on $(1,+\infty)$. Correspondingly, by changing $x$ back to $\gamma$, $c$ back to $1/\alpha^2$, we have shown that $\var(\gamma)$ is strictly increasing on $(0,\alpha^2/(\alpha^2+1))$ and strictly decreasing on $[\alpha^2/(\alpha^2+1),+\infty)$. Moreover, noting that\begin{align*}
\lim_{\gamma\to+\infty}\var(\gamma)&=\lim_{\gamma\to+\infty}-\frac{1}{2}+\frac{(c+1)\gamma+1}{2\sqrt{(c+1)^2\gamma^2+2(c-1)\gamma+1}}\\
&=\lim_{\gamma\to+\infty}-\frac{1}{2}+\frac{(c+1)\gamma}{2(c+1)\gamma}
=0.
\end{align*}
Therefore, we conclude that $\var(\gamma)$ is a unimodal function converges to zero at infinity with a unique maximum point at $\alpha^2/(\alpha^2+1)$.  Finally, proposition \ref{lin}(4) can be obtained by evaluating $\var(\ga)$ and $\bias^2(\ga)$ at $\alpha^2/(\alpha^2+1)$.

\end{proof}

\subsection{Proof of Theorem \ref{bias_var_for_fixed_lambda}}

\begin{proof}
(1). By plugging the  expression of $\theta_1$  into equation \eqref{biasf} and taking derivatives with respect to $\delta$ in equation \eqref{biasf}, we obtain
\begin{align*}
\deri{\pi}\lim_{d\to\infty}\Bl&=\deri{\pi}\alpha^2(1-\pi+\lam\pi\theta_1)^2\\
&={\alpha^2}(1-\pi+\lam\pi\theta_1)\left(\frac{\lam+\ga-1}{\sqrt{(-\lam+\ga-1)^2+4\lam\ga}}-1\right)<0
,\end{align*}
where the inequality follows from the facts that $1-\pi+\lam\pi\theta_1=\int{[\lam+(1-\pi)x]/(x+\lam)dF_{\ga}(x)}\geq0$ and $(\lam+\ga-1)^2-[(-\lam+\ga-1)^2+4\lam\ga]=-4\lam<0$. Similarly, 
\begin{align*}
\deri{\delta}\lim_{d\to\infty}\Bl&=\deri{\delta}\alpha^2(1-\pi+\lam\pi\theta_1)^2\\
&=\frac{\alpha^2}{\delta^2}(1-\pi+\lam\pi\theta_1)\left(\lam+1-\frac{\lam^2+(\ga+2)\lam+(1-\ga)}{\sqrt{(-\lam+\ga-1)^2+4\lam\ga}}\right)>0.
\end{align*}

Therefore,  the limiting $\bias^2$ is monotonically increasing as a function of $\delta$ and monotonically decreasing as a function of $\pi$.

(2).  Denote $\lambda^*=\delta(1-\pi+\sigma^2/\alpha^2)$ as before. From \eqref{varf}, we have \begin{align*}\lim_{d\to\infty}\var(\lam)&=\alpha^2\pi\left\{1-\pi+\frac{\lam}{\delta}+\left[(\pi-1)(2\lam-\delta)+\frac{\lam(\lam-\ga+1)}{\delta}+\frac{\delta\sigma^2}{\alpha^2}\right]\theta_1+\right.\\&
\left.\lam\left[
\lam-\delta\left(1-\pi+\sigma^2/\alpha^2\right)\right]\theta_2\right\}\\
&=\alpha^2\pi\left\{1-\pi+\frac{\lam}{\delta}+\left[\left(2\pi-2+\frac{\lam-\ga+1}{\delta}\right)\lam+\lambda^*\right]\theta_1(1,\lam)+\lam(\lam-\lambda^*)\theta_2\right\}
\end{align*}
Now suppose $\delta=1/\pi$ and let $\lam\to0+$, \begin{align*}\lim_{\lam\to0}\lim_{d\to\infty}\var(\lam)&=\lim_{\lam\to0}\alpha^2\pi\left\{1-\pi+\frac{\lam}{\delta}+\left[\left(2\pi-2+\frac{\lam-\ga+1}{\delta}\right)\lam+\lambda^*\right]\theta_1+\lam(\lam-\lambda^*)\theta_2\right\}\\
&=\lim_{\lam\to0}\alpha^2\pi[1-\pi+\lambda^*(\theta_1(1,\lam)-\lambda\theta_2(1,\lam))]\\
&=\lim_{\lam\to0}\alpha^2\pi\left[1-\pi+\lambda^*\left(\frac{4\lam}{(\sqrt{\lam^2+4\lam}+\lam)^2\sqrt{\lam^2+4\lam}}\right)\right]
=\lim_{\lam\to0}O(\frac{1}{\lam^{1/2}})=\infty.
\end{align*}
{Finally, letting $\lam\to0$ in Theorem \ref{sobolthm}, we obtain after some similar calculations that $V_{si}$ and $V_{sli}$ go to infinity while $V_s,V_{i}$ and $V_{sl}$ converge to some finite limits as $d\to\infty,\lam\to0$.}
\end{proof}

\subsection{Proof of Theorem \ref{ridgeopt}}

\begin{proof}[Proof of Theorem \ref{ridgeopt}]
For notational simplicity, we sometimes denote the 2-norm of vectors and the Frobenius norm of matrices by $\|\cdot\|$ in this proof.
From definition (\ref{optdef}), we have 
\begin{align*}
\beta_{opt}:&=\argmin_{\beta}\E_{p(\theta|XW^\top ,W,Y)}\E_{x,\ep}[(Wx)^\top \beta-(x^\top \theta+\ep)]^2\\
&=\argmin_{\beta}\E_{p(\theta|XW^\top ,W,Y)}\E_{x}[(Wx)^\top \beta-x^\top \theta]^2\\
&=\argmin_{\beta}\E_{p(\theta|XW^\top ,W,Y)}\E_{x}[({W}^\top \beta-\theta)^\top xx^\top ({W}^\top \beta-\theta)]^2\\
&=\argmin_{\beta}\E_{p(\theta|XW^\top ,W,Y)}\|{W}^\top \beta-\theta\|_2^2\\
&=W\E_{p(\theta|XW^\top ,W,Y)}\theta.
\end{align*} 
Thus, we have proved equation (\ref{optexpr}).  

Now, to prove \eqref{asymopt}, it suffices to do the following:
\begin{align*}
&(1). \text{\hspace{1em} Calculate the posterior $p(\theta|XW^\top ,W,Y)$.}\\
&(2). \text{\hspace{1em} Bound the difference between $\beta_{opt}$ and $\hat{\beta}$.}
\end{align*}

(1). Let $W_\perp =f(W)\in\R^{(d-p)\times d}$ be a deterministic orthogonal complement of $W$, such that $W_\perp ^\top W_\perp +W^\top W=I_{d}$. Then we have
\begin{align}
&p(\theta|XW^\top ,W,Y)\propto p(\theta)p(XW^\top ,W,Y|\theta)\nonumber
=p(\theta)p(XW^\top ,W|\theta)p(Y|XW^\top ,W,\theta)\nonumber\\
&\propto \exp\left(-\frac{\|\theta\|_2^2}{2\alpha^2/d}-\frac{\|XW^\top \|_{F}^2}{2}\right)\cdot p(Y|XW^\top ,W,\theta)\label{thetapost}\\ 
&\propto\exp\left(-\frac{\|\theta\|_2^2}{2\alpha^2/d}\right)\cdot \int p(Y|XW_\perp ^\top ,XW^\top,W, \theta)\cdot p(XW_\perp ^\top |XW^\top ,W,\theta)d(XW_\perp ^\top ) \nonumber\\
&\propto\exp\left(-\frac{\|\theta\|_2^2}{2\alpha^2/d}\right)\cdot \int \exp\left(-\frac{\|Y-X\theta\|_2^2}{2\sigma^2}\right)\cdot p(XW_\perp ^\top |XW^\top ,W,\theta)d(XW_\perp ^\top ),\nonumber\end{align}
where in the second line we used the facts that $XW^\top $ and $W$ are independent conditioned on $\theta$, $XW^\top $ has i.i.d. $\N(0,1)$ entries, and $W$ is uniformly distributed over partial orthogonal matrices.  Denote $XW^\top$ by $A$ and $XW_\perp ^\top $ by $A_1$. Then using the fact that $A,A_1$ have i.i.d. $\N(0,1)$ entries and $A$, $A_1$ are independent conditioned on $W$ and $\theta$, we obtain
\begin{align*}
p(XW_\perp ^\top |XW^\top ,W,\theta)
&\propto\exp\left(-\frac{\|XW_\perp ^\top \|_F^2}{2}\right)=\exp\left(-\frac{\|A_1\|_F^2}{2}\right).
\end{align*}
Therefore,
\begin{align*}
&p(Y|XW^\top ,W,\theta)=
\int p(Y|X,\theta)\cdot p(XW_\perp ^\top |XW^\top ,W,\theta)d(XW_\perp ^\top )\\
&=
\int \exp\left(-\frac{\|(Y-AW\theta)-A_1W_\perp \theta\|_2^2+\sigma^2\|A_1\|_F^2}{2\sigma^2}\right)dA_1\\
&=
\exp\left(-\frac{\|Y-AW\theta\|_2^2}{2\sigma^2}\right)
\int\exp\left\{-\frac{\sigma^2\|A_1\|_F^2+\|A_1W_\perp \theta\|_F^2-2\tr[A_1W_\perp \theta(Y-AW\theta)^\top ]}{2\sigma^2}\right\}dA_1
\end{align*}
Further denote $W\theta$ by $\tth $ , $W_\perp \theta$ by $\tth_ 1$ and $(Y-A\tth )\tth_ 1^\top (\tth_ 1\tth_ 1^\top +\sigma^2)^{-1}$ by $B$. For a fixed $Y$, conditioned on $A,W,\theta$, by separating $A_1$ into $n$ rows, applying Fubini's theorem and using properties of the p.d.f. of a normal distribution, we get
\begin{align}
&p(Y|XW^\top ,W,\theta)\propto\exp\left(-\frac{\|Y-A\tth \|_2^2}{2\sigma^2}\right)\cdot\nonumber\\
&
\int\exp\left\{-\frac{\tr[(A_1-B)(\tth_ 1\tth_ 1^\top +\sigma^2)(A_1-B)^\top ]}{2\sigma^2}+\frac{\tr[B(\tth_ 1\tth_ 1^\top +\sigma^2)B^\top ]}{2\sigma^2}\right\}dA_1\label{ypost1}\\
&\propto \exp\left(-\frac{\|Y-A\tth \|^2}{2\sigma^2}+\frac{\tr[B(\tth_ 1\tth_ 1^\top +\sigma^2)B^\top ]}{2\sigma^2}\right)\cdot\det(\tth_ 1\tth_ 1^\top +\sigma^2)^{-\frac{n}{2}}\nonumber\\
&=\exp\left(-\frac{\|Y-A\tth \|^2}{2\sigma^2}
+
\frac{\tr[(Y-A\tth )\tth_ 1^\top (\tth_ 1\tth_ 1^\top +\sigma^2)^{-1}\tth_ 1(Y-A\tth )^\top ]}{2\sigma^2}\right)
\cdot
\det(\tth_ 1\tth_ 1^\top +\sigma^2)^{-\frac{n}{2}}.\nonumber\end{align}
Let $\tth_ 1=UDV^\top $ be the SVD of $\tth_ 1$. Then 	
\begin{align*}\det(\tth_ 1\tth_ 1^\top +\sigma^2)&=\det(DD^\top +\sigma^2)=(\|\tth_ 1\|^2+\sigma^2)\prod_{i=2}^{d-p} \sigma^2\propto\|\tth_ 1\|^2+\sigma^2.\\
\tth_ 1^\top (\tth_ 1\tth_ 1^\top +\sigma^2)^{-1}\tth_ 1&=D^\top (DD^\top +\sigma^2)^{-1}D=\|\tth_ 1\|^2(\|\tth_ 1\|^2+\sigma^2)^{-1}.
\end{align*}
Thus,
\begin{align}
p(Y|XW^\top ,W,\theta)	
&\propto
\exp\left\{[-1+\|\tth_ 1\|^2(\|\tth_ 1\|^2+\sigma^2)^{-1}]\frac{\|Y-A\tth \|^2}{2\sigma^2}
\right\}
\cdot
(\|\tth_ 1\|^2+\sigma^2)^{-\frac{n}{2}}.\label{ypost2}
\end{align}
Finally, by substituting \eqref{ypost2} into \eqref{thetapost}, we  obtain the posterior:
\begin{align}
&p(\theta|XW^\top ,W,Y)
\propto	
 \exp\left(-\frac{\|\theta\|_2^2}{2\alpha^2/d}-\frac{\|Y-A\tth \|^2}{2(\|\tth_ 1\|^2+\sigma^2)}\right)
 \cdot
 (\|\tth_ 1\|^2+\sigma^2)^{-\frac{n}{2}}\nonumber\\
 &=
  \exp\left(-\frac{\|\tth \|^2}{2\alpha^2/d}-\frac{\|Y-A\tth \|^2}{2(\|\tth_ 1\|^2+\sigma^2)}\right)
 \cdot
 \exp(-\frac{\|\tth_ 1\|^2}{2\alpha^2/d})
\cdot
 (\|\tth_ 1\|^2+\sigma^2)^{-\frac{n}{2}}\label{post2}.\end{align}
\\

(2). 
Since $\beta_{opt}=W\E_{p(\theta|XW^\top ,W,Y)}\theta=
\E_{p(\theta|XW^\top ,W,Y)}\tth $, it suffices to calculate the posterior mean of $\tth $. In equation (\ref{post2}), we can see that conditioned on $\tth_ 1$, $\tth $ follows a normal distribution. Moreover, using the same technique as in equation (\ref{ypost1}), it is not hard to  verify that the expectation of $\tth $ conditioned on $\tth_ 1$ is \begin{equation}\left(\frac{d}{\alpha^2}+\frac{A^\top A}{\|\tth_ 1\|^2+\sigma^2}\right)^{-1}\frac{A^\top Y}{\|\tth_ 1\|^2+\sigma^2}.\label{postmean0}\end{equation}

 Now changing $A$, $\tth $ back to $X,W$, we obtain \begin{equation}\E_{p(\theta|\tth_1 , XW^\top ,W,Y)}\tth = \left[\frac{WX^\top XW^\top }{n}
+
\frac{(\|\tth_ 1\|^2+\sigma^2)d}{n\alpha^2}\right]^{-1}\frac{WX^\top Y}{n}.\label{condmean}\end{equation} 
Therefore, the conditional expectation of $\tth $ is the ridge estimator with $\lambda=\lambda(\tth_ 1):=(\|\tth_ 1\|^2+\sigma^2)d/[n\alpha^2]$. Thus $\beta_{opt}$ is in fact a weighted ridge estimator. Besides, note  that $W_\perp \theta$ has i.i.d. $\N(0,\alpha^2/d)$ entries. Hence letting $\chi^2(k)$ be a chi squared random variable with $k$ degrees of freedom,
\begin{equation}\|\tth_ 1\|^2=\|W_\perp \theta\|^2=\frac{\alpha^2}{d}\sum_{i=1}^{d-p}\left(\frac{\sqrt{d}W_\perp \theta}{\alpha}\right)_i^2\overset{d}{=}\frac{\alpha^2}{d}\chi^2(d-p)\overset{w}{\longrightarrow}\alpha^2(1-\pi).\label{theta1norm}
\end{equation}
Denote $\tilde R^* = (\frac{WX^\top XW^\top }{n}+\lambda^*)^{-1}$.
Thus, we may guess that the posterior of $\|\tth_ 1\|$ is close to $\alpha^2(1-\pi)$ with high probability and 
$$\beta_{opt}\approx\left[\frac{WX^\top XW^\top }{n}
+
\delta (1-\pi+\sigma^2/\alpha^2)\right]^{-1}\frac{WX^\top Y}{n}=\tilde R^*\frac{WX^\top Y}{n}=\hat{\beta},$$ which is the optimal ridge estimator. 

We  formalize this idea by bounding the difference between $\beta_{opt}$ and $\hbeta$. Denote the conditional mean \eqref{condmean} of $\tth $ given $\|\tth_ 1\|^2=c$ by $R(c)$. Let also $\tilde R_1 = (\frac{WX^\top XW^\top }{n}+\lambda(\tth_ 1))^{-1}$. Then the optimal ridge estimator is $\hbeta=R(\alpha^2(1-\pi))$ and we have
\begin{align*}
&\E_{XW^\top ,W,Y}\|\hbeta-\beta_{opt}\|_2^2=\E_{XW^\top ,W,Y}
\left\|\int [R(\|\tth_ 1\|^2)-R(\alpha^2(1-\pi))]\cdot p(\tth_ 1|XW^\top ,W,Y)d\tth_ 1\right\|^2\\
&\leq\E_{XW^\top ,W,Y,\tth_1}\left
\|[R(\|\tth_ 1\|^2)-R(\alpha^2(1-\pi))\right\|^2
=\E_{XW^\top ,W,Y,\tth_1}\left\|\left[\tilde R^*-\tilde R_1\right]\frac{WX^\top Y}{n}\right\|^2,
\end{align*}
where we used the Jensen inequality in the second line. Note that $A^{-1}-B^{-1}=A^{-1}(B-A)B^{-1}$ and omit the subscripts of the expectation. The above equals
\begin{align}
\E\left\|\left[\tilde R^*\left(\lambda^*-\lambda(\tth_ 1)\right)
\tilde R_1\right]\frac{WX^\top Y}{n}\right\|^2. \label{postbound}
\end{align}
Furthermore, by Cauchy-Schwartz inequality and the fact that $\|AB\|_F\leq \|A\|_2\|B\|_F$, the square of this quantity is upper bounded by
\begin{align}
&\E\left(\lambda^*-\lambda(\tth_ 1)\right)^4\cdot
\E\left\|\left[\tilde R^*
\cdot
\tilde R_1\right]\frac{WX^\top Y}{n}\right\|^4.\nnum\\
&\leq 
\E\left(\lambda^*-\lambda(\tth_ 1)\right)^4
\cdot\E\left\|\frac{1}{\lam(\tth_1)\lam^*}
\frac{WX^\top Y}{n}\right\|^4\nnum\\
&\leq \frac{1}{\lam^{*4}}
\E\left(\lambda^*-\lambda(\tth_ 1)\right)^4
\cdot\left(\E\frac{1}{\lam(\tth_1)^8}\right)^{1/2}\cdot\left(\E\left\|
\frac{WX^\top Y}{n}\right\|^8\right)^{1/2}.\label{postbound2}
\end{align}
Recall that $Y=X\theta+\Ep$ and note that $\E\|WX^\top X/{n}\|^k_{F}$, $\E\|WX^\top/\sqrt{n}\|^k_{F}$, $\E\|\theta\|_2^k$, $\E\|\Ep/\sqrt{n}\|_2^k$ are all uniformly bounded (as $d\to\infty$) for any non-negative integer $k$ because of the Gaussian asusmption and the boundeness of  moments of Wishart matrices \citep{muirhead2009aspects,bai2009spectral} etc. It follows directly by several applications of the Cauchy-Schwartz inequality that $\E \|WX^\top Y/n\|^8$ is bounded as $d\to\infty$, i.e.,  the third term in the R.H.S. of \eqref{postbound2} is bounded.

If $\sigma>0$, then $\lam(\tth_1)\geq\lam(0)=\sigma^2d/[n\alpha^2]>0$, and hence the second term in \eqref{postbound2} is bounded. If $\sigma=0$ and $\pi<1$, by the definition of $\lam(\tth_1)$ and integration in the polar system, it is readily verified that
$$\E\frac{1}{\lam(\tth_1)^8}=\E\frac{n^8}{(x_1^2+\cdots+x_{d-p}^2)^8}=\frac{n^8}{\prod_{k=1}^8(d-p-2k)}\overset{d\to\infty}{\longrightarrow} C_1<\infty.$$
Therefore, the second term in the R.H.S of \eqref{postbound2} is also bounded. Also
\begin{align*}
\eqref{postbound2}&\leq C\cdot
\lim_{d\to\infty}\E_{XW^\top ,W,Y,\theta}\left(\lambda^*-\lambda(\tth_ 1)\right)^4\\
&\leq C_2\cdot\lim_{d\to\infty}\left(\left[\lambda^*-\frac{d}{n}(1-\frac{p}{d}+\sigma^2/\alpha^2)\right]^4+\E{[\lambda(\tth_ 1)-\E\lam(\tth_1)]^4}\right)=0,
\end{align*}
where $C,C_2$ are some finite constants and the last equality follows from equation \eqref{theta1norm} and properties of the chi-square distribution.

Therefore, we have proved that $\E\|\hbeta-\beta_{opt}\|^2\to0$ and the optimal ridge estimator $\hbeta$ is asymptotically optimal. 
\end{proof}

\subsection{Proof of Theorem \ref{fixthetathm}}
\begin{proof}[Proof of Theorem \ref{fixthetathm}]
We only need to make small changes in the proof of theorems \ref{sobolthm} and \ref{2lthm1} to prove  theorem \ref{fixthetathm}. We first take $\lim_{d\to\infty}\bias^2(\lam)$ as an example. Similar to \eqref{2lbias}
\begin{align}
 \bias^{2}(\lambda) &=\E_{x}\left[\E_{X,W,\Ep}\left(x^\top  M \theta+x^\top  \tM \Ep\right)-x^\top  \theta\right]^2 \nnum\\
 	&=\E_{x}\left[x^\top (\E_{X,W} M-I) \theta\right]^{2}
=\E_{x}\tr((\E M^\top-I)xx^\top(\E{M}-I)\theta\theta^\top)\nnum\\
&=\tr((\E M^\top-I)(\E{M}-I)\theta\theta^\top)\label{fixthetaf1}.\end{align}
Note that we have shown $\E M$ is a multiple of identity in Lemma \ref{2llm1} (under Gaussian assumption), therefore
\begin{align}
\eqref{fixthetaf1}&=\left[(\E M^\top-I)(\E{M}-I)\right]\cdot\tr(\theta\theta^\top)=\frac{1}{d}\tr\left[(\E M^\top-I)(\E{M}-I)\right]\cdot\tr(\theta\theta^\top).\label{fixthetaf2}
\end{align}
From Lemma \ref{2llm1}, we know the first term in the R.H.S. of \eqref{fixthetaf2} converges to $\pi^2(1-\lam\theta_1(\pi\delta,\lam))^2$. Note that $\tr(\theta\theta^\top)\overset{d}{=}\alpha^2(\sum_{i=1}^{d}x_i^2)/d$, where $x_i\sim\N(0,1)$. By the Borel-Cantelli lemma and the concentration inequality for $\chi^2$-variables, 
we have $\tr(\theta\theta^\top)\overset{a.s.}{\to}\alpha^2$. 

Thus, \eqref{fixthetaf2} almost surely converges to $\alpha^2\pi^2(1-\theta_1(\pi\delta,\lam))^2$ and the same asymptotic result for bias as in theorem \ref{2lthm1} holds almost surely over the randomness in $\theta$.

From this example, we can see that the results in theorem \ref{sobolthm} and \ref{2lthm1} will automatically hold in the non-random  setting if we can separate $\theta$ from other variables (e.g. $\E M$) in \eqref{2lbias} ---\eqref{2lmse}, \eqref{sobolvsf}---\eqref{sobolvslif} by showing that the matrices which are multiplied by $\theta\theta^\top$ are in fact a multiple of identity. For instance, in \eqref{2lbias}, since $\theta\theta^\top$ is multiplied by a multiple of identity $(\E M^\top-I)(\E M-I)$, the same result follows. 

To generalize other results in theorem \ref{sobolthm} and \ref{2lthm1} to the almost sure setting, from \eqref{2lbias}---\eqref{2lmse}, \eqref{sobolvsf}---\eqref{sobolvslif}, we can see that it is enough to show the following matrices are all multiples of the identity:
\begin{align*}
(a).\text{\phantom{a}} \E M^\top M, &&(b).\text{\phantom{a}} \E_{W}(\E_{X}M^\top \E_XM), &&(c). \text{\phantom{a}} \E_{X}(\E_{W}M^\top \E_WM). 
\end{align*}

(a). $\E M^\top M$. We will denote $R =\left({WX^\top  XW^\top /n}+\lambda I_p\right)^{-1}$ in what follows. By the definition of $M$ and the fact that $XW^\top$ and $X(I-W^\top W)$ (denoted by $X_2$) are two independent matrices with Gaussian entries for any fixed $W$ with orthogonal rows 
\begin{align}
\E M^\top M&= \E \frac{X^\top XW^\top}{n}
 R^{2}\frac{WX^\top X}{n}\nnum\\
&=\E \frac{W^\top WX^\top XW^\top}{n}
 R^{2}\frac{WX^\top XW^\top W}{n}\label{mtmmult1}
\\&+
\E \frac{X_2^\top XW^\top}{n}
 R^{2}\frac{WX^\top X_2}{n}\label{mtmmult2}.\end{align}
For \eqref{mtmmult1}, note that $W$ and $XW^\top$ are independent, so
\begin{align}
\eqref{mtmmult1}&=\E_{W} W^\top \left[\E_{XW^\top}\frac{WX^\top XW^\top }{n}
 R^{2}\frac{WX^\top XW^\top }{n}\right]W\nnum.\end{align}
Since $\E_{W}W^\top A W=\tr(A)\cdot I_d/d$ for any constant matrix $A$, it follows directly that \eqref{mtmmult1} is a multiple of identity.
For \eqref{mtmmult2}, note that $XW^\top$ and $X(I-W^\top W)$ are independent conditioned on $W$, thus
\begin{align}
\eqref{mtmmult2}&=\E_{W}\E_{X_2|W} X_2^\top \left[\E_{XW^\top}\frac{XW^\top }{n}
 R^{2}\frac{WX^\top }{n}\right]X_2.\label{mtmmult3}
\end{align}
Let $XW^\top=UDV^\top$ be the SVD. Since $XW^\top$ has i.i.d. $N(0,1)$ entries, we can assume $U$ follows the Haar measure and is independent of $DV^\top$. Therefore, with $X_2=X(I-W^\top W)$
\begin{align*}
\eqref{mtmmult3}&=
\E_{W}\E_{X_2|W} X_2^\top \left[\E_{U,D,V}U\frac{DV^\top}{n}
\left(\frac{VD^\top DV^\top}{n}+\lam I_{p}\right)^{-2}\frac{VD^\top}{n}U^\top\right]X_2\\
&=\E_{W}\E_{X_2|W} X_2^\top \left\{\frac{1}{n}\tr\left[\E_{D,V}\frac{DV^\top}{n}
\left(\frac{VD^\top DV^\top}{n}+\lam I_{p}\right)^{-2}\frac{VD^\top}{n}\right]\right\}X_2\\
&=c_0\cdot\E_{W}\E_{X_2|W} X_2^\top X_2
=c_1\cdot\E_{W}(I-W^\top W)
=c_2\cdot I_{d},
\end{align*}
where $c_0,c_1,c_2$ are some constants.
Combining \eqref{mtmmult1} and \eqref{mtmmult2}, we have proved $\E M^\top M$ is a multiple of identity.

(b). $\E_{W}(\E_{X}M^\top \E_XM)$. 
Since $W$ and $XW^\top$ are independent, similarly
\begin{align*}
\E_{X}M
&=\E _{X}
W^\top R\frac{WX^\top XW^\top}{n}W
=
W^\top\left[\E_{XW^\top|W} R\frac{WX^\top XW^\top}{n}\right]W\\
&=
W^\top \E_{V}V\E_{D}\left(\frac{D^\top D}{n}+\lam I_{p}\right)^{-1}\frac{D^\top D}{n}V^\top W
=c_0W^\top W,
\end{align*}
where $c_0$ is a constant (different from previous constants) and the last line follows from the fact that $\E_V VAV^\top=\tr(A)/p\cdot I_p$ for any constant matrix $A$. Therefore
\begin{align*}
\E_{W}(\E_{X}M^\top\E_{X}M)=\E_{W}c_0^2W^\top WW^\top W=c_0^2\cdot\E_{W}W^\top W=c_1\cdot I_d,
\end{align*}
where $c_1$ is a constant.

(c). $\E_{X}(\E_{W}M^\top \E_WM).$ Let $X^\top X=U\Gamma U^\top$ be the spectral decomposition of $X^\top X$. By the definition of $M$, we have
\begin{align*}
\E_{W}M
&=\E _{W}\left[
W^\top RW\right]\frac{X^\top X}{n}\\
&=\E _{W}U\left[
(WU)^\top\left(\frac{(WU)\Gamma (WU)^\top}{n}+\lam I_{p}\right)^{-1}WU\right]\frac{\Gamma U^\top}{n}\\
&=\E _{W}\left[
W^\top\left(\frac{W\Gamma W^\top}{n}+\lam I_{p}\right)^{-1}W\right]\frac{\Gamma U^\top}{n}
=c_0(\Gamma)\cdot I_{d}\cdot\frac{U\Gamma U^\top}{n},
\end{align*}
where $c_0(\Gamma)$ is a constant depending on $\Gamma$, the second line is due to $W\overset{d}{=}WU$ and the last line follows from the proof of Lemma \ref{commutative} (see \eqref{ewmf1}).
Finally, note that $U$ follows the Haar distribution and is independent of $\Gamma$. Thus, 
\begin{align*}
\E_{X}(\E_{W}M^\top \E_WM)&=
\E_{\Gamma,U}c_0(\Gamma)^2\cdot\frac{U\Gamma^2 U^\top}{n}=\E_{U}U\left(\E_{\Gamma}c_0(\Gamma)^2\cdot\frac{\Gamma^2}{n}\right)U^\top=c_1\cdot I_{d},
\end{align*}
where $c_1$ is a constant and  the last equality follows again from the fact that $\E_{U}UAU^{\top}=\tr(A)/d\cdot I_{d}$ for any constant matrix $A$.

\end{proof}

\subsection{Proof of Theorem \ref{nlbiasvardecomp}}

\begin{proof}
By definition, we have
\begin{align*}
\mse(\lam)&:=\E_{\theta,x,W,X,\Ep}(\hat{f}(x)-x^\top\theta)^2+\sigma^2\nonumber\\
&=\E_{\theta,x,W,X,\Ep}\hat{f}(x)^2-2\E_{\theta,x,W,X,\Ep}\hat{f}(x)\cdot x^\top\theta+\E_{\theta,x}(x^\top\theta)^2+\sigma^2.\\
\bias^2(\lam)&:=\E_{\theta,x}|\E_{W,X,\Ep}\hat{f}(x)-x^\top\theta|^2\nonumber\\
&=\E_{\theta,x}(\E_{W,X,\Ep}\hat{f}(x))^2-2\E_{\theta,x,W,X,\Ep}\hat{f}(x)\cdot x^\top\theta+\E_{\theta,x}(x^\top\theta)^2\\
\var(\lam)&:=\E_{\theta,x,W,X,\Ep}|\hat{f}(x)-\E_{X,W,\Ep}\hat{f}(x)|^2\nonumber\\
&=\E_{\theta,x,W,X,\Ep}\hat{f}(x)^2-\E_{\theta,x}(\E_{W,X,\Ep}\hat{f}(x))^2.
\end{align*}

To prove equation \eqref{nlmsef}, \eqref{nlbiasf} and \eqref{nlvarf}, it is thus enough to calculate $\E_{\theta,x,W,X,\Ep}\hat{f}(x)^2$, $\E_{\theta,x,W,X,\Ep}\hat{f}(x)\cdot x^\top\theta$ and $\E_{\theta,x}(\E_{W,X,\Ep}\hat{f}(x))^2$. In Lemma \ref{nllemma1}, \ref{nllemma2} and \ref{nllemma3},
 we calculate these three terms separately. Equation  \eqref{nlmsef}, \eqref{nlbiasf} and \eqref{nlvarf} follow directly from these three lemmas and the fact that $\E_{\theta,x}(x^\top\theta)^2=\alpha^2$.

Using the same technique as in the proof of theorem \ref{2lthm1}, it can be shown that the limiting MSE as a function of $\lam$ has a unique minimum at $\lam^*:=\frac{v^2}{\mu^2}\left[\delta(1-\pi+\frac{\sigma^2}{\alpha^2})+\frac{(v-\mu^2)\ga}{v}\right]$. Furthermore, results in Table \ref{allmonotonicitynonlinear} can be proved in the same way as results in Table \ref{allmonotonicity}. For simplicity, we omit the proof of optimal $\lam^*$ and Table \ref{allmonotonicitynonlinear} here.
 \end{proof}
\begin{lemma}[Asymptotic limit of $\E_{\theta,x}(\E_{W,X,\Ep}\hat{f}(x))^2$]\label{nllemma1}
Under  assumptions in theorem \ref{nlbiasvardecomp}
\begin{align}
\lim_{d\to\infty}\E_{\theta,x}(\E_{W,X,\Ep}\hat{f}(x))^2&=
\alpha^2\pi^2\frac{\mu^4}{v^2}\left(1-\frac{\lam}{v}\theta_1\right)^2,\label{nlef2f}
\end{align}
where $\theta_1:=\theta_1(\ga,\lam/v)$.
\end{lemma}
\begin{lemma}[Asymptotic limit of $\E_{\theta,x,W,X,\Ep}\hat{f}(x)\cdot x^\top\theta$]\label{nllemma2}
Under  assumptions in theorem \ref{nlbiasvardecomp}
\begin{align}
\lim_{d\to\infty}\E_{\theta,x,W,X,\Ep}\hat{f}(x)\cdot x^\top\theta&=\alpha^2\pi\frac{\mu^2}{v}\left(1-\frac{\lam}{v}\theta_1\right),\label{nletrxff}
\end{align}
where $\theta_1:=\theta_1(\ga,\lam/v)$.
\end{lemma}
\begin{lemma}[Asymptotic limit of $\E_{\theta,x,W,X,\Ep}\hat{f}(x)^2$]\label{nllemma3}
Under  assumptions in theorem \ref{nlbiasvardecomp} we have
\begin{align}
\lim_{d\to\infty}\E_{\theta,x,W,X,\Ep}\hat{f}(x)^2&=\alpha^2\pi\left[1-\frac{2(v-\mu^2)}{v}+\left(\rho(1-\pi)-\frac{2\lam\mu^2}
{v^2}\right)\theta_1+\frac{\lam}{v}\left(\frac{\lam\mu^2}{v^2}-\rho(1-\pi)\right)\theta_2\nonumber\right.\\&\left.+\frac{v-\mu^2}{v}\left({1}+{\ga}\theta_1-\frac{\lam\ga}{v}\theta_2\right)\right]+\sig^2\ga\left(\theta_1-\frac{\lam}{v}\theta_2\right),\label{nlf2f}
\end{align}
where $\theta_1:=\theta_1(\ga,\lam/v)$ and $\theta_2:=\theta_2(\ga,\lam/v)$.
\end{lemma}
The proofs of these lemmas proceed by applying the leave-one-out technique and the Marchenko-Pastur law (refer to Lemmas \ref{nlt1expr}---\ref{nlt3expr}), and by leveraging properties of the orthogonal projection matrix and the normal distribution. 
\begin{proof}[Proof of Lemma \ref{nllemma1}]
Different from previous sections, we will denote $R = \left(\frac{\sigma(W X^\top)\sigma(XW^\top)}{n}+\lam\right)^{-1}$ in the proof of Lemmas \ref{nllemma1}---\ref{nllemma3}.
By definition,
\begin{align}
&\E_{\theta,x}(\E_{W,X,\Ep}\hat{f}(x))^2=
\E_{\theta,x}\left[\E_{W,X,\Ep}\sigma(x^\top W^\top)R\frac{\sigma(WX^\top)Y}{n}\right]^2\nnum\\
&=\E_{\theta,x}\left[\E_{W,X}\sigma(x^\top W^\top)R\frac{\sigma(WX^\top)X\theta}{n}\right]^2\nnum\\
&=\frac{\alpha^2}{d}\E_{x}\left\|\E_{W,X}\sigma(x^\top W^\top)R\frac{\sigma(WX^\top)X}{n}\right\|_{F}^2
=:\frac{\alpha^2}{d}\E_{x}\|T_1\|_{F}^2\label{nlt1def},
\end{align}
For $T_1$, we continue
\begin{align}
T_1&=\E_{W,X}\sigma(x^\top W^\top)R\frac{\sigma(WX^\top)}{n}[XW^\top W+X(I-W^\top W)]\nnum\\
&=\E_{W,X}\sigma(x^\top W^\top)R\frac{\sigma(WX^\top)}{n}XW^\top W\nnum\\
&=\E_{W}\sigma(x^\top W^\top)\E_{X}\left[R\frac{\sigma(WX^\top)}{n}XW^\top\right]W,\label{nlT1f}
\end{align}
where the second line follows from the fact that $XW^\top$ and $X(I-W^\top W)$ are independent and $\E_X X(I-W^\top W)$. Denote 
$$D_1:=\E_{X}\left[R\frac{\sigma(WX^\top)}{n}XW^\top\right],$$ then $D_1$ is a constant matrix independent of $W$ since $WX$ has i.i.d. $\N(0,1)$ entries for any  $W$ with orthonormal rows. We can write the vector $x$ as $x=U(\sqrt{d_1},0,...,0)^\top$ $U$ is a random orthogonal matrix following the Haar distribution. Denoting the $i$-th column of $W$ by $W_{\cdot i}$ and substituting $x$ into \eqref{nlT1f}, we get  
\begin{align}
T_1&=
\E_{W}\sigma(x^\top W^\top)D_1W\nnum
=\E_{W}\sigma((\sqrt{d_1},0,...,0)U^\top W^\top)D_1WUU^\top\nnum\\
&=\E_{W}\sigma((\sqrt{d_1},0,...,0)W^\top)D_1WU^\top\nnum
\end{align}
where the last line follows from $W\overset{d}{=}WU$. We can further write this as
\begin{align}
&\E_{W}\sigma(\sqrt{d_1}W_{\cdot1}^\top)D_1(W_{\cdot1},...,W_{\cdot d})U^\top\nnum\\
&=\left(\E_{W}\sigma(\sqrt{d_1}W_{\cdot1}^\top)D_1W_{\cdot1},0,...,0\right)U^\top\nnum\\
&=\left(\E_{W_{\cdot1}}\tr[W_{\cdot1}\sigma(\sqrt{d_1}W_{\cdot1}^\top)D_1],0,...,0\right)U^\top\nnum\\
&=\left(\E_{W_{11}}[W_{11}\sigma(\sqrt{d_1}W_{11})]\tr(D_1),0,...,0\right)U^\top\label{eqnlt1expr},
\end{align}
where the second line follows from the symmetry of $W_{.i}(i\geq2)$ conditional on $W_{\cdot1}$ and the last line is due to the fact that $\E_{W_{\cdot1}}W_{\cdot1}\sigma(\sqrt{d_1}W_{\cdot1}^\top)$ is a multiple of identity since $W_{i,1}$ is symmetric conditional on $W_{j,1}$, $j\neq i$. Therefore, 
\begin{align}
&\E_{\theta,x}(\E_{W,X,\Ep}\hat{f}(x))^2 =\frac{\alpha^2}{d}\E_{x}\|T_1\|^2_{F}\nonumber\\
&=\frac{\alpha^2}{d}\E_{d_1}\left[\E_{W_{11}}W_{11}\sigma(\sqrt{d_1}W_{11})\right]^2\tr(D_1)^2\label{nlefeq}.
\end{align}
Denote $\sqrt{d_1}W_{11}$ by $\tw$. Noting that $\tw\overset{w}{\longrightarrow}\N(0,1)$ and $d_1/d=\|x\|^2/d\overset{a.s.}{\longrightarrow}1$, we obtain
\begin{align*}
&\lim_{d\to\infty}\E_{\theta,x}(\E_{W,X,\Ep}\hat{f}(x))^2 =\lim_{d\to\infty}\frac{\alpha^2}{d^2}\tr(D_1)^2\E_{d_1}\left[\E_{\tw}\tw\sigma\left(\sqrt{\frac{d_1}{d}}\tw\right)\right]^2\\
&=
\lim_{d\to\infty}{\alpha^2}\pi^2\frac{\mu^2}{v^2}\left(1-\frac{\lam}{v}\theta_1\right)^2\E_{d_1/d}\left[\E_{\tw}\tw\sigma\left(\sqrt{\frac{d_1}{d}}\tw\right)\right]^2\\
&=
{\alpha^2}\pi^2\frac{\mu^2}{v^2}\left(1-\frac{\lam}{v}\theta_1\right)^2[\E_{a\sim\N(0,1)}\sigma(a)a]^2=
{\alpha^2}\pi^2\frac{\mu^4}{v^2}\left(1-\frac{\lam}{v}\theta_1\right)^2,
\end{align*}
where the third line can be rigorously justified using the concentration inequality for $\tw$ and $d_1/d$: 
\begin{align*}
&P(|\tw|>t)\leq2e^{-(d-2)t^2/d}\text{\hspace{0.5em}(Levy's lemma)}, &P(|d_1/d-1|>t)\leq 2e^{-dt^2/8}\text{\hspace{0.5em}(concentration for $\chi^2$)}.
\end{align*} the fact that $|\sigma(x)|,|\sigma'(x)|\leq c_1e^{c_2|x|}$, integration by parts and the bounded convergence theorem. Here Levy's lemma refers to usual concentration of the Haar measure \citep{boucheron2013concentration}.
\end{proof}
\begin{proof}[Proof of Lemma \ref{nllemma2}]
By definition, \begin{align*}
&\E_{\theta,x}(\E_{X,W,\Ep}\hat{f}(x)\theta^\top x)=
\E_{\theta,x}\left[\E_{W,X,\Ep}\sigma(x^\top W^\top)R\frac{\sigma(WX^\top)Y}{n}\theta^\top x\right]\\
&=
\E_{\theta,x}\left[\E_{W,X}\sigma(x^\top W^\top)R\frac{\sigma(WX^\top)X\theta}{n}\theta^\top x\right]\\
&=\frac{\alpha^2}{d}
\E_{x}\left[\E_{W,X}\sigma(x^\top W^\top)R\frac{\sigma(WX^\top)Xx}{n} \right]=\frac{\alpha^2}{d}
\E_{x}\tr(T_1 x),
\end{align*}
where $T_1$ is defined in equation \eqref{nlt1def} in the proof of Lemma \ref{nllemma1}. Let us again write $x=U(\sqrt{d_1},0,...,0)^\top$ for an orthogonal $U$, and denote $\tw=\sqrt{d}W_{11}$. Then from equation \eqref{eqnlt1expr} and the fact that $\tw\overset{w}{\longrightarrow}\N(0,1)$,  $d_1/d\overset{a.s.}{\longrightarrow}1$ we have
\begin{align*}
\lim_{d\to\infty}\E_{\theta,x}(\E_{X,W,\Ep}\hat{f}(x)\theta^\top x)&=
\lim_{d\to\infty}\frac{\alpha^2}{d}\E_{d_1}\left(\E_{W_{11}}[W_{11}\sigma(\sqrt{d_1}W_{11})]\tr(D_1),0,...,0\right)U^\top U(\sqrt{d_1},0,...,0)^\top\\
&=
\lim_{d\to\infty}
\frac{\alpha^2}{d}\tr(D_1)\E_{d_1}\E_{W_{11}}[\sqrt{d_1}W_{11}\sigma(\sqrt{d_1}W_{11})]\\
&=\lim_{d\to\infty}{\alpha^2}\pi\frac{\mu}{v}\left(1-\frac{\lam}{v}\theta_1\right)\E_{d_1/d}\E_{W_{11}}\left[\sqrt{\frac{d_1}{d}}\tw\sigma\left(\sqrt{\frac{d_1}{d}}\tw\right)\right]\\
&={\alpha^2}\pi\frac{\mu}{v}\left(1-\frac{\lam}{v}\theta_1\right)\E_{a\sim\N(0,1)}[a\sigma(a)]={\alpha^2}\pi\frac{\mu^2}{v}\left(1-\frac{\lam}{v}\theta_1\right),
\end{align*}
where the third line follows from similar arguments as in the proof of Lemma \ref{nllemma1}.
\end{proof}
\begin{proof}[Proof of Lemma \ref{nllemma3}]
By definition,
\begin{align*}
&\E_{\theta,x,W,X,\Ep}\hat{f}(x)^2=\E_{\theta,x,W,X,\Ep}\left|\sigma(x^\top W^\top)R\frac{\sigma(WX^\top)(X\theta+\Ep)}{n}\right|^2\\
&=
\frac{\alpha^2}{d}\E\left\|\sigma(x^\top W^\top)R\frac{\sigma(WX^\top)X}{n}\right\|_{2}^2
+\sigma^2\E\left\|\sigma(x^\top W^\top)R\frac{\sigma(WX^\top)}{n}\right\|_{2}^2
=:T_2+T_3.
\end{align*}
For $T_2$, we further have that it equals
\begin{align*}
&\frac{\alpha^2}{d}\E_{W,X,x}\left\|\sigma(x^\top W^\top)R\frac{\sigma(WX^\top)XW^\top}{n}\right\|_{2}^2
+\frac{\alpha^2}{d}\E_{W,X,x}\left\|\sigma(x^\top W^\top)R\frac{\sigma(WX^\top)X(I-W^\top W)}{n}\right\|_{2}^2\\
&=\frac{v\alpha^2}{d}\E_{W,X}\left\|R\frac{\sigma(WX^\top)XW^\top}{n}\right\|_{F}^2
+\frac{v\alpha^2}{d}\E_{W,X}\left\|R\frac{\sigma(WX^\top)(I_d-W^\top W)}{n}\right\|_{F}^2\\
&=\frac{v\alpha^2}{d}\E_{W,X}\left\|R\frac{\sigma(WX^\top)XW^\top}{n}\right\|_{F}^2
+\frac{v\alpha^2(d-p)}{d}\E_{W,X}\left\|R\frac{\sigma(WX^\top)}{n}\right\|_{F}^2
=:T_4+T_5,
\end{align*}
where the first and third equations follow from the fact that $XW^\top$ and $X(I-W^\top W)$ are independent conditional on $W$ and $\E_{X}X(I_d-W^\top W)(I_d-W^\top W)X^\top=\tr(I_d-W^\top W)=d-p$. Also, we used that $Wx\sim N(0,I_{p})$ given any orthgonal $W$, $\E_{a\sim\N(0,1)}\sigma(a)=0$, definition \eqref{nlmu1vdef} and the independence of $Wx$, $XW^\top$ and $X(I-W^\top W)$ conditional on $W$. 

Also, due to the independence of $Wx$, $XW^\top$ conditional on $W$, we have
\begin{align*}
T_3=\sigma^2v\E_{W,X}\left\|R\frac{\sigma(WX^\top)}{n}\right\|_{F}^2=\frac{\sigma^2d}{\alpha^2(d-p)}T_5.
\end{align*}
Finally, substituting Lemma \ref{nlt2expr}, \ref{nlt3expr} into $T_4,T_5$ 
, we get
\begin{align*}
&\E_{\theta,x,W,X,\Ep}\hat{f}(x)^2=T_3+T_4+T_5
=T_4+\left(1+\frac{\sigma^2d}{\alpha^2(d-p)}\right)T_5\\&=\frac{v\alpha^2}{d}\E_{W,X}\left\|R\frac{\sigma(WX^\top)XW^\top}{n}\right\|_{F}^2
+\left(\frac{v\alpha^2(d-p)}{d}+v\sigma^2\right)\E_{W,X}\left\|R\frac{\sigma(WX^\top)}{n}\right\|_{F}^2\\
&\to\alpha^2\pi\left[1-\frac{2(v-\mu^2)}{v}+\left(\rho(1-\pi)-\frac{2\lam\mu^2}
{v^2}\right)\theta_1+\frac{\lam}{v}\left(\frac{\lam\mu^2}{v^2}-\rho(1-\pi)\right)\theta_2\nonumber\right.\\&\left.+\frac{v-\mu^2}{v}\left({1}+{\ga}\theta_1-\frac{\lam\ga}{v}\theta_2\right)\right]+\sig^2\ga\left(\theta_1-\frac{\lam}{v}\theta_2\right)
\end{align*}
\end{proof}
\begin{lemma}[Asymptotic behavior of $D_1$]\label{nlt1expr}
Under assumptions in theorem \ref{nlbiasvardecomp}, we have 
\begin{equation}\label{nlt1expr2}\lim_{d\to\infty}\frac{1}{d}\E_{X}\tr\left[R\frac{\sigma(WX^\top)}{n}XW^\top\right]
=\pi\frac{\mu}{v}\left(1-\frac{\lam}{v}\theta_1\right),\end{equation}
where $\theta_1:=\theta_1(\ga,\lam/v)$.
\end{lemma}
\begin{proof}[Proof of Lemma \ref{nlt1expr}]
Since $WW^\top=I_p$, $WX^\top$ has i.i.d. $\N(0,1)$ entries and the L.H.S. of \eqref{nlt1expr2} (if exists) is a constant independent of $W$. 
Denote $\tX:=XW^\top$ for notational simplicity. Also, let $\tX_{\cdot i }$ be the $i$-th column of $\tX$, $\tX_{\cdot-i}$ be the matrix obtained by deleting the $i$-th column of $\tX$.  
By symmetry of $\tX$, it suffices to compute the first diagonal entry of the matrix in the L.H.S. of \eqref{nlt1expr2}. Namely,  we have
\begin{align}
&\phantom{1em}\frac{1}{d}\E_{X}\tr\left[ R\frac{\sigma(WX^\top)}{n}XW^\top\right]
=\frac{p}{d}\E_{\tX}\left[ R\frac{\sigma(\tX^\top)}{n}\tX\right]_{11}\nnum\\
&=\frac{p}{d}\E_{\tX}\left[\frac{\sigma(\tX^\top)}{n}\left(\frac{\sigma(\tX)\sigma(\tX^\top)}{n}+\lam I_{n}\right)^{-1}\tX\right]_{11}\nnum\\
&=\frac{p}{d}\E_{\tX}\left[\frac{\sigma(\tX_{\cdot1})^\top}{n}\left(\frac{\sigma(\tX_{\cdot1})\sigma(\tX_{\cdot1}^\top)}{n}+\frac{\sigma(\tX_{\cdot-1})\sigma(\tX_{\cdot-1}^\top)}{n}+\lam I_{n}\right)^{-1}\tX_{\cdot1}\right]\label{nlt1expr3}.
\end{align} 
Define $C:=[\sigma(\tX_{\cdot-1})\sigma(\tX_{\cdot-1}^\top)/n+\lam I_n]$, $u:=\tX_{\cdot1}$ and $\tu:=\sigma(\tX_{\cdot1})$. By the Sherman-Morrison formula, the above equals
\begin{align}
\frac{p}{d}\E\frac{u^\top}{n}\left(\iC-\frac{\iC \tu\tu^\top\iC/n}{1+\tu^\top\iC\tu/n}\right)u
&=\frac{p}{d}\left[\E\frac{\tu^\top\iC u}{n}-\E\left(\frac{u^\top\iC \tu\tu\iC u/n}{n+\tu^\top\iC\tu}\right)\right]\label{nlt1expr4}.
\end{align}
Since $\tu=\sigma(\tX_{\cdot1})$ has i.i.d. zero mean $v$ variance entries, by theorem 1 in \cite{rubio2011spectral}, the proof of theorem 2.1 in \cite{liu2019ridge} , we know that $C^{-i}$ are determinstically equivalent to certain multiples of the identity matrix. Also, the multiples will converge to certain limits, which can be determined by the Marchenko-Pastur law. Thus, we have after some calculations that 
\begin{align}
\iC=\left(\frac{\sigma(\tX_{\cdot-1})\sigma(\tX_{\cdot-1}^\top)}{n}+\lam I_{n}\right)^{-1}&\asymp\frac{1}{\ga v}\theta_1\left(\frac{1}{\ga},\frac{\lam}{\ga v}\right)\cdot I_n=\frac{\ga}{v}\theta_1\left(\ga,\frac{\lam}{v}\right)+\frac{1-\ga}{\lam}.\label{nlCdeter1}\\
\iCt=\left(\frac{\sigma(\tX_{\cdot-1})\sigma(\tX_{\cdot-1}^\top)}{n}+\lam I_{n}\right)^{-2}&\asymp\frac{1}{\ga^2 v^2}\theta_2\left(\frac{1}{\ga},\frac{\lam}{\ga v}\right)\cdot I_n=\frac{\ga}{v^2}\theta_2\left(\ga,\frac{\lam}{v}\right)+\frac{1-\ga}{\lam^2}.\label{nlCdeter2}
\end{align}
Therefore, it remains to calculate the limit of \eqref{nlt1expr4}.

Since $u\sim\N(0,I_n)$, we have by the strong law of large numbers that $$\limsup \left \|\frac{u u^\top}{n} \right\|_{\tr}=\limsup u^\top u/n<_{a.s.}\infty.$$ Note that $C$ is independent of $u$, thus  we have by \eqref{nlCdeter1} that almost surely for a sequence of $u_k$, $k=1,2,...$
\begin{align*}
&\lim_{d\to\infty}\frac{u^\top\iC u}{n}=\lim_{d\to\infty}\tr\left(\frac{(u u^\top)}{n}\iC\right)\\
&=\lim_{d\to\infty}\left[\frac{\ga}{v}\theta_1\left(\ga,\frac{\lam}{v}\right)+\frac{1-\ga}{\lam}\right]\tr\left(\frac{(uu^\top)}{n}\right)
=\left[\frac{\ga}{v}\theta_1\left(\ga,\frac{\lam}{v}\right)+\frac{1-\ga}{\lam}\right], \textnormal{a.s.}
\end{align*} 
Similar results also hold for other terms: $(i=1,2)$
\begin{align}
\frac{u^\top C^{-i} u}{n}&\overset{a.s.}{\longrightarrow}\left[\frac{\ga}{v^i}\theta_i\left(\ga,\frac{\lam}{v}\right)+\frac{1-\ga}{\lam^i}\right]\label{nlquadc1}.\\
\frac{\tu^\top C^{-i} u}{n}&\overset{a.s.}{\longrightarrow}\mu\left[\frac{\ga}{v^i}\theta_i\left(\ga,\frac{\lam}{v}\right)+\frac{1-\ga}{\lam^i}\right]\label{nlquadc2}.\\
\frac{\tu^\top C^{-i} \tu}{n}&\overset{a.s.}{\longrightarrow} v\left[\frac{\ga}{v^i}\theta_i\left(\ga,\frac{\lam}{v}\right)+\frac{1-\ga}{\lam^i}\right].\label{nlquadc3}
\end{align}
Therefore, simple calculations give
\begin{align}
\left[\frac{\tu^\top\iC u}{n}-\left(\frac{u^\top\iC \tu\tu\iC u/n^2}{1+\tu^\top\iC\tu/n}\right)\right]\to\frac{\mu}{v}\left(1-\frac{\lam}{v}\theta_1\right),\textnormal{a.s.} \label{nlquadc4}
\end{align}
Now, we only need to show that the expectation of the L.H.S. of \eqref{nlquadc4} also converges to its pointwise limit.

\noindent To prove this, we start with bounding  the mean squared error of \eqref{nlquadc1}.
\begin{align*}
\E&\left|\frac{u^\top \iC u}{n}-\left[\frac{\ga}{v}\theta_1+\frac{1-\ga}{\lam}\right]\right|^2\leq
\E\left|\frac{u^\top \iC u}{n}-\E_{u}\frac{u^\top \iC u}{n}\right|^2+\E\left|\E_u\frac{u^\top \iC u}{n}-\left[\frac{\ga}{v}\theta_1+\frac{1-\ga}{\lam}\right]\right|^2\\
&=\frac{1}{n^2}\E_{C}\left[\left(\sum\limits_{i=j=k=l}+\sum\limits_{i=j\neq k=l}+\sum\limits_{i=l\neq j=k}+\sum\limits_{i=k\neq j=k}\right)\iC_{ij}\iC_{kl}\E u_iu_ju_ku_l-\tr(\iC)^2\right]\\
&+\E\left|\frac{\tr(\iC)}{n}-\left[\frac{\ga}{v}\theta_1+\frac{1-\ga}{\lam}\right]\right|^2.
\end{align*}
This can be further bounded by
\begin{align*}
\frac{1}{n^2}K_1\|\iC\|^2_F+\E\left|\frac{\tr(\iC)}{n}-\left[\frac{\ga}{v}\theta_1+\frac{1-\ga}{\lam}\right]\right|^2
&\to0,
\end{align*}
where $K_1$ is a constant independent of $n$. This follows from some simple calculations, and the convergence is due to  $\iC\preceq\lam^{-1}$ and the bounded convergence theorem.

Similarly, we can prove the same results for the other five terms corresponding to equation \eqref{nlquadc1}, \eqref{nlquadc2} and \eqref{nlquadc3}.
With the mean squared error converging to zero, we are now able to show that the expectation of the L.H.S. of \eqref{nlquadc4} also converges to the pointwise constant limit. 

For notational simplicity, we further denote $a:=u^\top\iC\tu/n, b:=\tu^\top\iC\tu/n,c:=u^\top\iC u/n$ and define constants $A:=\lim_{d\to\infty}a, B:=\lim_{d\to\infty}b$. Now, for the second term in the L.H.S. of \eqref{nlquadc4}, we have
\begin{align}
\lim&_{d\to\infty}\E\left|\frac{a^2}{1+b}-\frac{A^2}{1+B}\right|
\leq
\lim_{d\to\infty}\E\left|a\frac{A-a}{1+b}\right|+\left|A\frac{A-a}{1+b}\right|+\left|\frac{A^2(B-b)}{(1+B)(1+b)}\right|\nnum\\
&\leq
\sqrt{\E(A-a)^2\E\left|\frac{a}{1+b}\right|^2}+\sqrt{\E(A-a)^2\E\left|\frac{A}{1+b}\right|^2}+\sqrt{\E(B-b)^2\E\left|\frac{A^2}{(1+b)(1+B)}\right|^2}\nnum\\
&\leq
\sqrt{K_2\E(A-a)^2}+\sqrt{K_3\E(A-a)^2}+\sqrt{K_4\E(B-b)^2}
\to0\label{nlt1expr5},
\end{align}
where $K_2,K_3,K_4$ are some constants independent of $n$. The existence of $K_2,K_3,K_4$ is clear, and here we only take $K_2$ as an example. Since
\begin{align*}
\E\left|\frac{a}{1+b}\right|^2 &\leq\E a^2\leq \E bc\leq \frac{1}{\lam^2n^2}\E u^\top u\tu^\top\tu\to\frac{2\mu+v}{\lam^2},
\end{align*}
we can choose $K_2:=\sup_{n}\E u^\top u\tu^\top\tu/n^2\lam^2<\infty$. (The second inequality can be proved by comparing each term in the expression of $a^2$ and $bc$ and noting that $\iC_{ii}\iC_{jj}\geq (\iC_{ij})^2$ holds for any $i,j$.)

Finally, noting $\lim_{d\to\infty}\E a=\lim_{d\to\infty}\E\mu\tr(\iC)/n\to A$ and using \eqref{nlt1expr5}, we obtain that the expectation of L.H.S. of $\eqref{nlquadc4}$ satisfies 
\begin{align*}
\lim_{d\to\infty}\frac{p}{d}\E\left(a-\frac{a^2}{1+b}\right)&=
\pi\left(A-\frac{A^2}{1+B}\right)
=\pi\frac{\mu}{v}\left[1-\frac{\lam}{v}\theta_1\left({\ga},\frac{\lam}{v}\right)\right].
\end{align*}
This finishes the proof.
\end{proof}
\begin{lemma}[Asymptotic behavior of $T_4$]\label{nlt2expr}
Under assumptions in theorem \ref{nlbiasvardecomp}, we have 
\begin{align}\label{nlt2expr2}&\lim_{d\to\infty}\frac{1}{d}\E_{W,X}\left\|R\frac{\sigma(WX^\top)XW^\top}{n}\right\|_{F}^2\nnum\\
&=\frac{\pi}{v}\left[1-\frac{2(v-\mu^2)}{v}-\frac{2\lam\mu^2}{v^2}\theta_1+\right.\left.\frac{\lam^2\mu^2}{v^3}\theta_2+\frac{v-\mu^2}{v}\left({1}+{\ga}\theta_1-\frac{\lam\ga}{v}\theta_2\right)\right],
\end{align}
where $\theta_1:=\theta_1(\ga,\lam/v),\theta_2:=\theta_2(\ga,\lam/v)$.
\end{lemma}
\begin{proof}[Proof of Lemma \ref{nlt2expr}]
Denote $\tilde R = \left(\frac{\sigma(XW^\top)\sigma(WX^\top)}{n}+\lam\right)^{-1}$. By definition,
\begin{align*}
&\frac{1}{d}\E_{W,X}\left\|R\frac{\sigma(WX^\top)XW^\top}{n}\right\|_{F}^2
=\frac{1}{d}\E_{W,X}\left\|\frac{\sigma(WX^\top)}{\sqrt{n}}\tilde R \frac{XW^\top}{\sqrt{n}}\right\|_{F}^2\\
&=
\frac{1}{nd}\E_{W,X}\tr\left[{WX^\top }\tilde R   {XW^\top} \right]-
\frac{\lam}{nd}\E_{W,X}\tr\left[{WX^\top }\tilde R^2{XW^\top}\right]
=:M_1-\lam M_2.
\end{align*}
Using the same notations and techniques as in the proof of Lemma \ref{nlt1expr}, after some similar calculations, we get
\begin{align*}
\lim_{d\to\infty} M_i &=\lim_{d\to\infty}\frac{\pi}{n} \E u^\top\left(\iC -\frac{\iC\tu\tu^\top\iC }{n+\tu^\top\iC\tu}\right)^{i}u, &i=1,2.
\end{align*}
More specifically,
\begin{align}
\lim_{d\to\infty} M_1 &=\pi\lim_{d\to\infty}\E\left( \frac{u^\top\iC u}{n} -\frac{u^\top\iC\tu\tu^\top\iC u/n^2}{1+\tu^\top\iC\tu/n}\right). \label{nlt2expr3}\\
\lim_{d\to\infty} M_2 &=\pi\lim_{d\to\infty}\E\left( \frac{u^\top\iCt u}{n} -2\frac{u^\top\iCt\tu\tu^\top\iC u/n^2}{1+\tu^\top\iC\tu/n}+\frac{u^\top\iC\tu\tu^\top\iCt\tu\tu^\top\iC u/n^3}{(1+\tu^\top\iC\tu/n)^2}\right).\label{nlt2expr4}
\end{align}
Substituting equation \eqref{nlquadc1}, \eqref{nlquadc2} and \eqref{nlquadc3} into \eqref{nlt2expr3}, \eqref{nlt2expr4}, we can see that the random variables on the R.H.S. of \eqref{nlt2expr3}, \eqref{nlt2expr4} almost surely converge to some constant. Using a similar argument as in the proof of Lemma \ref{nlt1expr}, it can be shown that the expectations on the R.H.S. of equation \eqref{nlt2expr3}, \eqref{nlt2expr4} both converge to their corresponding pointwise constant limits.

Therefore, denoting the R.H.S. of \eqref{nlquadc1} by $k_i(i=1,2)$ and  replacing the R.H.S of \eqref{nlt2expr3}, \eqref{nlt2expr4} by their pointwise constant limits, we obtain
\begin{align}
\lim_{d\to\infty}M_1-\lam M_2
&=\pi\left(k_1-\frac{\mu^2k_1^2}{1+vk_1}\right)-\lam\pi\left[k_2-\frac{2\mu^2k_1k_2}{1+vk_1}+\frac{\mu^2 vk_1^2k_2}{(1+vk_1)^2}\right].\label{formulak1k2}
\end{align}
From the remark after definition \ref{resolvent}, it is readily verified that $1/(1+vk_1)=\lam\theta_1/v$. Thus \eqref{formulak1k2} is a polynomial function of $\theta_{1,2}$. Finally, canceling the high order ($\geq 2$) terms in \eqref{formulak1k2} using the equations in the remark after definition \ref{resolvent}, we have after some calculations that 
\begin{align*}
\lim_{d\to\infty}M_1-\lam M_2&=\frac{\pi}{v}\left[1-\frac{2(v-\mu^2)}{v}-\frac{2\lam\mu^2}{v^2}\theta_1+\right.\left.\frac{\lam^2\mu^2}{v^3}\theta_2+\frac{v-\mu^2}{v}\left({1}+{\ga}\theta_1-\frac{\lam\ga}{v}\theta_2\right)\right].
\end{align*}
\end{proof}
\begin{lemma}[Asymptotic behavior of $T_5$]\label{nlt3expr}
Under assumptions in theorem \ref{nlbiasvardecomp}, we have 
\begin{align}\lim_{d\to\infty}\E_{W,X}\left\|R\frac{\sigma(WX^\top)}{n}\right\|_{F}^2=\frac{\ga}{v}\left(\theta_1-\frac{\lam}{v}\theta_2\right)\label{nlt3expr2},
\end{align}
where $\theta_1:=\theta_1(\ga,\lam/v),\theta_2:=\theta_2(\ga,\lam/v)$.
\end{lemma}
\begin{proof}[Proof of Lemma \ref{nlt3expr}]
By definition,
\begin{align*}
\mathrm{L.H.S.  \text{ of\phantom{a}}} \eqref{nlt3expr2}&=
\E_{W,X}\tr\left[ R^2 \frac{\sigma(WX^\top)\sigma(XW^\top)}{n^2}\right]
=\E_{W,X}\tr\left[R\right]-\lam\E_{W,X}\tr\left[ R^2 \right].
\end{align*}
Note that $\sigma(XW^\top)$ has i.i.d. zero mean $v$ variance entries, by the Marchenko-Pastur theorem and similar methods in the proof of theorem \ref{2lthm1}, we get
\begin{align*}
\mathrm{L.H.S.  \text{ of\phantom{a}}} \eqref{nlt3expr2}&=\frac{\ga}{v}\theta_1\left(\ga,\frac{\lam}{v}\right)-\frac{\lam\ga}{v^2}\theta_2\left(\ga,\frac{\lam}{v}\right).
\end{align*}
\end{proof}

\bibliography{references}


\end{document}